\documentclass[twocolumn,twoside,lettersize]{IEEEtran}

\usepackage{setspace,cite}

\usepackage[table]{xcolor}

\usepackage[tbtags]{amsmath}
\usepackage{amsbsy}
\usepackage{amssymb}
\usepackage{amsfonts}
\usepackage{amsthm}

\usepackage{graphicx}
\usepackage{algorithmic}
\usepackage{algorithm}
\usepackage{balance}
\usepackage{rotating}
\usepackage{multirow}
\usepackage{subcaption}
\usepackage{fancyvrb}
\usepackage{latexsym}
\usepackage{verbatim}

\newtheorem{proposition}{Proposition}
\newtheorem{lemma}{Lemma}

\newtheorem{remark}{Remark}

\newtheorem{example}{Example}

{\begin{list}               
    {$\bullet$ \hfill}{
        \setlength{\leftmargin}{\parindent}
        \setlength{\parsep}{0.04\baselineskip}
        \setlength{\itemsep}{0.5\parsep}
        \setlength{\labelwidth}{\leftmargin}
        \setlength{\labelsep}{0em}}
    }
{\end{list}}

\providecommand{\eref}[1]{\eqref{#1}}  
\providecommand{\cref}[1]{Chapter~\ref{#1}}

\providecommand{\fref}[1]{Figure~\ref{#1}}

\providecommand{\R}{\ensuremath{\mathbb{R}}}

\providecommand{\E}{\ensuremath{\mathbb{E}}}

\providecommand{\abs}[1]{\lvert#1\rvert}
\providecommand{\norm}[1]{\left\lVert#1\right\rVert}

\providecommand{\bydef}{\overset{\text{def}}{=}}

\renewcommand{\vec}[1]{\ensuremath{\boldsymbol{#1}}}
\providecommand{\mat}[1]{\ensuremath{\boldsymbol{#1}}}

\providecommand{\calB}{\mathcal{B}}

\providecommand{\calD}{\mathcal{D}}

\providecommand{\calL}{\mathcal{L}}
\providecommand{\calM}{\mathcal{M}}
\providecommand{\calN}{\mathcal{N}}

\providecommand{\mA}{\mat{A}}
\providecommand{\mB}{\mat{B}}

\providecommand{\mI}{\mat{I}}

\providecommand{\mS}{\mat{S}}

\providecommand{\mU}{\mat{U}}

\providecommand{\mZ}{\mat{Z}}

\providecommand{\ve}{\vec{e}}

\providecommand{\vg}{\vec{g}}

\providecommand{\vu}{\vec{u}}
\providecommand{\vv}{\vec{v}}
\providecommand{\vw}{\vec{w}}
\providecommand{\vx}{\vec{x}}
\providecommand{\vy}{\vec{y}}
\providecommand{\vz}{\vec{z}}

\providecommand{\mSigma}{\mat{\Sigma}}

\providecommand{\veta}{\vec{\eta}}

\providecommand{\vztilde}{\boldsymbol{\widetilde{z}}}
\providecommand{\mSigmatilde}{\mat{\widetilde{\Sigma}}}
\providecommand{\vwtilde}{\boldsymbol{\widetilde{w}}}

\providecommand{\sigmahat}{\widehat{\sigma}}

\providecommand{\vzhat}{\boldsymbol{\widehat{z}}}

\providecommand{\mZhat}{\mat{\widehat{Z}}}
\providecommand{\vwhat}{\boldsymbol{\widehat{w}}}

\providecommand{\vone}{\vec{1}}

\newcommand{\subjectto}{\mathop{\mathrm{subject\, to}}}
\newcommand{\argmin}[1]{\mathop{\underset{#1}{\mbox{argmin}}}}

\newcommand{\minimize}[1]{\mathop{\underset{#1}{\mathrm{minimize}}}}

\newcommand{\MSE}{\mathrm{MSE}}

\graphicspath{{pix/}}
\usepackage[pagebackref=true,breaklinks=true,letterpaper=true,colorlinks,bookmarks=false]{hyperref}

\begin{document}
\title{Optimal Combination of Image Denoisers}

\author{Joon~Hee~Choi,~\IEEEmembership{Member,~IEEE,}
        Omar~A.~Elgendy,~\IEEEmembership{Student Member,~IEEE,}
        \\ and~Stanley~H.~Chan,~\IEEEmembership{Senior Member,~IEEE}
        \thanks{The authors are with the School of Electrical and Computer Engineering, Purdue University, West Lafayette, IN 47907, USA. S. Chan is also with the Department of Statistics, Purdue University, West Lafayette, IN 47907, USA. Emails: \texttt{ \{choi240, oelgendy, stanchan\}@purdue.edu}. }
        \thanks{This work was supported, in part, by the U.S. National Science Foundation under Grant CCF-1718007 and CCF-1763896.}
    }

\markboth{IEEE Transactions on Image Processing, Vol. X, No. X, February 2019}%
{Choi \MakeLowercase{\textit{et al.}}: Optimal Combination of Image Denoisers}

\maketitle

\begin{abstract}
Given a set of image denoisers, each having a different denoising capability, is there a provably optimal way of combining these denoisers to produce an overall better result? An answer to this question is fundamental to designing an ensemble of weak estimators for complex scenes. In this paper, we present an optimal combination scheme by leveraging deep neural networks and convex optimization. The proposed framework, called the Consensus Neural Network (CsNet), introduces three new concepts in image denoising: (1) A provably optimal procedure to combine the denoised outputs via convex optimization; (2) A deep neural network to estimate the mean squared error (MSE) of denoised images without needing the ground truths; (3) An image boosting procedure using a deep neural network to improve contrast and to recover lost details of the combined images. Experimental results show that CsNet can consistently improve denoising performance for both deterministic and neural network denoisers.
\end{abstract}


\IEEEpeerreviewmaketitle

\section{Introduction}
\IEEEPARstart{W}{hile} image denoising algorithms over the past decade have produced very promising results, it is also safe to say that no single method is uniformly better than others. In fact, any image denoiser, either deterministic \cite{Buades2005,Dabov2007,Beck2009,Zoran2011,Chan2011,Gu2014,Chan2014,Chi2018} or learning-based \cite{Elad2006,Mairal2009,Roth2009,Burger2012,Xie2012,Dong2013_1,Dong2013_2,Xu2015,Mao2016,Remez2017,Zhang2017_tip,Lefkimmiatis2017,Zhang2017_arxiv,Xu_Zhang_Zhang_2018}, has an implicit prior model that determines its denoising characteristics. Since a particular prior model encapsulates the statistics of a limited set of imaging conditions, the corresponding denoiser is only an expert for the type of images it is designed to handle. We refer to this gap between the denoising model and the denoising task as a model mismatch.

\begin{figure}[t]
	\centering
	\begin{subfigure}[b]{0.15\textwidth}
		\centering
		\includegraphics[width=\textwidth]{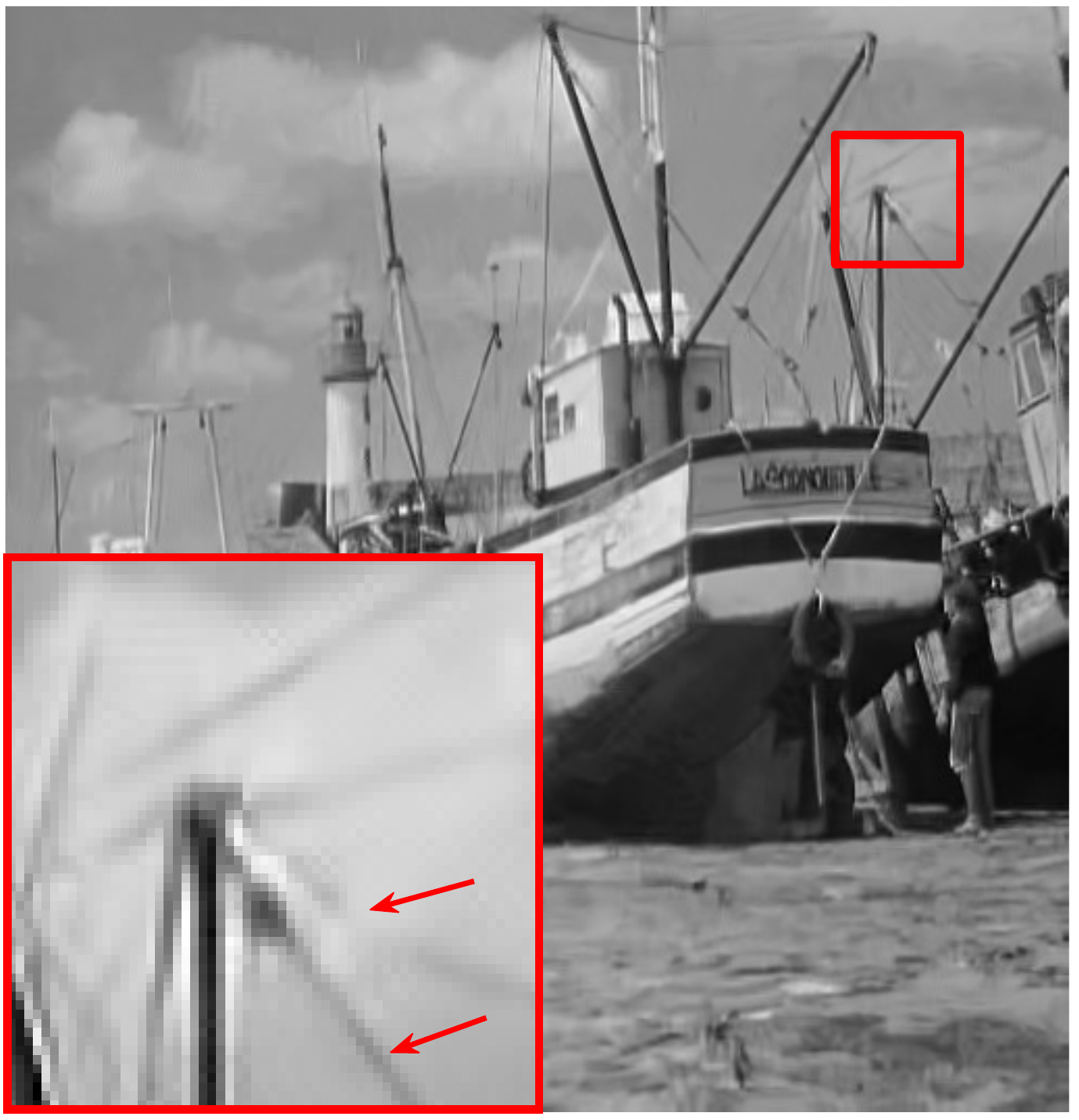}
		\caption{\scriptsize{BM3D, 30.85dB}}
	\end{subfigure}
	\begin{subfigure}[b]{0.15\textwidth}
		\centering
		\includegraphics[width=\textwidth]{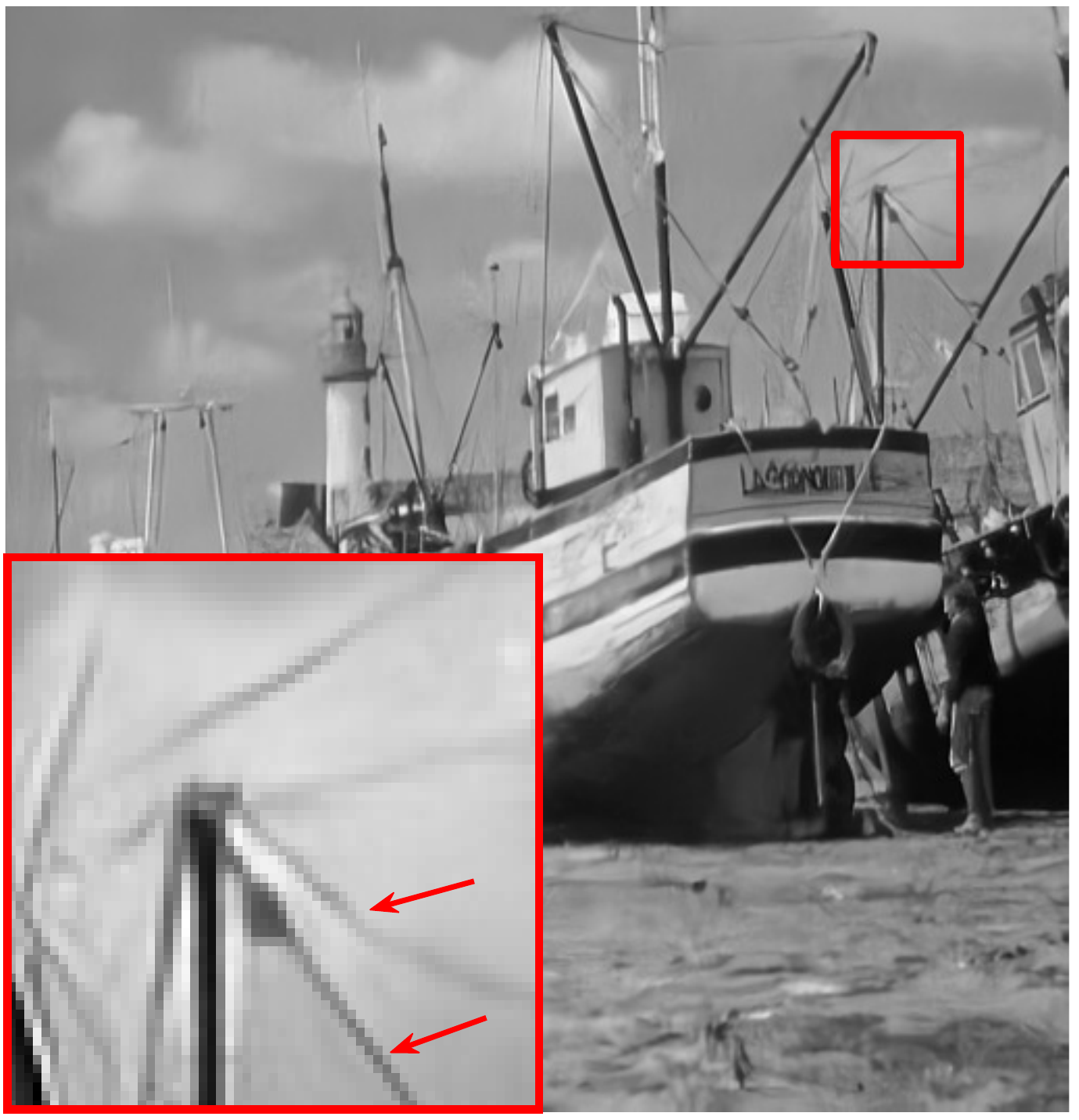}
		\caption{\scriptsize{DnCNN, 31.14dB}}
	\end{subfigure}
	\begin{subfigure}[b]{0.15\textwidth}
		\centering
		\includegraphics[width=\textwidth]{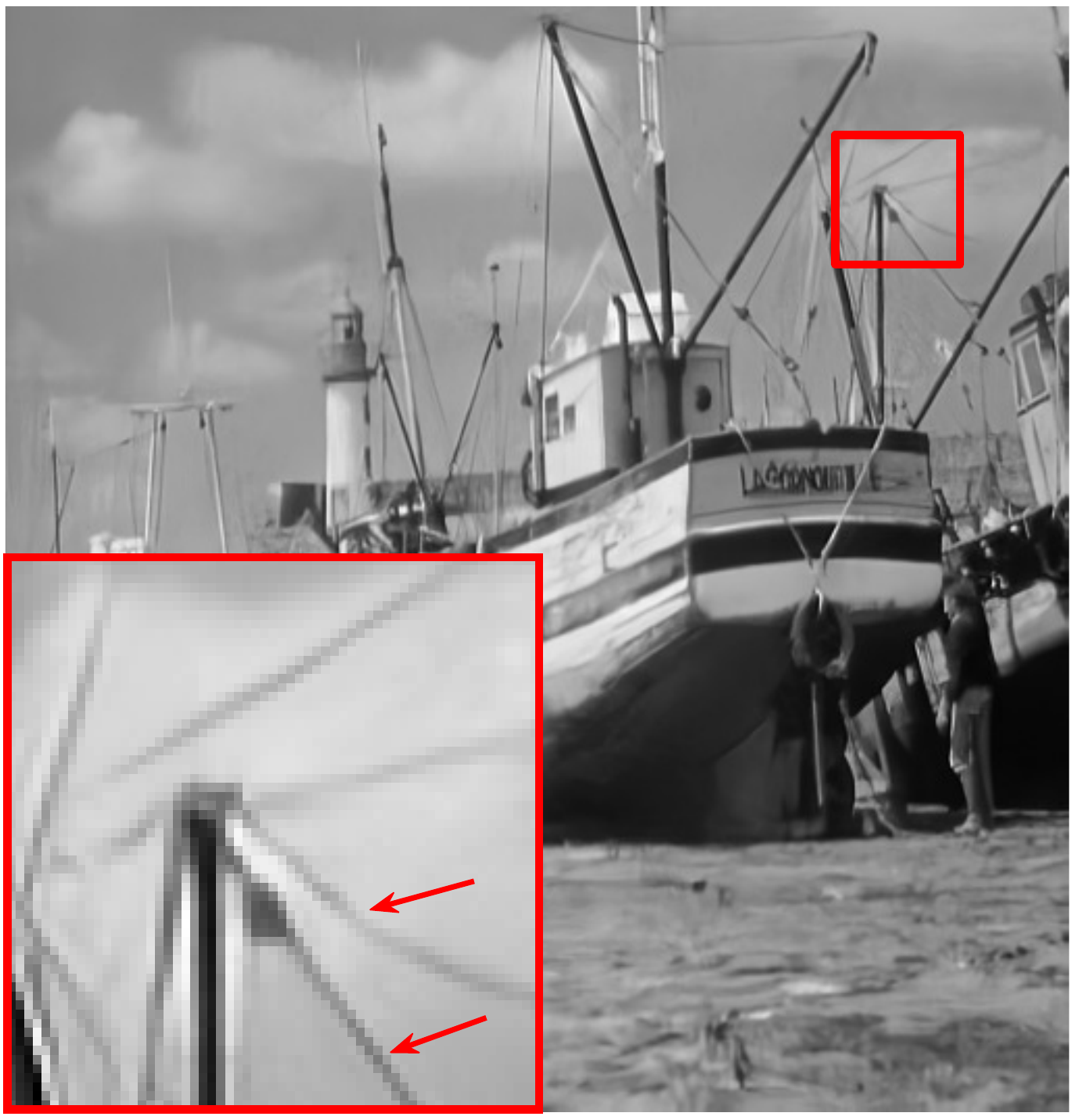}
		\caption{\scriptsize{Ours, 31.32dB}}
	\end{subfigure}
	
	\begin{subfigure}[b]{0.15\textwidth}
		\centering
		\includegraphics[width=\textwidth]{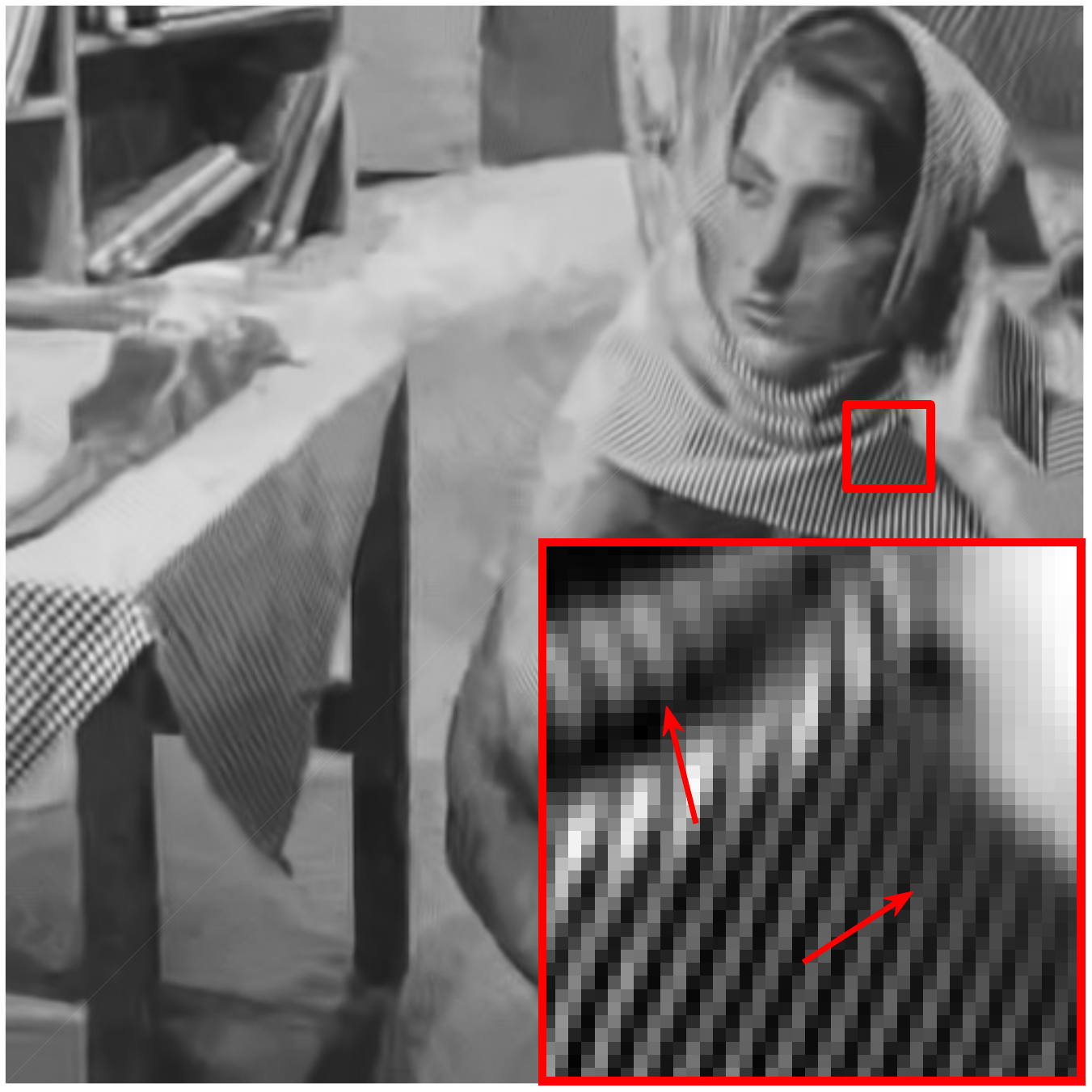}
		\caption{\scriptsize{BM3D, 26.80dB}}
	\end{subfigure}
	\begin{subfigure}[b]{0.15\textwidth}
		\centering
		\includegraphics[width=\textwidth]{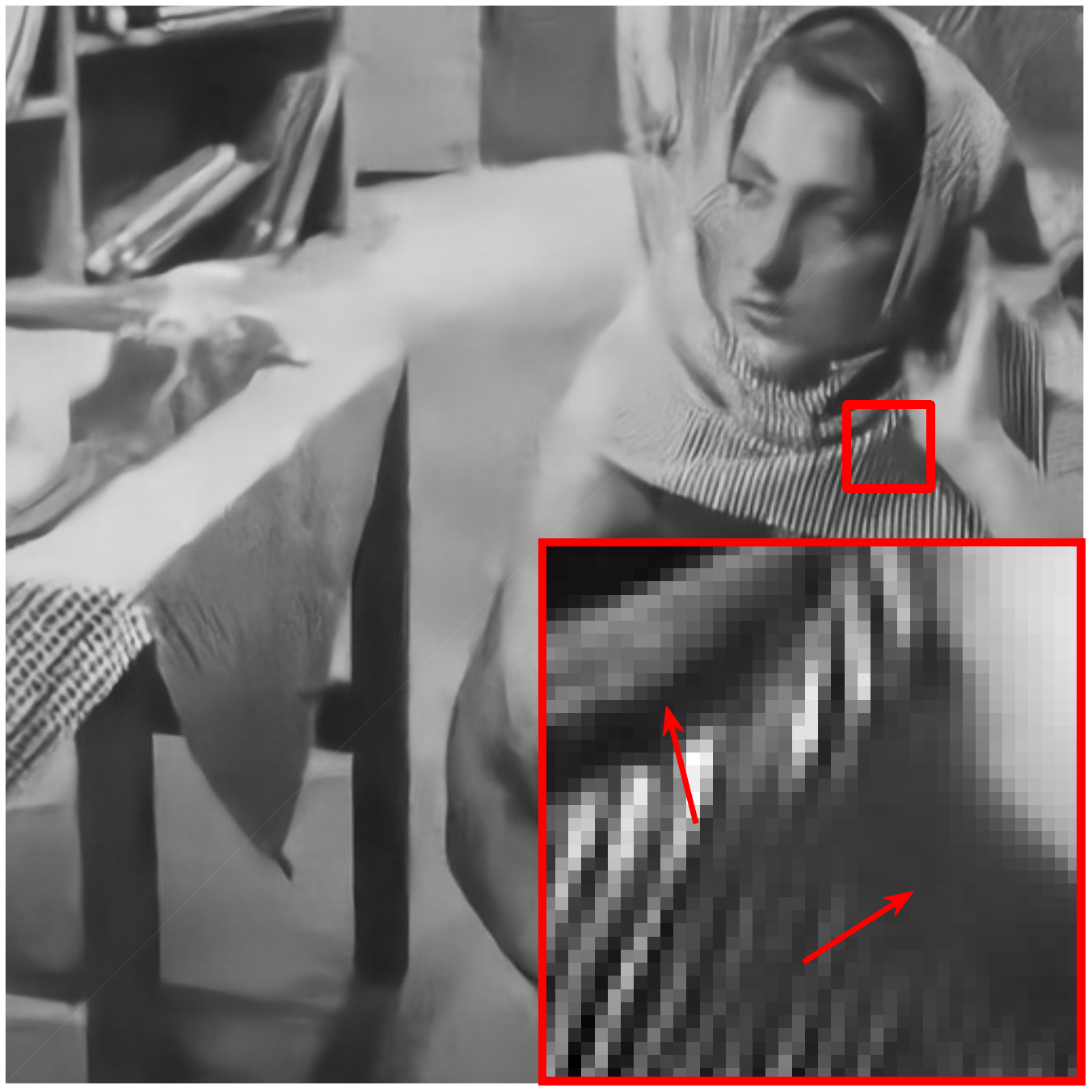}
		\caption{\scriptsize{DnCNN, 26.49dB}}
	\end{subfigure}
	\begin{subfigure}[b]{0.15\textwidth}
		\centering
		\includegraphics[width=\textwidth]{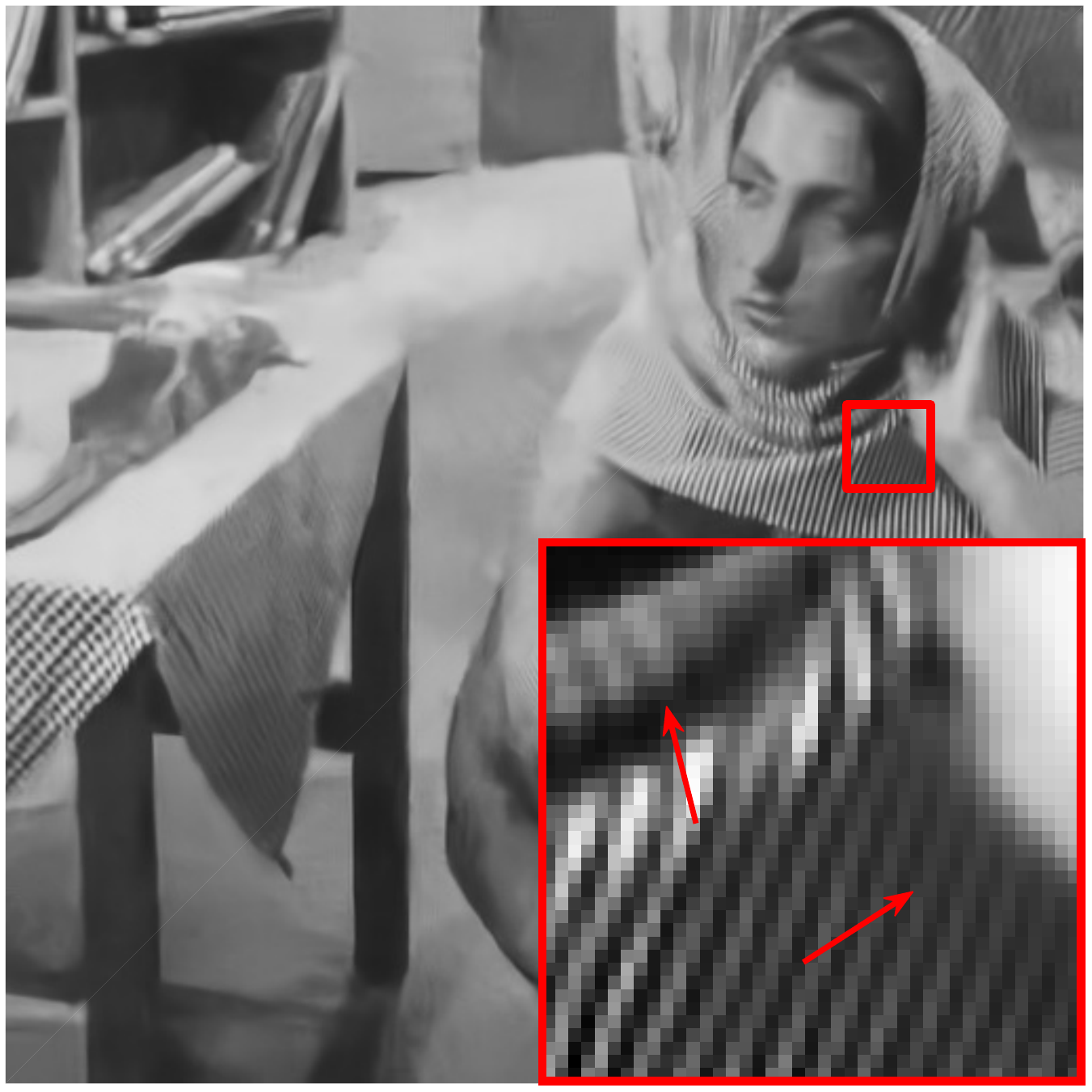}
		\caption{\scriptsize{Ours, 26.80dB}}
	\end{subfigure}
	\caption{Comparison of BM3D \cite{Dabov2007}, DnCNN \cite{Zhang2017_tip} and the proposed CsNet. The \texttt{boat} image is corrupted by noise of $\sigma = 20$, whereas \texttt{Barbara} is corrupted by noise of $\sigma = 40$. The denoising strength of the denoisers are adjusted to match the actual noise level. The results show that different denoisers are better for different types of images, e.g., BM3D is better for repeated pattern whereas DnCNN is better for generic content. The combination scheme proposed in this paper is able to leverage the better among the two.}
\vspace{-3ex}
	\label{fig:front page example}
	
\end{figure}

Model mismatch is common in practice. In this paper, we are particularly interested in the following three examples:

\begin{itemize}
\item Denoiser Characteristic: Every denoiser has a different characteristic. For example, BM3D \cite{Dabov2007} assumes patch reoccurrence, and thus it works well for images with repeated patterns. Neural network denoisers are trained on generic images, and thus they work well for those images. Figure~\ref{fig:front page example} shows an example of BM3D \cite{Dabov2007} and a neural network denoiser DnCNN \cite{Zhang2017_tip}. The \texttt{Boat512} image is corrupted by i.i.d. Gaussian noise of noise level $\sigma = 20$. In this example, DnCNN (trained at $\sigma = 20$) gives a PSNR of 31.14dB which is approximately 0.3dB higher than BM3D. The other image \texttt{Barbara512} is corrupted by a noise of level $\sigma = 40$. In this case, BM3D actually performs better than DnCNN (trained at $\sigma = 40$), yielding 26.80dB over the 26.49dB. If we look at the image content, we can see that \texttt{Barbara512} has a repeated pattern on the cloth which is more favorable to BM3D. This shows the influence of the implicit modelings of a denoiser to the performance.
\item Noise Level: For neural network image denoisers, the performance is strongly affected by the noise level under which the denoiser is trained. For example, if a denoiser is trained for i.i.d. Gaussian noise of standard deviation $\sigma$, it only works well for this particular $\sigma$. As soon as the noise level deviates, the performance will degrade. The same argument holds for deterministic denoisers such as BM3D, as its denoising strength must match the actual noise level. \fref{fig:dncnn} illustrates the behavior of DnCNN and BM3D as the denoising strength $\widehat{\sigma}$ deviates from the actual level $\sigma$. In this experiment, we use five denoising strengths $\widehat{\sigma} = \{10,20,30,40,50\}$ and a continuous range of $\sigma \in [10,50]$. As shown in the plot, BM3D has a slightly more robust performance, in the sense that a chosen denoising strength $\widehat{\sigma}$ can work for a reasonable wide range of actual noise levels $\sigma$. In contrast, DnCNN has a narrow performance regime for a fixed $\widehat{\sigma}$.
\item Image Class: A denoiser trained for a particular class of images (e.g., building) may not work for other classes (e.g., face). When this type of class-aware issue appears, the typical solution is by means of scene classification \cite{Remez2017}. However, scene classification itself is an open problem and there is no consensus of the best approach. Therefore, it would be more convenient if the denoiser can automatically pick a class that gives the best performance without seeking classification algorithms.
\end{itemize}

\begin{figure}[t]
	\centering
	\includegraphics[width=0.45\textwidth]{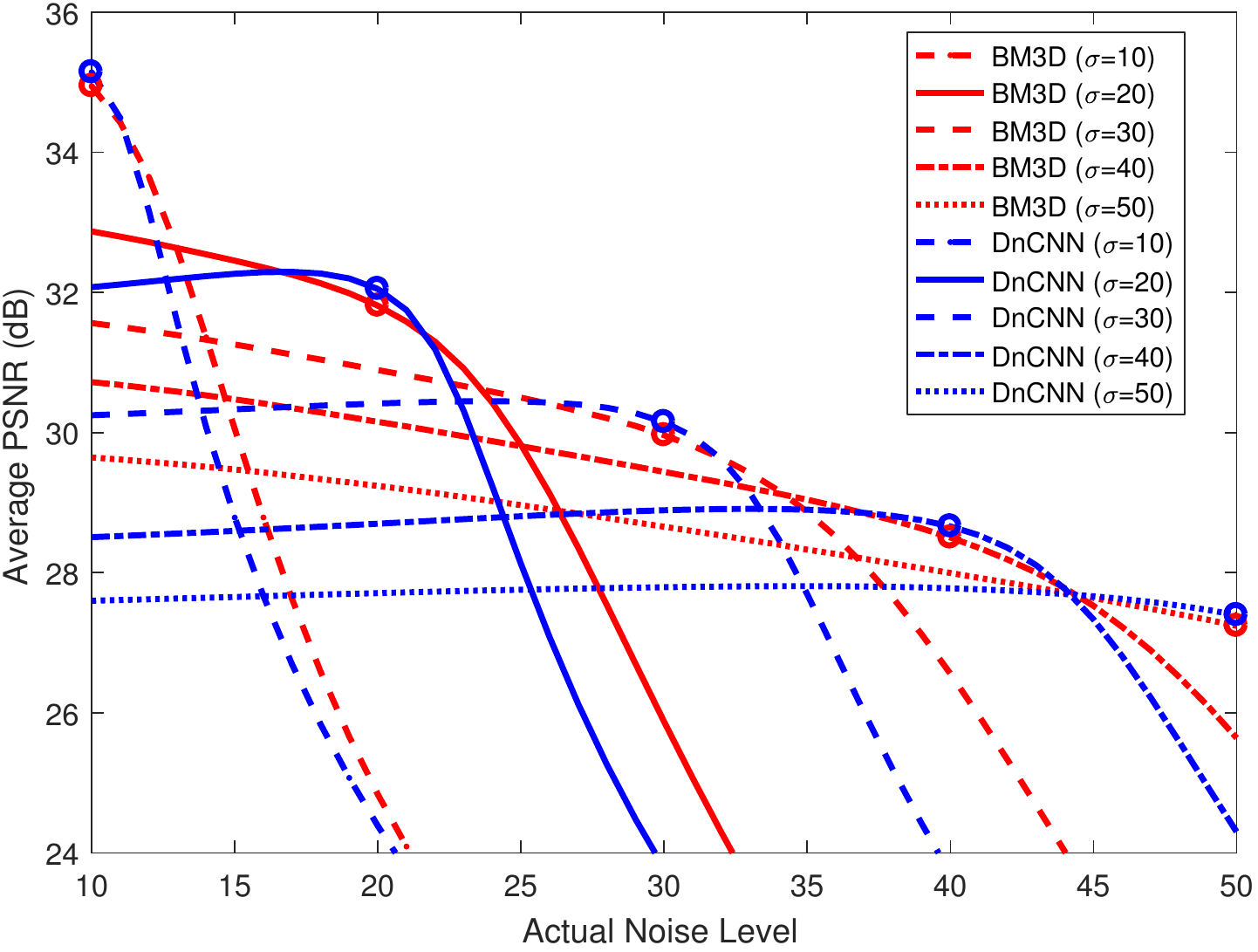}
	\caption{Illustration of noise-level mismatch. We compare BM3Ds and DnCNNs at noise levels $\widehat{\sigma} \in \{10,20,30,40,50\}$ in terms of true noise levels $[10, 50]$ on 10 Kodak images.}
	\label{fig:dncnn}
\end{figure}

The examples above bring out a question that if we have a set of denoisers, each having a different characteristic, how do we combine them to produce a better result? Answering this question is fundamental to designing ensembles of expert image restoration methods for complex scenes. The goal of this paper is to present a framework called the Consensus Neural Network (CsNet) which seeks consensus by using neural networks and convex optimization.

\subsection{Related Work}
Combining estimators is a long-standing statistical problem. In as early as 1959, Graybill and Deal \cite{Graybill_Deal_1959} started to consider linearly combining two unbiased scalar estimators to yield a new estimator that remains unbiased and has lower variance. More properties of the such combination scheme was discussed by Samuel-Cahn \cite{SamuelCahn_1994}. In \cite{Rubin_Weisberg_1975}, Rubin and Weisberg extended the idea by estimating weights from the samples. However, the estimators are still scalars and are assumed to be independent. Correlated scalar estimators are later studied by Keller and Olkin \cite{Keller_Olkin_2002}. For vector estimators (which is the case for image denoisers), Odell et al. \cite{Odell_Dorsett_Young_1989} presented a very comprehensive study. However, their result is limited to two vector estimators. The general case of multiple estimators is studied by Lavancier and Rochet \cite{Lavancier_Rochet_2016}, who proposed an optimization approach to estimate the weights.

Specific to image denoising, methods seeking linear combination of denoisers are scattered in the literature. The most popular approach is perhaps the linear expansion of thresholds by Blu and colleagues \cite{Blu2007}, using the Stein's unbiased risk estimator (SURE). In \cite{Chaudhury2015}, Chaudhury et al. presented an improved bilateral filter using the SURE estimator. For learning based methods, the loss-specific training approach by Jancsary et al. \cite{Jancsary_Nowozin_Rother_2012} presented a regression tree field model to optimize the denoising performance over different metrics. There is also an end-to-end neural network solution for selecting denoisers by Agostinelli et al. \cite{Agostinelli2013}, where the authors proposed to learn the weights using an auto-encoder.

The noise-level mismatch is discussed more often in the neural network literature. Conventional approach is to either truncate the noise level to the nearest trained level \cite{Zhang2017_cvpr} or to train the network with a large number of examples covering all noise levels \cite{Zhang2017_tip}. A more recent approach is to feed a noise map to the network and train the network to recognize the noise level \cite{Zhang2017_arxiv}. However, this approach requires a redesign of the network structure. In contrast, CsNet uses the same structure for all initial denoisers.

\subsection{Contributions}
An overview of the proposed CsNet framework is shown in \fref{fig:framework}. We summarize the three key contributions of this paper in the followings:

\begin{itemize}
\item Optimal Combination. We present an optimal combination framework via convex optimization. By minimizing a quadratic function over a unit simplex, we prove that the resulting combination is optimal in the MSE sense. We provide geometric interpolation of the solution, and a fast algorithm to determine the optimal point.
\item MSE Estimator. We present a novel deep neural network to estimate the mean square error (MSE) in the absence of the ground truth. Existing deep neural network based image quality assessment methods are designed to predict perceptual quality and not MSE. To the best of our knowledge, our deep learning based MSE estimator is the first of this kind in the literature.
\item Denoising Booster. We present a new deep neural network to boost the combined estimates. Unlike the existing boosters which are iterative, we cascade multiple simple neural networks to achieve a one-shot booster.
\end{itemize}

To help readers understand the design process, we proceed the paper by first discussing the optimal combination and its associated theoretical properties in Section II. Section III discusses the neural network estimator for estimating the MSE. We emphasize that the neural network presented here is just one of the many possible ways of estimating the MSE. Readers preferring non-training based approaches can use estimators such as SURE, although we will provide examples where SURE does not work. Section IV discusses the booster, and its cascade structure. Experiments are discussed in Section V.

\begin{figure}[t]
	\centering
	\vspace{-1ex}
	\includegraphics[width=0.48\textwidth]{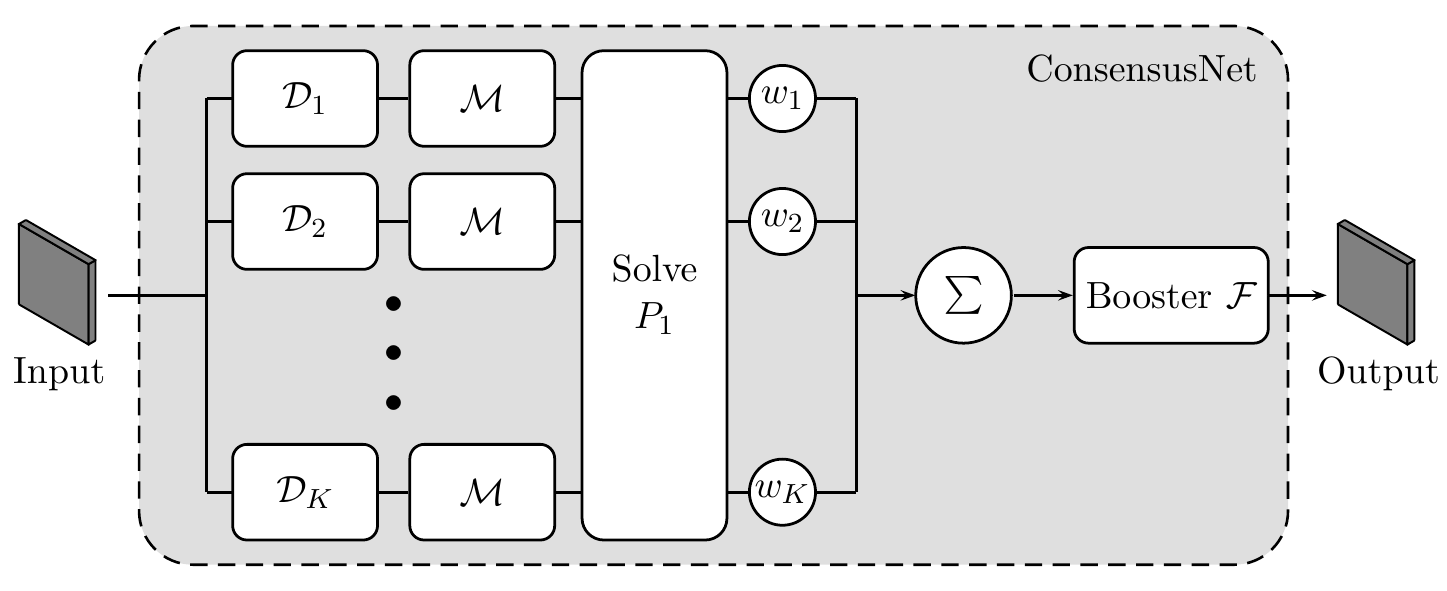}
	\caption{Structure of the proposed CsNet: Given a set of $K$ initial denoisers $\calD_1,\ldots,\calD_K$, CsNet uses an MSE estimator $(\calM)$ to estimate the MSE of each initial denoiser. After the MSEs are estimated, we solve a convex optimization problem ($P_1$) to determine the optimal weight $w_1,\ldots,w_K$. The combined estimate is then boosted using a booster neural network to improve contrast and details.}
	\label{fig:framework}
\end{figure}

\subsection{Notation}
Throughout this paper, we use lower case bold letters to denote vectors, e.g., $\vx \in \R^N$, and upper case bold letters to denote matrices, e.g., $\mA \in \R^{K \times K}$. An all-one vector is denoted as $\vone$. Standard basis vectors are denoted as $\ve_i$, i.e., $\ve_i = [0,\ldots,1,\ldots,0]^T$. For any vector $\vx$, $\|\vx\|_2$ means the $\ell_2$-Euclidean norm, and for any matrix $\mA$, $\|\mA\|_2 = \max_{\|\vx\|_2 = 1} \|\mA\vx\|_2$ denotes the matrix operator norm. To specify that a vector $\vx$ is non-negative for all its elements, we write $\vx \succeq 0$. For matrices, $\mA \succeq 0$ means that $\mA$ is positive semi-definite. Images in this paper are normalized so that every pixel is in $[0,1]$. Noise level of an i.i.d. Gaussian noise is specified by its standard deviation $\sigma$. For notational simplicity, we write $\sigma$ in the scale of $[0,255]$, e.g., ``$\sigma = 20$'' means $\sigma = 20/255$. Finally, an image denoiser $\calD$ is a mapping $\calD: [0,1]^N \to [0,1]^N$. We assume $\calD$ is bounded and is asymptotically invariant \cite{Chan2017}.

\section{Optimal Combination of Estimators}

\subsection{Problem Formulation}
Consider a linear forward model where a clean image $\vz \in \R^N$ is corrupted by additive i.i.d. Gaussian noise $\veta \sim \calN(0,\sigma^2\mI)$ so that the observed image is $\vy = \vz + \veta$. We apply a set of $K$ image denoisers $\calD_1,\ldots,\calD_K$ to yield $K$ initial estimates $\vzhat_k = \calD_k(\vy)$ for $k = 1,\ldots,K$. For convenience, we concatenate these initial estimates by constructing a matrix $\mZhat = [\vzhat_1,\ldots,\vzhat_K] \in \R^{N \times K}$.

In this paper, we are interested in the \emph{linear combination} of estimators. That is, for a given $\mZhat$, we construct the linearly combined estimate as
\begin{equation}
\vzhat = \sum_{k=1}^K w_k \vzhat_k = \mZhat \vw,
\end{equation}
where $\vw \bydef [w_1,\ldots,w_K]^T \in \R^K$ is the vector of combination weights. The goal of our work is to formulate an optimization problem to determine the optimal weights.

For analytic tractability, we use mean squared error (MSE) to measure the optimality, although it is known that alternative visual quality metrics correlate better to human visual systems \cite{Wang_Bovik_Sheikh_2004}. Denoting $\vz \in \R^N$ as the ground truth, we define the MSE between the combined estimate $\vzhat$ and the ground truth $\vz$ as
\begin{equation}
\mathrm{MSE}(\vzhat,\vz) \bydef \E \left[ \|\vzhat - \vz\|^2 \right] = \E \left[ \left\| \mZhat\vw - \vz \right\|^2 \right].
\end{equation}

The optimal combination problem can be posed as minimizing the MSE by seeking the weight vector $\vw \in \R^K$:
\begin{equation}
\begin{array}{ll}
\minimize{\vw} &\;\; \E \left[ \|\mZhat\vw - \vz\|^2 \right]\\
\subjectto     &\;\; \vw^T\vone = 1, \quad\mbox{and}\quad \vw \succeq 0.
\end{array}
\label{eq:optimization 1}
\end{equation}
Here, the constraint $\vw^T\vone = 1$ ensures that the sum of the weights is 1, and the constraint $\vw \succeq 0$ ensures that the combined estimate remains in $[0,1]^N$.


Let us simplify \eref{eq:optimization 1}. First, we define $\mZ = [\vz,\ldots,\vz] \in \R^{N \times K}$, i.e., a matrix with the ground truth $\vz$ in each column. Since $\vw^T\vone = 1$, we can show that
\begin{align*}
\E \left[ \left \|\mZhat\vw - \vz \right \|^2 \right]
&= \E \left[ \left \|\mZhat\vw - \mZ\vw \right \|^2 \right]\\
&= \E \left[ \vw^T (\mZhat - \mZ)^T (\mZhat - \mZ)\vw  \right]\\
&= \vw^T \mSigma \vw,
\end{align*}
where $\mSigma$ is defined as
$$\mSigma \bydef \E \left[(\mZhat - \mZ)^T (\mZhat - \mZ) \right].$$ We call $\mSigma$ the \emph{covariance matrix}\footnote{Straightly speaking, $\mSigma \bydef \E \left[(\mZhat - \mZ)^T (\mZhat - \mZ) \right]$ is not the conventional covariance matrix because denoisers are not necessarily unbiased, i.e., $\E[\mZhat] \not= \mZ$. }. Using this result, we can rewrite \eref{eq:optimization 1} into an equivalent form as
\begin{equation}
\begin{array}{ll}
\minimize{\vw} &\;\; \vw^T \mSigma \vw  \\
\subjectto     &\;\; \vw^T\vone = 1, \quad\mbox{and}\quad \vw \succeq 0,
\end{array}
\tag{$P_1$}
\label{eq:optimization}
\end{equation}
which is a convex problem because $\mSigma$ is positive semi-definite and the feasible set is convex.

Before we discuss how to solve \eref{eq:optimization}, we should first discuss how to obtain $\mSigma$. The $(i,i)$-th entry of $\mSigma$ is
\begin{equation*}
\Sigma_{ii} = \E \left[ \left\| \vzhat_i - \vz \right\|^2 \right] \bydef \mathrm{MSE}_{i},
\end{equation*}
which is the MSE of the $i$-th estimate. The $(i,j)$-th entry of $\mSigma$ is the correlation between $\vzhat_i$ and $\vzhat_j$:
\begin{align*}
\Sigma_{ij} = \E \left[ (\vzhat_i - \vz)^T(\vzhat_j - \vz)\right].
\end{align*}
To express $\Sigma_{ij}$ in terms of $\mathrm{MSE}_{i}$ and $\mathrm{MSE}_{j}$, we notice that
\begin{align*}
\E \left[ \left\| \vzhat_i - \vzhat_j \right\|^2\right]
&= \E \left[ \left\| \vzhat_i - \vz + \vz - \vzhat_j \right\|^2\right] \\
&= \E \left\| \vzhat_i - \vz \right\|^2 + \E \left\| \vzhat_j - \vz \right\|^2  + \ldots \\
&\quad\quad - 2 \E \left[ (\vzhat_i - \vz)^T(\vzhat_j - \vz)\right]\\
&= \mathrm{MSE}_{i} + \mathrm{MSE}_{j} - 2 \Sigma_{ij}.
\end{align*}
Rearranging the terms we can write $\Sigma_{ij}$ as
\begin{equation}
\Sigma_{ij} = \frac{ \mathrm{MSE}_{i} + \mathrm{MSE}_{j} - \E \left[ \left\| \vzhat_i - \vzhat_j \right\|^2 \right]}{2}.
\label{eq:Sigmaij}
\end{equation}
Therefore, when we do not have true $\mathrm{MSE}_{i}$ and $\mathrm{MSE}_j$ but estimates $\widetilde{\mathrm{MSE}}_i$ and $\widetilde{\mathrm{MSE}}_j$, \eref{eq:Sigmaij} provides a convenient way to construct $\Sigma_{ij}$ because $\E \left[ \left\| \vzhat_i - \vzhat_j \right\|^2\right]$ does not require the ground truth.

\subsection{Solving \eref{eq:optimization}}
The optimization problem in \eref{eq:optimization} is a quadratic minimization over a unit simplex. The problem does not have a closed form solution because the KKT conditions involve a complementary slackness term due to the non-negativity constraint. Iterative algorithms are available though, e.g., using general purpose semi-definite programming such as CVX \cite{cvx,Grant_Boyd_2008}, or using projected gradients \cite{Condat_2017,Mairal_2013}. However, since \eref{eq:optimization} has a simple structure, efficient algorithms can be derived.

Our algorithm is an accelerated gradient method following the work of Jaggi \cite{Jaggi_2011}. We briefly describe the algorithm for completeness. Let
\begin{equation}
f(\vw) = \vw^T\mSigma\vw
\end{equation}
be the objective function, and
\begin{equation}
\Omega \bydef \{\vw \;|\; \vw^T\vone = 1, \;\mbox{and}\; \vw \succeq 0\}
\end{equation}
be the feasible set. The first order linear approximation at the $t$-th iterate is
\begin{equation*}
f(\vu) = f(\vw^{(t)}) + \nabla f (\vw^{(t)})^T (\vu - \vw^{(t)}), \quad \forall \vu \in \Omega.
\end{equation*}
Thus, for any $\vu \in \Omega$, $\vu - \vw^{(t)}$ is a feasible search direction. One choice of $\vu$ is to make $\nabla f (\vw^{(t)})^T \vu$ minimized so that $f(\vu)$ has a lower cost. This leads to
\begin{equation}
\minimize{\vu \in \Omega} \;\; \nabla f(\vw^{(t)} )^T \vu,
\label{eq:greedy search direction}
\end{equation}
which has a linear objective function. Once $\vu$ is determined, we construct a standard accelerated gradient step:
\begin{equation}
\vw^{(t+1)} = \vw^{(t)} + \alpha (\vu - \vw^{(t)}),
\end{equation}
where $\alpha = \frac{2}{t+2}$ is the step size.

It remains to find the solution of the subproblem \eref{eq:greedy search direction}. However, the subproblem \eref{eq:greedy search direction} is a linear programming over the unit simplex. Therefore, the solution has to lie on one of the vertices. We derive a closed-form solution in Proposition~\ref{prop:algorithm}. The pseudo-code is shown in Algorithm~\ref{alg:greedy algorithm}.

\begin{proposition}
\label{prop:algorithm}
The solution to \eref{eq:greedy search direction} is $\vu = \ve_{i^*}$, where $i^* = \mathrm{argmin}_{i} (\nabla f(\vw^{(t)}))_i$.
\end{proposition}

\begin{proof}
Let $\vg = \nabla f(\vw^{(t)})$. Then it follows that
\begin{align*}
\vg^T\vu &= \sum_{i=1}^K g_i u_i \ge g_{\min} \sum_{i=1}^K u_i = g_{\min},
\end{align*}
where $g_{\min} = \min_{i} g_i$, and $\sum_{i=1}^K u_i = 1$ because $\vu \in \Omega$. The lower bound can be attained when $\vu = \ve_{i^*}$, where $i^* = \mathrm{argmin}_{i} \;\; g_i$.
\end{proof}

\begin{algorithm}[h]
\caption{Algorithm to Solve \eref{eq:optimization}}
\begin{algorithmic}[1]
\STATE Initialize $\vw^0 = \ve_1$.
\FOR{$t = 0,1,\ldots,T_{\max}$}
    \STATE Let $i^* = \argmin{i} \; (\mSigma\vw^{(t)})_i$
    \STATE Update $\vw^{(t+1)} = \vw^{(t)} + \left(\frac{2}{t+2}\right)(\ve_{i^*} - \vw^{(t)})$.
\ENDFOR
\end{algorithmic}
\label{alg:greedy algorithm}
\end{algorithm}

\begin{example}
As an illustration of Algorithm~\ref{alg:greedy algorithm}, we compare its performance with an ADMM algorithm by Condat \cite{Condat_2017}. The reference method is CVX \cite{cvx}. We repeat the experiment $1000$ times using different random matrices $\mathbf{\Sigma}$, and take the average. As shown in \fref{fig:algorithm1}, Algorithm~\ref{alg:greedy algorithm} converges significantly faster than \cite{Condat_2017}. In terms of runtime, Algorithm~\ref{alg:greedy algorithm} takes about $4.4$ msec, \cite{Condat_2017} takes $13$ msec, and CVX takes $223.1$ msec.
\begin{figure}[h]
\centering
\includegraphics[width=0.8\linewidth]{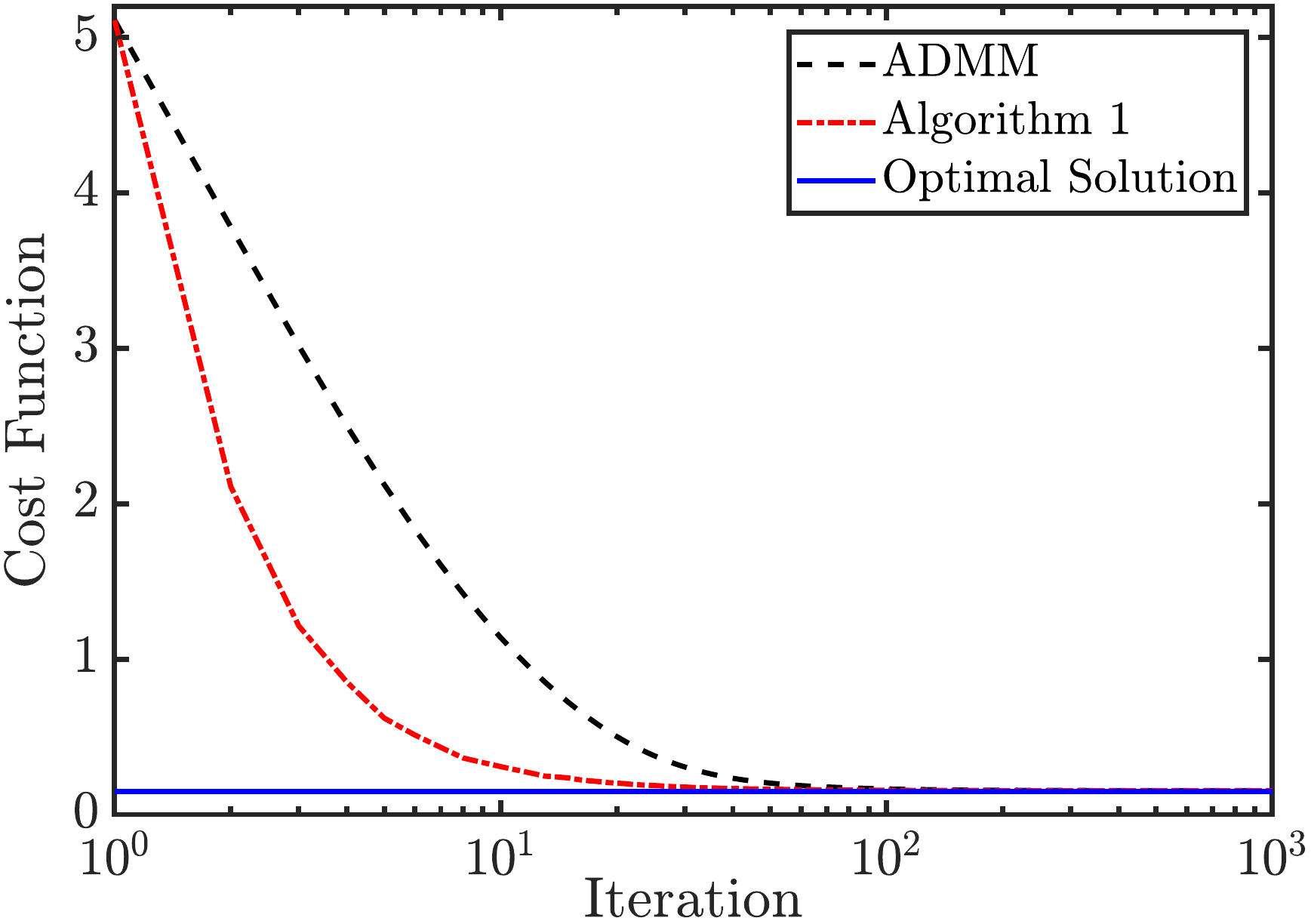}
\caption{Comparison of Algorithm~\ref{alg:greedy algorithm} and the ADMM algorithm by \cite{Condat_2017}, using the optimal solution obtained by CVX~\cite{cvx}.}
\label{fig:algorithm1}
\vspace{-1em}
\end{figure}
\end{example}

\subsection{Geometric Interpretation of \eref{eq:optimization}}
\noindent \textbf{Uniqueness}. The uniqueness of the solution of \eref{eq:optimization} is determined by the positive definiteness of $\mSigma$. If $\mSigma$ is positive definite, then \eref{eq:optimization} is strictly convex, and hence the optimal weight is unique. If $\mSigma$ is only positive semi-definite, then there are infinitely many optimal weights. The following proposition explains this phenomenon.
\begin{proposition}
Suppose that $\mSigma$ is positive semi-definite. Let $\vw_1^*$ and $\vw_2^*$ be two solutions of \eref{eq:optimization}. Then, for any $0 \le t \le 1$, the vector $\vw^* \bydef t\vw_1^* + (1-t) \vw_2^*$ is also a solution of \eref{eq:optimization}.
\label{proposition: uniqueness}
\end{proposition}
\begin{proof}
Let $f(\vw) = \vw^T\mSigma\vw$. Since both $\vw_1^*$ and $\vw_2^*$ are solutions to \eref{eq:optimization}, we have $f(\vw_1^*) = f(\vw_2^*)$. Also, by linearity, we have that $\vone^T\vw^* = 1$ and $\vw^* \succeq 0$. Since $f$ is convex, we can show that
\begin{align*}
f(\vw^*) &= f(t \vw_1^* + (1-t)\vw_2^*) \\
&\le t f(\vw_1^*) + (1-t) f(\vw_2^*) = f(\vw_1^*).
\end{align*}
But since $\vw_1^*$ is an optimal solution, it is impossible for $f(\vw^*) < f(\vw_1^*)$. So the only possibility is $f(\vw^*) = f(\vw_1^*)$. This implies that $\vw^*$ is also a solution.
\end{proof}
The implication of Proposition~\ref{proposition: uniqueness} is that if two initial estimates $\vzhat_i$ and $\vzhat_j$ are identical (or scalar multiple of one and the other), then $\mSigma$ will have dependent columns (hence positive semi-definite). When this happens, there will be infinitely many ways of combining the two initial estimates. However, in practice this is not an issue because even if the pair $(w_i^*, w_j^*)$ is not unique, the combined estimate $w_i^* \vzhat_i + w_j^*\vzhat_j$ remains unique when $\vzhat_i=\vzhat_j$.

\vspace{1ex}
\noindent \textbf{Geometry}. The geometry of \eref{eq:optimization} can be interpreted in low dimensions, e.g., \fref{fig:geometry}. In this figure, we consider a 2D case so that $\mSigma$ is a $2 \times 2$ matrix. We can show that the ellipse always has its minor axis pointing to the northeast direction if the two initial estimates are positively correlated.
\begin{figure}[t]
\centering
\begin{tabular}{cc}
	\includegraphics[width=0.32\linewidth]{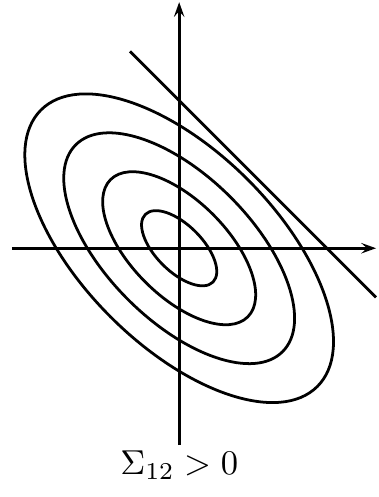}&
	\hspace{8ex} \includegraphics[width=0.32\linewidth]{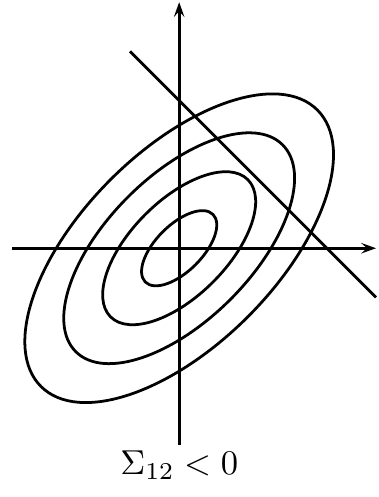}
\end{tabular}
\vspace{-0.5ex}
\caption{Geometry of the optimal weight minimization problem.}
\label{fig:geometry}
\vspace{-2ex}
\end{figure}

\begin{proposition}
\label{prop:geometry}
Consider a two-dimensional $\mSigma$. If $\Sigma_{12} > 0$, then $\mSigma$ always has its minor axis pointing to the northeast direction and major axis to the northwest direction.
\end{proposition}
\begin{proof}
Consider the eigen-decomposition of $\mSigma = \mU\mS\mU^T$. For a $2\times 2$ matrix, classical results in matrix analysis \cite{Deledalle_Denis_Tabti_2017} shows that the eigen-value and eigen-vectors are
\begin{align*}
s_1 = \frac{1}{2}\left( \Sigma_{11} + \Sigma_{22} - \lambda \right), \quad s_2 = \frac{1}{2}\left( \Sigma_{11} + \Sigma_{22} + \lambda \right),
\end{align*}
and
\begin{align*}
\vu_1 =
\begin{bmatrix}
\frac{\Sigma_{11}-\Sigma_{22}+\lambda}{2\Sigma_{12}}\\
1
\end{bmatrix}
, \quad
\vu_2 =
\begin{bmatrix}
\frac{\Sigma_{11}-\Sigma_{22}-\lambda}{2\Sigma_{12}}\\
1
\end{bmatrix}
\end{align*}
where $\lambda = \sqrt{4\Sigma_{12}^2 + (\Sigma_{11} - \Sigma_{12})^2}.$

Note that $\lambda \ge |\Sigma_{11}-\Sigma_{22}|$ because $\Sigma_{12}^2 \ge 0$. Therefore, $s_2 \ge s_1$ and so $\vu_1$ is the minor axis and $\vu_2$ is the major axis. The numerator of the first entry of $\vu_1$ is
\begin{align*}
\Sigma_{11} - \Sigma_{22} + \lambda
&\ge \Sigma_{11} - \Sigma_{22} + |\Sigma_{11} - \Sigma_{22}| \\
&= \begin{cases}
2 |\Sigma_{11}-\Sigma_{22}| \ge 0, &\mbox{if }\; \Sigma_{11} \ge \Sigma_{22},\\
0, &\mbox{otherwise}.
\end{cases}
\end{align*}
As a result, the numerator of the first entry of $\vu_1$ is always non-negative, implying that the sign of the denominator determines the sign of the entry. Therefore, if $\Sigma_{12} > 0$, then $\vu_1$ will be pointing to the northeast direction. By orthogonality of the eigen-vectors, $\vu_2$ points to the northwest direction.
\end{proof}

Proposition~\ref{prop:geometry} provides some insights about the solution. If $\Sigma_{12}>0$ (which is usually the case), the major axis must point to northwest. Therefore, the solution is more likely to be at one of the two vertices. In other words, the optimal solution tends to be \emph{sparse}. Such sparsity should come with no surprise, because the linear constraint $\vw^T\vone = 1$ is equivalent to $\|\vw\|_1 = 1$ if $\vw \succeq 0$. This also explains why the non-negativity constraint in our problem is essential.

\begin{remark}
In practice, if we only have an estimated covariance matrix $\mSigmatilde$, there is no guarantee that $\mSigmatilde$ is positive semi-definite. (Symmetry can be preserved by constructing the off-diagonals using \eref{eq:Sigmaij}.) When $\mSigmatilde$ is not positive semi-definite, we project $\mSigmatilde$ onto its closest positive semi-definite matrix by solving
\begin{equation}
\mSigma = \argmin{\mS \succeq 0} \;\|\mS - \mSigmatilde\|_F^2.
\label{eq:positive definite S}
\end{equation}
The solution to \eref{eq:positive definite S} is the truncated eigen-decomposition where negative eigenvalues of $\mSigmatilde$ are set to 0.
\end{remark}

\subsection{Optimal MSE Lower Bound}
We derive the MSE lower bound of \eref{eq:optimization}. To do so, we consider a relaxed optimization by removing the non-negativity constraint:
\begin{equation}
\begin{array}{ll}
\minimize{\vw} &\;\; \vw^T \mSigma \vw  \\
\subjectto     &\;\; \vw^T\vone = 1.
\end{array}
\tag{$P_2$}
\label{eq:optimization 2}
\end{equation}
Clearly, the feasible set of \eref{eq:optimization 2} includes the feasible set of \eref{eq:optimization}, and so the MSE obtained by solving \eref{eq:optimization 2} must be a lower bound of the MSE obtained by solving \eref{eq:optimization}. More precisely, if we let $\vwhat$ be the optimal weight vector obtained by \eref{eq:optimization}, and $\vw^*$ be that obtained by \eref{eq:optimization 2}, then
\begin{equation}
\E \left[ \left \|\mZhat\vwhat - \vz \right \|^2 \right] \ge \E \left[ \left \|\mZhat\vw^* - \vz \right \|^2 \right].
\label{eq:MSE lower bound}
\end{equation}

Let us analyze the right hand side of \eref{eq:MSE lower bound}. The optimization in \eref{eq:optimization 2} is a standard linear equality constrained quadratic minimization. Closed-form solution can be derived via the standard Lagrangian approach by defining:
\begin{equation}
\calL(\vw, \lambda) = \frac{1}{2} \vw^T\mSigma\vw - \lambda (\vw^T\vone - 1).
\end{equation}
The first order KKT conditions state that
\begin{equation*}
\frac{\partial \calL}{\partial \vw} = 0, \quad\quad \vw^T\vone = 1,
\end{equation*}
where the first condition is equivalent to
\begin{equation}
\mSigma\vw - \lambda\vone = 0, \quad\mbox{or}\quad \vw = \lambda \mSigma^{\dag} \vone,
\label{eq:optimal solution 1}
\end{equation}
where $\mSigma^{\dag}$ denotes the pseudo-inverse of a symmetric positive semi-definite matrix $\mSigma$. If $\mSigma$ is positive definite, then $\mSigma^{\dag} = \mSigma^{-1}$ and \eref{eq:optimal solution 1} can be written as $\vw = \lambda \mSigma^{-1} \vone$. Substituting \eref{eq:optimal solution 1} into the constraint, we have that
\begin{align}
\vone^T \left ( \lambda \mSigma^{\dag} \vone \right) = 1 \quad\Rightarrow\quad \lambda = \frac{1}{\vone^T \mSigma^\dag \vone}.
\label{eq:optimal lambda}
\end{align}
Substituting \eref{eq:optimal lambda} into \eref{eq:optimal solution 1}, we prove the following.

\begin{proposition}
The solution to \eref{eq:optimization 2} is given by
\begin{equation}
\vw^* = \frac{\mSigma^{\dag} \vone}{\vone^T \mSigma^\dag \vone},
\label{eq:optimal solution}
\end{equation}
where $\mSigma^\dag$ denotes the pseudo-inverse of the symmetric positive semi-definite matrix $\mSigma$.
\end{proposition}

Given the optimal weight vector $\vw^*$, we can determine the corresponding mean squared error:
\begin{align}
\E \left[ \left \|\mZhat\vw^* - \vz \right \|^2 \right] = (\vw^*)^T \mSigma \vw^* = \frac{1}{\vone^T \mSigma^\dag \vone}.
\label{eq:MSE lower bound 2}
\end{align}
Since the weight $\vw^*$ provides a lower bound on the MSE, in particular if we consider a weight vector $\ve_k = [0,\ldots,1,\ldots,0]^T$ (i.e., the $k$-th standard basis vector), we must have
\begin{align}
\mathrm{MSE}_k = \ve_k^T \mSigma \ve_k \ge \vwhat^T \mSigma \vwhat \ge \frac{1}{\vone^T \mSigma^\dag \vone}.
\label{eq: lower bound result}
\end{align}
The first inequality holds because $\ve_k$ is one of the feasible vectors of \eref{eq:optimization} but $\vwhat$ is the optimal solution. The second inequality holds because $\vw^*$ is a solution of \eref{eq:optimization 2}. The result of \eref{eq: lower bound result} states that an optimally combined estimate using $\vwhat$ has to be at least as good as any initial estimate.

\begin{remark}
The MSE lower bound result presented here is more general than the previous result by Odell et al. \cite{Odell_Dorsett_Young_1989} which only considered $K = 2$. When $K = 2$, we have
\begin{equation}
w_1^* = \frac{\Sigma_{22} - \Sigma_{12}}{\Sigma_{11} + \Sigma_{22} - 2\Sigma_{12}}, \quad\mbox{and}\quad w_2^* = 1-w_1^*,
\end{equation}
which is the same as Equation 2 of Table 3 in \cite{Odell_Dorsett_Young_1989}. \footnote{In Equation 2 of Table 3 in \cite{Odell_Dorsett_Young_1989}, there is a typo of the numerator which should be corrected as $m_{22}-m_{12}$.}
\end{remark}

\subsection{Perturbation in $\mSigma$}
We conclude this section by discussing the perturbation issue when we use an estimated covariance matrix $\mSigmatilde$ instead of $\mSigma$. To facilitate the discussion, we define two weight vectors:
\begin{align}
\vwtilde = \argmin{\vv \in \Omega} \;\; \vv^T \mSigmatilde \vv, \quad\mbox{and}\quad \vw = \argmin{\vv \in \Omega} \;\; \vv^T \mSigma \vv.
\end{align}
That is, $\vwtilde$ is the optimal weight vector found according to the estimated covariance matrix $\mSigmatilde$, and $\vw$ is the optimal weight vector found according to the true covariance matrix $\mSigma$. Correspondingly, we define their combined estimates as
\begin{equation}
\vztilde = \mZhat\vwtilde, \quad\mbox{and}\quad \vzhat = \mZhat\vw.
\end{equation}
The following proposition summarizes the perturbation result.
\begin{proposition}
\label{prop:perturbation}
Assume that $\mSigmatilde \succ 0$ and $\mSigma \succ 0$. Then,
\begin{equation}
\E \|\vztilde - \vzhat \|^2 \le \E\|\vzhat - \vz\|^2 (2\Delta + \Delta^2),
\label{eq:pertubation theorem}
\end{equation}
where
\begin{equation*}
\Delta \bydef \|\mSigmatilde\mSigma^{-1}-\mSigmatilde^{-1}\mSigma\|_2.
\end{equation*}
\end{proposition}

\begin{proof}
The proof is given in the Appendix. Our proof simplifies the multi-block concept of \cite{Lavancier_Rochet_2016}. We also utilize the generalized Rayleigh quotient idea to obtain the bound.
\end{proof}

The implication of Proposition~\ref{prop:perturbation} can be seen from the two terms on the right hand side of \eref{eq:pertubation theorem}. First, $\E\|\vzhat - \vz\|^2$ measures the bias between the oracle combination $\vzhat$ and the ground truth $\vz$. That it is an upper bound in \eref{eq:pertubation theorem} implies that the perturbed estimate is upper limited by the bias. The second term $\Delta$ measures the closeness between the oracle covariance $\mSigma$ and the estimated covariance $\mSigmatilde$. If $\mSigma\mSigmatilde^{-1} = \mI$, then $\Delta = 0$ and so the perturbation is minimized. In practice, if $\mSigmatilde$ can be estimated in $n$ random trials and if $\mSigma\mSigmatilde_n^{-1} \overset{p}{\rightarrow} \mI$ as $n \rightarrow \infty$, then we can also show that $\Delta \overset{p}{\rightarrow} 0$. (For example, use SURE on multiple noisy observations, if available.)

\section{MSE Estimator}
\label{sec:covariance estimator}
The key to make ($P_1$) succeed is an accurate covariance matrix $\mSigma$. Estimating the covariance matrix requires estimating the mean squared error (MSE). In this section we discuss a neural network solution.

\subsection{Why not SURE?}
In image processing, perhaps the most popular approach to estimate MSE is the Stein's Unbiased Risk Estimator (SURE). (See, e.g., \cite{Blu2007,Luisier2007} for illustrations, \cite{Ramani2008} for a Monte-Carlo version, and \cite{Cha2018} for a recent work using SURE in deep neural network.) As its name suggested, SURE is an unbiased estimator of the true MSE, i.e., the estimator will approach to the true MSE as the number of samples grows.

While SURE-based estimators work well in ideal situations, it also has many shortcomings:
\begin{itemize}
\item Large Variance. SURE only provide \emph{average performance} guarantee. For Monte-Carlo SURE, there is another level of randomness due to the Monte-Carlo scheme. Therefore, given a single noisy image, SURE can be inaccurate, especially for non-linear denoises such as BM3D.
\item Clipped Noise. SURE is designed to handle additive i.i.d. Gaussian noise. However, most real images are clipped to $[0,1]^N$. Most neural network denoisers also clip the signal to stabilize training. If the observed image is clipped, then SURE will fail \cite{Foi_2009}.
\item Beyond Denoisers. While SURE is a good choice for image denoising problems, one has to re-derive the SURE equations for different forward models, e.g., deblurring or super-resolution. This severely limits the generality of the present optimal combination framework.
\end{itemize}

To illustrate the problems of SURE, we conduct two experiments comparing SURE and the proposed neural network approach. The task of the experiments is to denoise the \texttt{cameraman256} image, corrupted by i.i.d. Gaussian noise of different noise levels. In the first experiment, the i.i.d. Gaussian noise is unclipped so that the theory of SURE applies. The result of this experiment is shown in \fref{fig:SURE vs MSE}(a). The average of SURE (over 100 random trials of different noise realizations) is very similar to the true MSE, something we expect from the theory. However, the variance of SURE is big; indeed very big. If we use SURE to construct a $\mSigma$, the resulting $\mSigma$ can be bad.

The second experiment modifies the i.i.d. Gaussian noise to clipped Gaussian so that the resulting signal is bounded to $[0,1]$. We argue that the clipped noise is more realistic because no physical sensor can produce a signal level below 0 or beyond 1. When the noise is clipped, the symmetry of Gaussian distribution is destroyed and the clipping is signal dependent. As a result, the MSE predicted by SURE is significantly off from the theory. \fref{fig:SURE vs MSE}(b) illustrates the result. SURE produces a completely opposite trend of the MSE whereas the NN produces a more reasonable estimate.

\begin{figure}[h]
\centering
\begin{tabular}{c}
\includegraphics[width=0.8\linewidth]{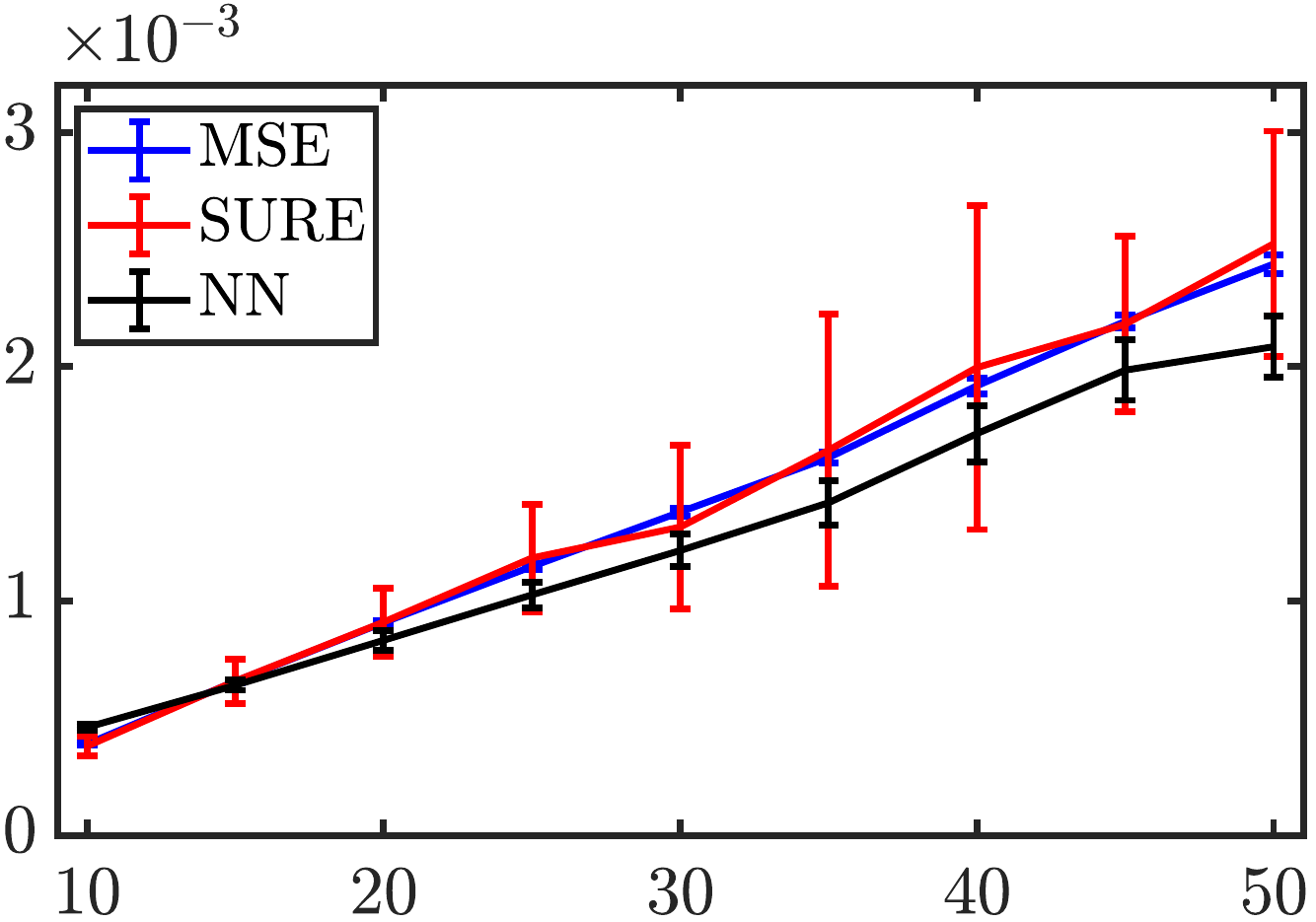}\\
\footnotesize{(a)}\\
\includegraphics[width=0.8\linewidth]{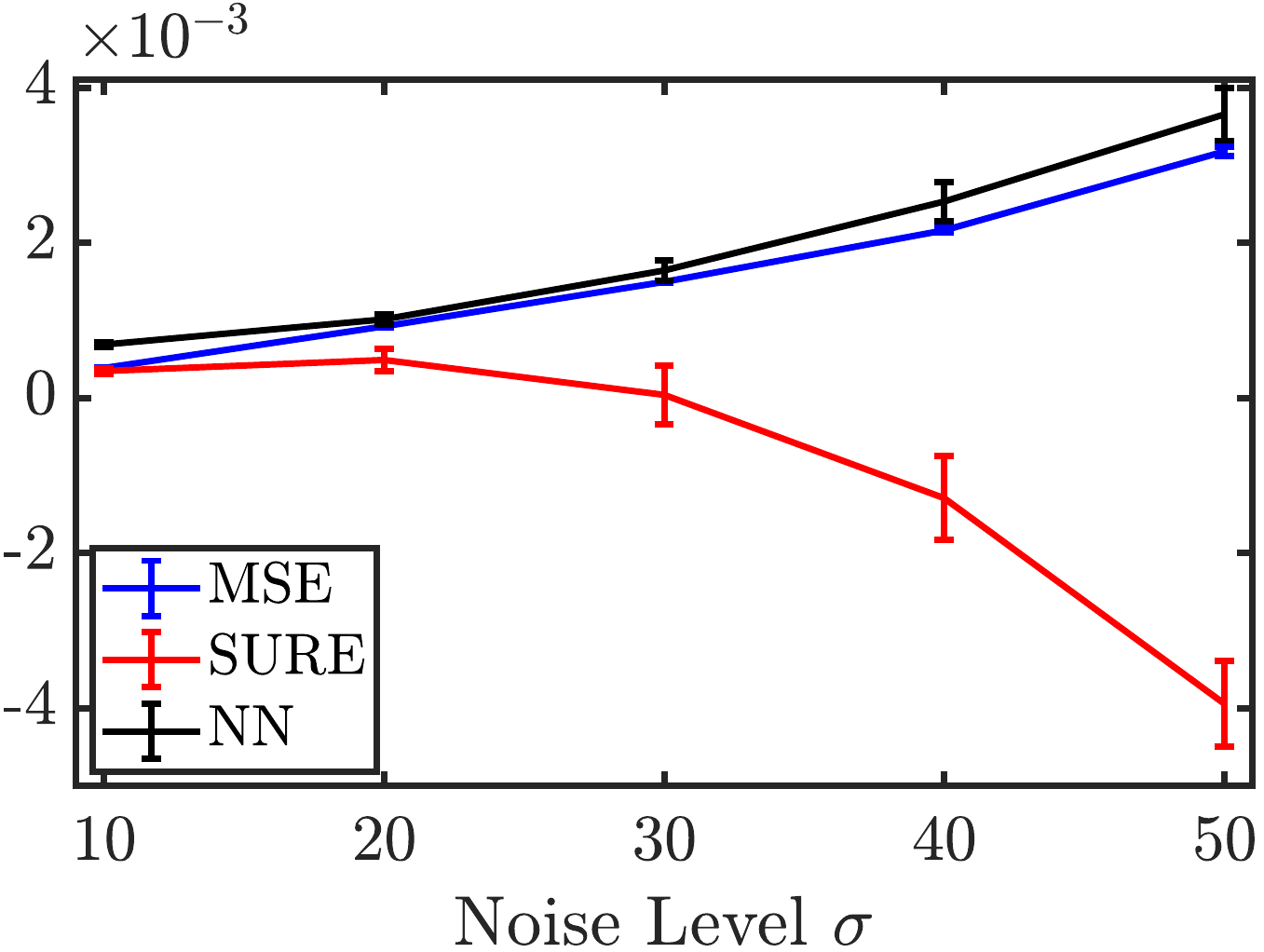}\\
\footnotesize{(b)}
\end{tabular}
\caption{(a) Unclipped and (b) Clipped Noise Examples. Compare SURE and the proposed neural network (NN) on estimating the MSE. In this experiment, we use BM3D to denoise the \texttt{cameraman} image. The noise level changes from $\sigma = 10$ to $\sigma = 50$. The observed images are clipped to $[0,1]^N$. The error bars are computed using 50 random trials of the i.i.d. Gaussian noise realizations. Dotted lines indicate the max and min of the realizations.}
\label{fig:SURE vs MSE}
\end{figure}

\subsection{Neural Network MSE Estimator}
Our proposed solution is a deep neural network MSE estimator. Using deep neural networks for image quality assessment is an active research topic  \cite{Kang2014,Li_Po_Feng_2016,Bosse_Maniry_Wiegand_2016,Kim_Zeng_Ghadiyaram_2017,Li_Ye_Li_2017}. However, the existing neural network based image quality assessment methods are tailored to predict the human visual system responses when presenting an image to a user. A pure MSE estimator is not common. To the best of our knowledge, the only existing MSE estimator is \cite{Bosse_Maniry_Wiegand_2016}. However, the MSE estimator in \cite{Bosse_Maniry_Wiegand_2016} is used to quantify \emph{noisy} images, i.e., the amount of noise. An MSE estimator for \emph{denoised} images does not currently exist.

The proposed neural network based MSE estimator is shown in \fref{fig:mse_estimator}. There are two unique features of the network. First, the input to the network is a pair of images $(\vy,\vzhat_k)$, i.e., the noisy observation and the $k$-th denoised image. Using both $\vy$ and $\vzhat_k$ is reminiscent to the SURE approach, as $\vy$ provides noise statistics that cannot be obtained from $\vzhat_k$ alone.

Second, instead of feeding the entire image into the network, we partition the image into non-overlapping patches of size $64\times64$. That is, if we denote the MSE of the $i$-th patch of the $k$-th denoiser as $\widetilde{\MSE}_{k,i} \bydef \widetilde{\MSE}(\vy_i, \vzhat_{k,i})$, then the overall MSE of the $k$-th denoiser is
$$
\widetilde{\MSE}_k = \frac{1}{M}\sum_{i=1}^M \widetilde{\MSE}(\vy_i, \vzhat_{k,i}),
$$
where $\vy_i$ is the $i$-th patch of $\vy$, $\vzhat_{k,i}$ is the $i$-th patch of $\vzhat_k$, and $M$ is the number of non-overlapping patches in the image. Partitioning the image into small patches reduces the breath and depth of the neural network.

The network consists of 8 convolutional layers, 3 maxpool layers and 2 fully connected layers. The inputs to the network are the $i$-th noisy patch $\vy_i$ and the $i$-th denoised patch $\vzhat_{k,i}$ of the $k$-th denoiser. The patches separately pass through two convolutional layers, and then concatenate and pass over four convolutional layers. The convoultional layers use $3\times 3$ kernels with zero-padding and the rectifier activation function (ReLU). We apply maxpool layer with $2 \times 2$ kernel every two convoultional layer. Fully connected layers use ReLU and dropout regularization of ratio 0.5. The cost function is the $L_1$-loss, defined as
\begin{equation}
\label{eq:mse_loss}
L = \abs{\MSE_{k,i} - \widetilde{\MSE}_{k,i}}
\end{equation}
where $\MSE_{k,i}$ is the true MSE of $i$-th block of the $k$-th denoiser. For implementation, we use ADAM optimizer \cite{Kingma2014} with learning rate $\alpha=10^{-4}$.

\begin{figure*}[t]
	\centering
	\includegraphics[width=0.8\linewidth]{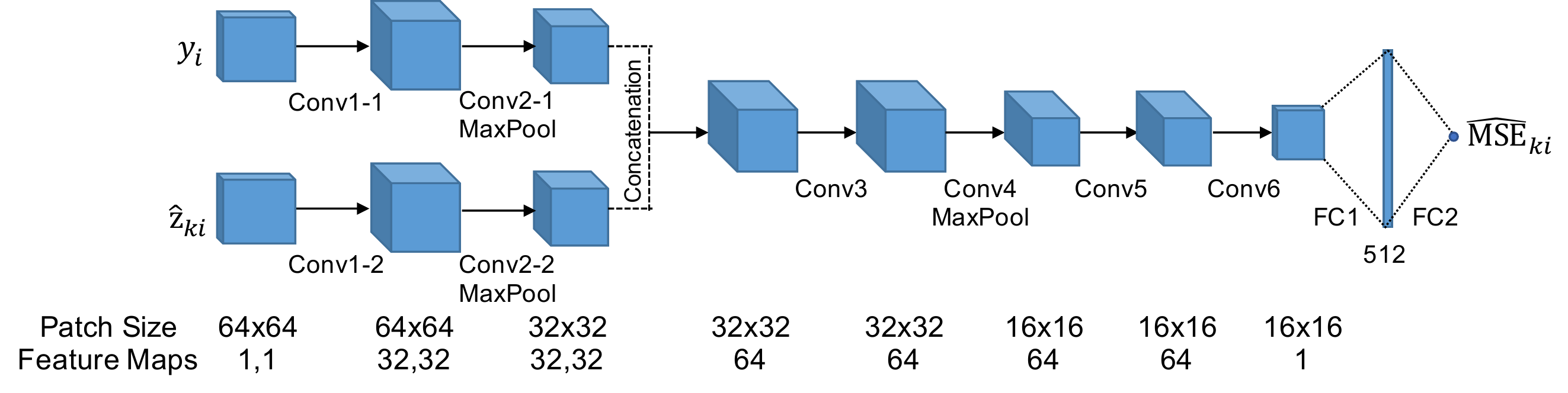}
	\caption{Network structure of a proposed MSE Estimator.}
	\label{fig:mse_estimator}
\end{figure*}

The training data we use is the 300 Training and Validation images in BSD500. For each image, we randomly extracted 32 patches of size $64 \times 64$ and generate 6 variations by flipping horizontally and vertically and rotating at $0^\circ$, $90^{\circ}$, $180^{\circ}$ and $270^{\circ}$. The noise level is $\sigma \in [1,60]$, with clipping to $[0,1]^N$. To prepare denoised images for training the networks, we use five pre-trained REDNets \cite{Mao2016} at noise levels $\sigmahat = 10, 20, 30, 40, 50$. Therefore, for every noisy input we generate multiple denoised images, and every denoised image forms an input-output pair with the ground truth MSE. We trained the MSE estimator network with 100 epochs for around 7 hours.

\subsection{Comparison with SSDA}
Readers familiar with the image denoising literature may ask about the difference between the proposed method and the AMC-SSDA method by Agostinelli et al. \cite{Agostinelli2013} (or SSDA in short). The SSDA method is an end-to-end neural network for denoising images of different noise types, e.g., salt-pepper, Gaussian, and Poisson. We are not interested in this problem because it is less common to have an image denoising problem where the noise type is totally blind. In contrast, it is more likely to have multiple denoisers for different noise levels (Section~\ref{sec:experiment 1}), different image classes (Section~\ref{sec:experiment 2}), and different denoiser types (Section~\ref{sec:experiment 3}).

There are other differences. First, the SSDA has a set of fixed neural network denoisers. In contrast, CsNet can support \emph{any} initial denoisers. Second, the weight prediction of the SSDA is done using a neural network which does not have optimality guarantee. CsNet, however, is optimal in the MMSE sense. Additionally, CsNet estimates the MSE (which is a scalar) from an image. This is easier than estimating the weight vector in SSDA. Third, CsNet can be generalized to other estimation problems such as deblurring and super-resolution. SSDA, however, has limited generalization capability because the initial estimators are limited to SSDA.

\section{Booster Network}
\label{sec:booster}
In our proposed CsNet, besides the convex optimization algorithm and the MSE estimator, there is a third component known as the booster. The booster is used to improve the combined estimates by enhancing the contrast and to recover lost details. To provide readers a quick preview of the booster, we show a few examples in \fref{fig:booster}.

\subsection{What is a Booster?}
The concept of boosting can be traced back to at least the 70's, when Tukey \cite{Tukey_1977} proposed a ``twicing procedure''. In machine learning, the same concept was studied by B\"{u}hlmann and Yu \cite{Buhlmann_Yu_1998}. The essential step of boosting is simple: Given a current estimate $\vzhat^{(t)}$ and the observation $\vy$, we construct a mapping $\calB: \R^N \rightarrow \R^N$ (usually another denoising algorithm), and then define the next estimate $\vzhat^{(t+1)}$ in terms of $\vzhat^{(t)}$, $\vy$ and $\calB$ with the goal to improve the MSE. In Tukey's ``twicing'', the relationship between $\vzhat^{(t)}$ and $\vzhat^{(t+1)}$ is
\begin{equation}
\vzhat^{(t+1)} = \calB(\vy - \vzhat^{(t)}) + \vzhat^{(t)}.
\label{eq:twicing}
\end{equation}
Thus, if $\calB$ is a denoiser, then $\calB(\vy - \vzhat^{(t)})$ is the filtered version of the residue. As shown in \cite{Charest_Milanfar_2008}, MSE is not monotonically decreasing as $t \rightarrow \infty$ because of the bias-variance trade-off. However, with proper monitoring such as cross-validation, MSE can be minimized by stopping the boosting procedure before saturation. (See additional discussion for the image denoising problem in \cite{Talebi_Zhu_Milanfar_2013}.)


In the image denoising literature, the above idea of boosting has been studied in multiple places such as \cite{Charest_Elad_Milanfar_2006,Charest_Milanfar_2008,Talebi_Zhu_Milanfar_2013}. There are several variations, e.g., Osher's iterative regularization \cite{Osher_Burger_Goldfarb_2005}, and Romano and Elad's SOS \cite{Romano2015}. In all these boosting methods, the idea is the take the noisy input and the estimate $\vzhat^{(t)}$ to recursively update the estimate. Table~\ref{table:booster} shows a comparison of different denoising boosters.

\begin{table}[h]
\begin{tabular}{cl}
\hline
Method & Idea \\
\hline\hline
Twicing \cite{Tukey_1977, Buhlmann_Yu_1998}   & $\vzhat^{(t+1)} = \calB(\vy - \vzhat^{(t)}) + \vzhat^{(t)}$\\
Osher et al. \cite{Osher_Burger_Goldfarb_2005}        & $\vzhat^{(t+1)} = \calB\left(\vy + \sum_{i=1}^t (\vy - \vzhat^{(t)}) \right)$\\
Charest-Milanfar \cite{Charest_Milanfar_2008}  & $\vzhat^{(t+1)} = \vy + \left(\vzhat^{(t)} - \calB(\vzhat^{(t)})\right)$\\
Talebi-Milanfar \cite{Talebi_Zhu_Milanfar_2013}& $\vzhat^{(t+1)} = \calB(\vy - \vzhat^{(t)}) + \vzhat^{(t)}$\\
Romano-Elad \cite{Romano2015}                  & $\vzhat^{(t+1)} = \calB(\vy + \vzhat^{(t)}) - \vzhat^{(t)}$\\
Proposed & $\vzhat^{(t+1)} = \calB_t(\vy, \vzhat^{(t)}) + \vzhat^{(t)}$\\
\hline
\end{tabular}
\caption{Different denoising boosters in the literature. Our proposed method generalizes the classical boosters by replacing $\calB$ with deep neural networks $\calB_t$. }
\label{table:booster}
\end{table}


\subsection{Deep Learning based Booster}
Our proposed neural network booster is motivated by the above examples of classical boosters. The specific network architecture is shown in \fref{fig:booster_structure}. Instead of using a deterministic function $\calB$, we use a multi-layer neural network as the building block of the booster. We then cascade the building blocks to form the overall booster.

Referring to \fref{fig:booster_structure}, if we denote the $t$-th building block as $\calB_t$, then the input-output relationship of $\calB_t$ is
\begin{equation}
\vzhat^{(t+1)} = \calB_t(\vy, \vzhat^{(t)}) + \vzhat^{(t)}.
\label{eq:twicing, NN}
\end{equation}
Clearly, \eref{eq:twicing, NN} is a generalization of \eref{eq:twicing} as $\calB_t$ now becomes a nonlinear mapping trained from the data. Also, when cascading a sequence $\{\calB_t\}$, we generalize \eref{eq:twicing} by allowing each $\calB_t$ to have its own network weights.

\begin{figure}[!]
	\centering
		\includegraphics[width=1\linewidth]{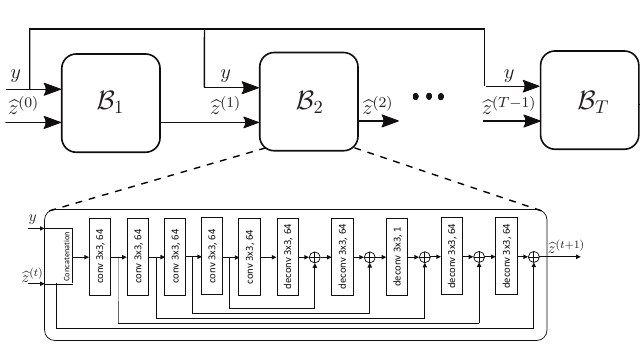}
	\caption{Network structure of the proposed booster network. The network contains 5 convolutional layers followed by 5 deconvolutional layers. Convolutional and deconvolutional layers consists of residual neural network blocks. Skip connections are used to enforce symmetry of the network. This network is repeated five times, i.e., $T=5$.}
	\label{fig:booster_structure}
	\vspace{-1.5em}
\end{figure}

The architecture of the $t$-th building block $\calB_t$ consists of 5 convolutional layers followed by 5 deconvolutional layers, each using kernels of size $3 \times 3$. The input to the network is the pair $(\vy,\vzhat^{(t)})$, which is concatenated to form a common input. The convolutional layers are used to smooth out the noisy input $\vy$, whereas the deconvolutional layers are used to recover the sharp details. Skip connections are used to ensure that signals are not attenuated as it passes through the layers. Note that we purposely add a skip connection from the input $\vzhat^{(t)}$ to the output $\vzhat^{(t+1)}$ to mimic the addition in \eref{eq:twicing}. We cascade $\calB_t$ for $t = 1,\ldots,T$, where $T$ is typically small ($T = 5$).

\begin{figure*}[!]
\centering
\begin{tabular}{cccccc}
    \multicolumn{2}{c}{\includegraphics[width=5.1cm]{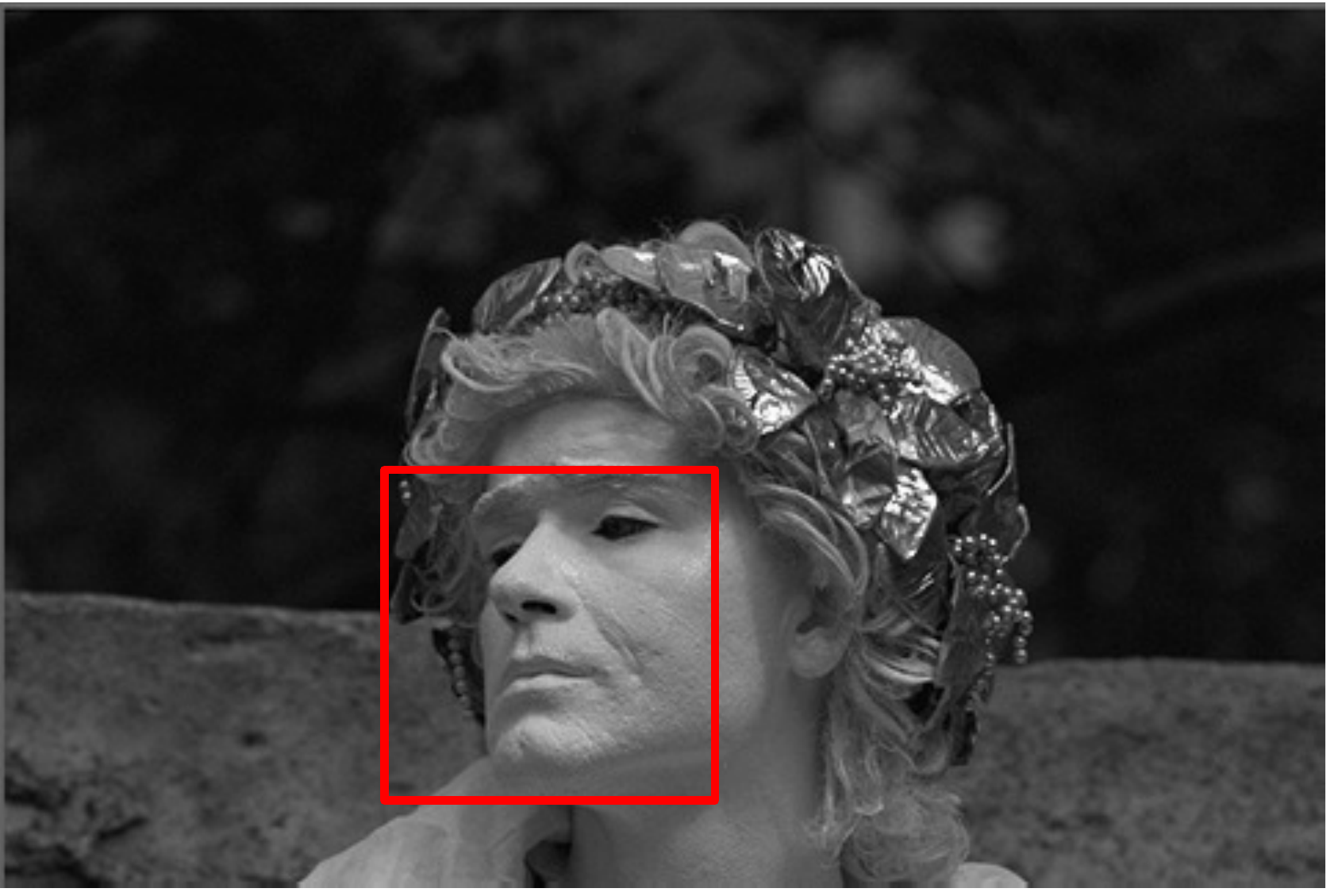}}&
    \multicolumn{2}{c}{\includegraphics[width=5.1cm]{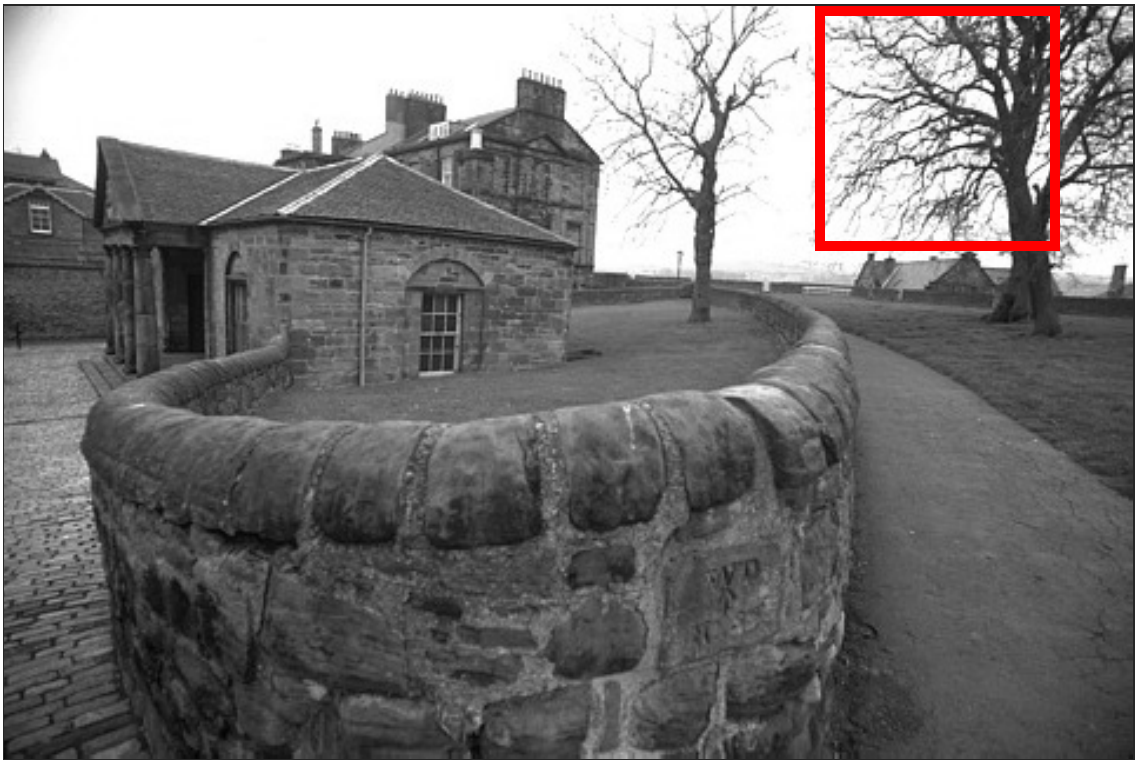}}&
    \multicolumn{2}{c}{\includegraphics[width=5.1cm]{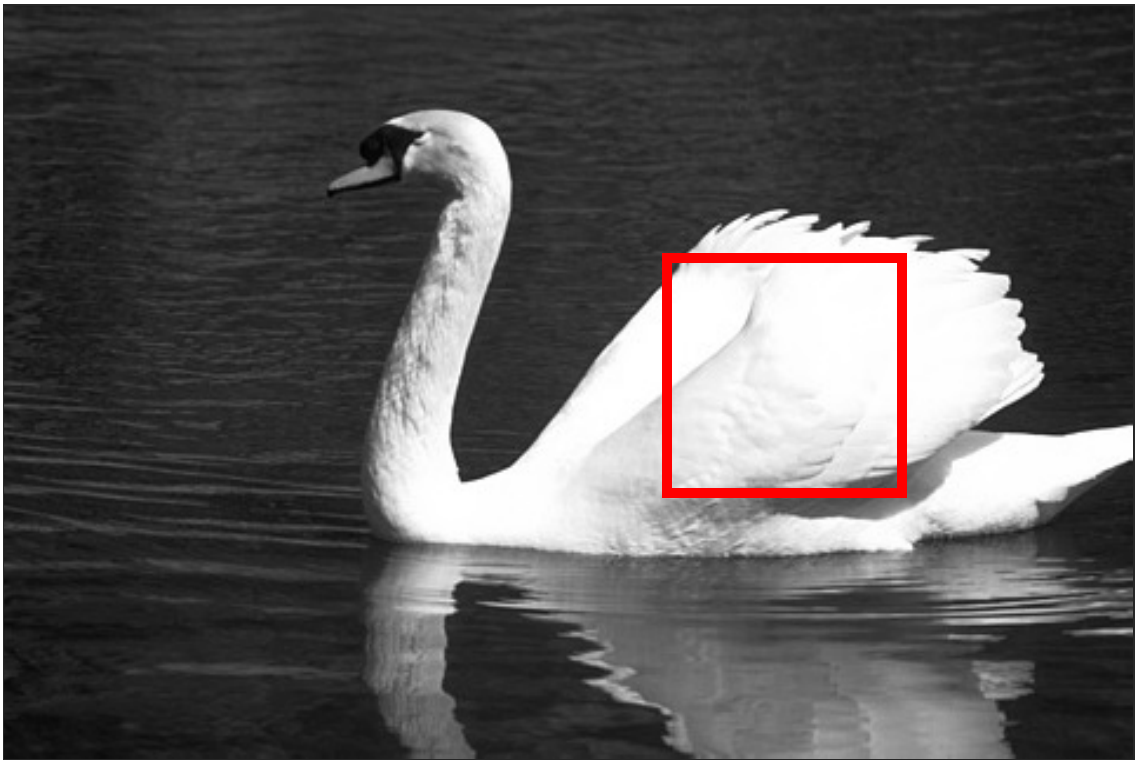}}\\
    \includegraphics[width=2.5cm]{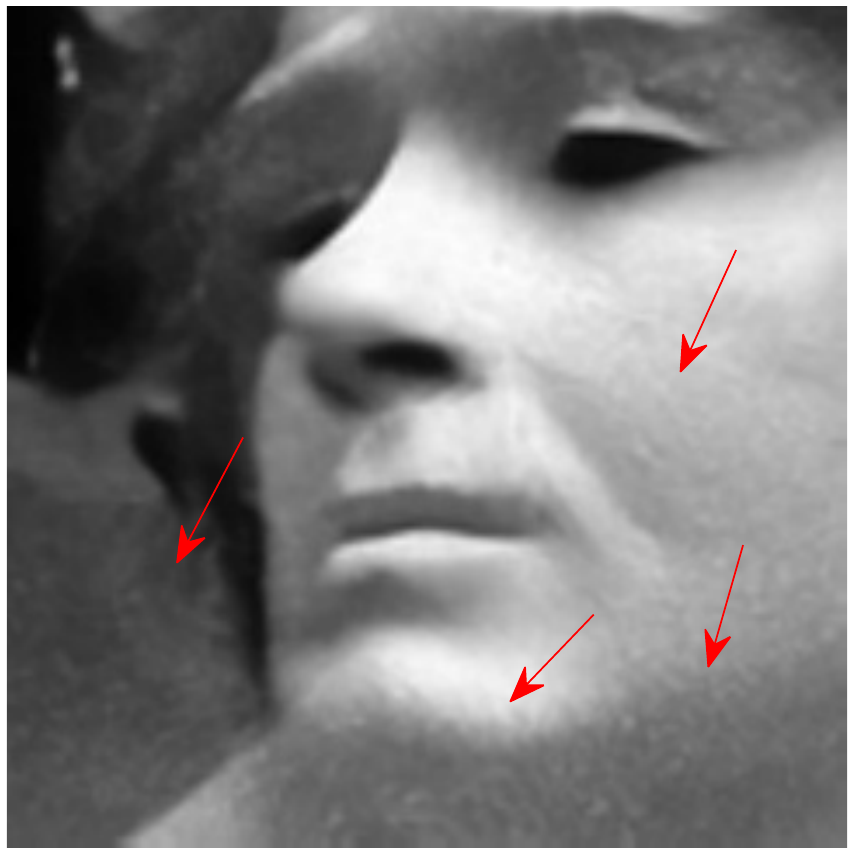} &
    \hspace{-2ex}\includegraphics[width=2.5cm]{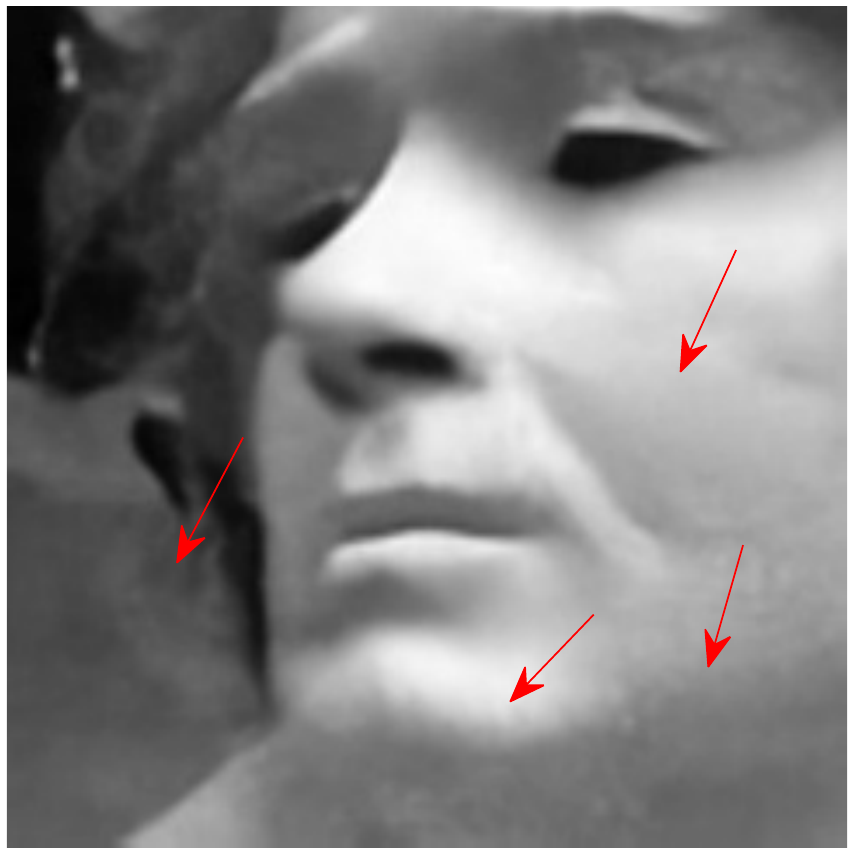}&
    \includegraphics[width=2.5cm]{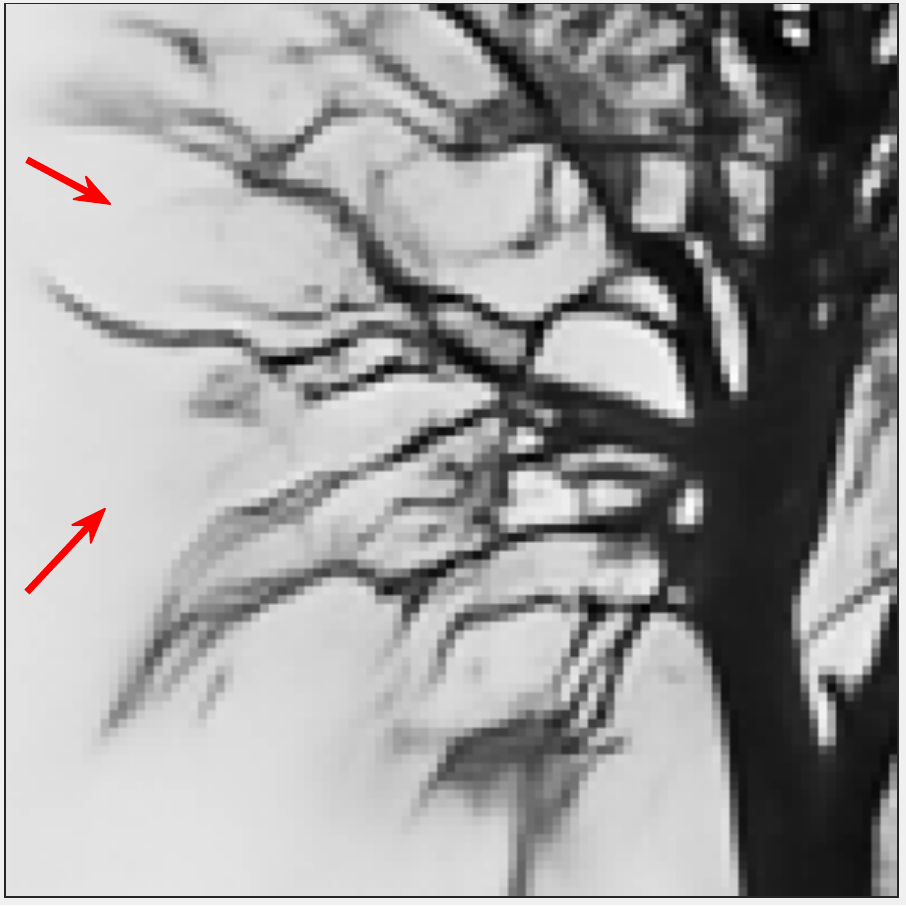} &
    \hspace{-2ex}\includegraphics[width=2.5cm]{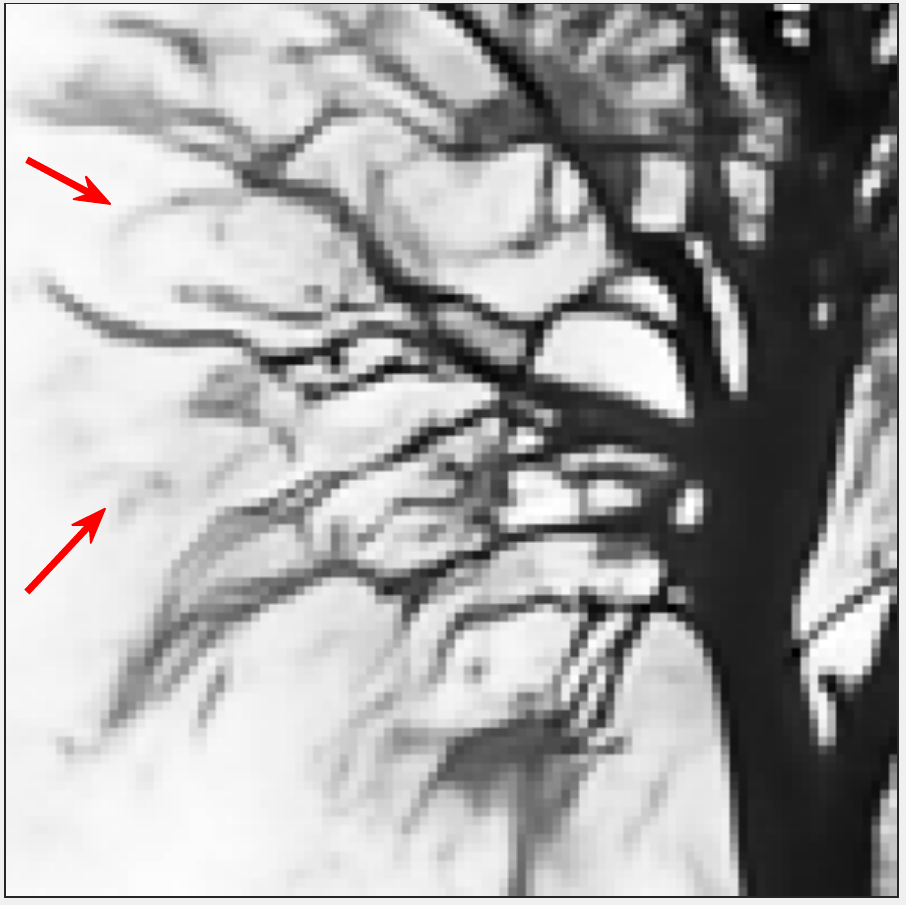}&
    \includegraphics[width=2.5cm]{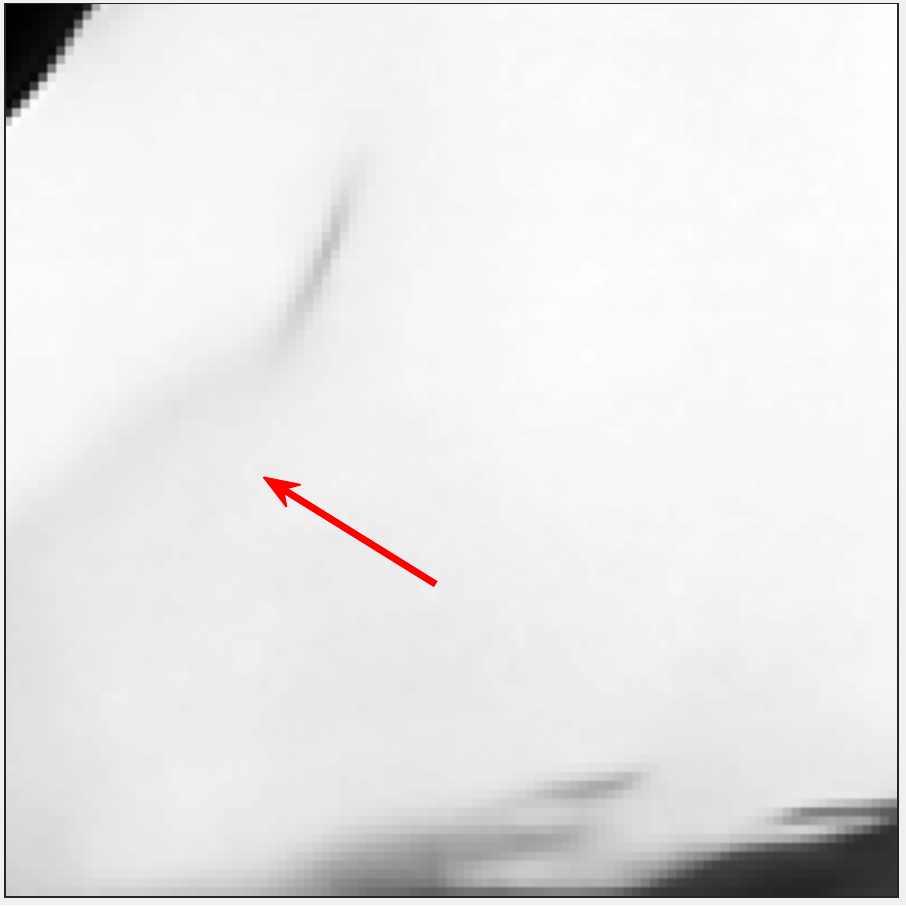} &
    \hspace{-2ex}\includegraphics[width=2.5cm]{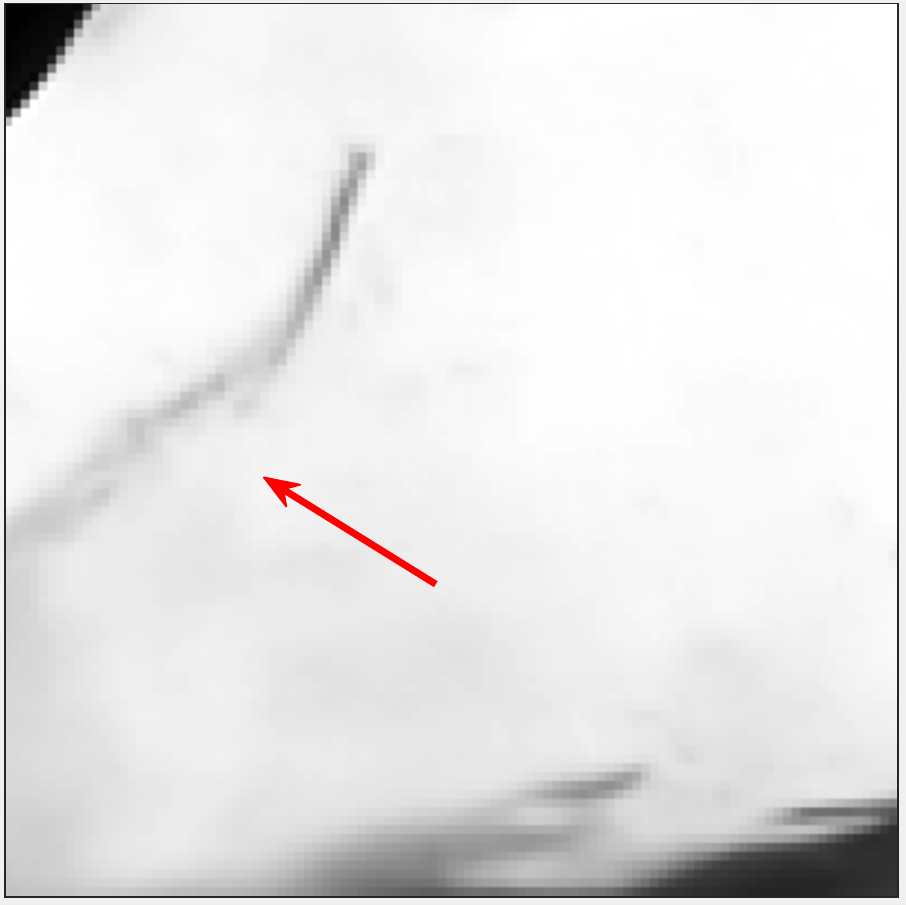}\\
    \footnotesize\begin{tabular}{@{}c@{}}29.08dB, 0.7939 \\  Before Booster\end{tabular}&
    \footnotesize\begin{tabular}{@{}c@{}}29.95dB, 0.8202 \\  After Booster\end{tabular}&
    \footnotesize\begin{tabular}{@{}c@{}}23.55dB, 0.6317 \\  Before Booster\end{tabular}&
    \footnotesize\begin{tabular}{@{}c@{}}24.75dB, 0.6334 \\  After Booster\end{tabular}&
    \footnotesize\begin{tabular}{@{}c@{}}26.63dB, 0.7483 \\  Before Booster\end{tabular}&
    \footnotesize\begin{tabular}{@{}c@{}}29.51dB, 0.7680 \\  After Booster\end{tabular}
\end{tabular}
\caption{Examples showing the effectiveness of the booster in improving the details and contrast of the combined result. See Section \ref{sec:experiment 3} for experiment details.}
\vspace{-3ex}
\label{fig:booster}
\end{figure*}

To train a booster, we feed the booster network with linearly combined estimates and the ground truths. The initial denoisers are the REDNets at different noise levels. The training data we use is the 300 train and validation images in BSD500. We extract 32 patches of size $64 \times 64$ from each training dataset. For each patch we generate 6 variations by flipping horizontally and vertically and rotating at $0^\circ$, $90^{\circ}$, $180^{\circ}$ and $270^{\circ}$. The cost function we use in training the booster network is the standard $L_1$-loss:
\begin{equation}
\label{eq:mse_loss}
L = \norm{\vz - \vzhat^{(T)}}_1
\end{equation}
where $\MSE_{k,i}$ is the true MSE of $i$-th block of the $k$-th denoiser. During the training, we use ADAM optimizer with learning rate $10^{-4}$. We trained booster network with 100 epochs for 12 hours.

\begin{figure*}[!]
	\hfill
	\captionsetup[subfigure]{labelformat=empty}
	\begin{subfigure}[b]{0.19\textwidth}
		\centering
		\includegraphics[width=\textwidth]{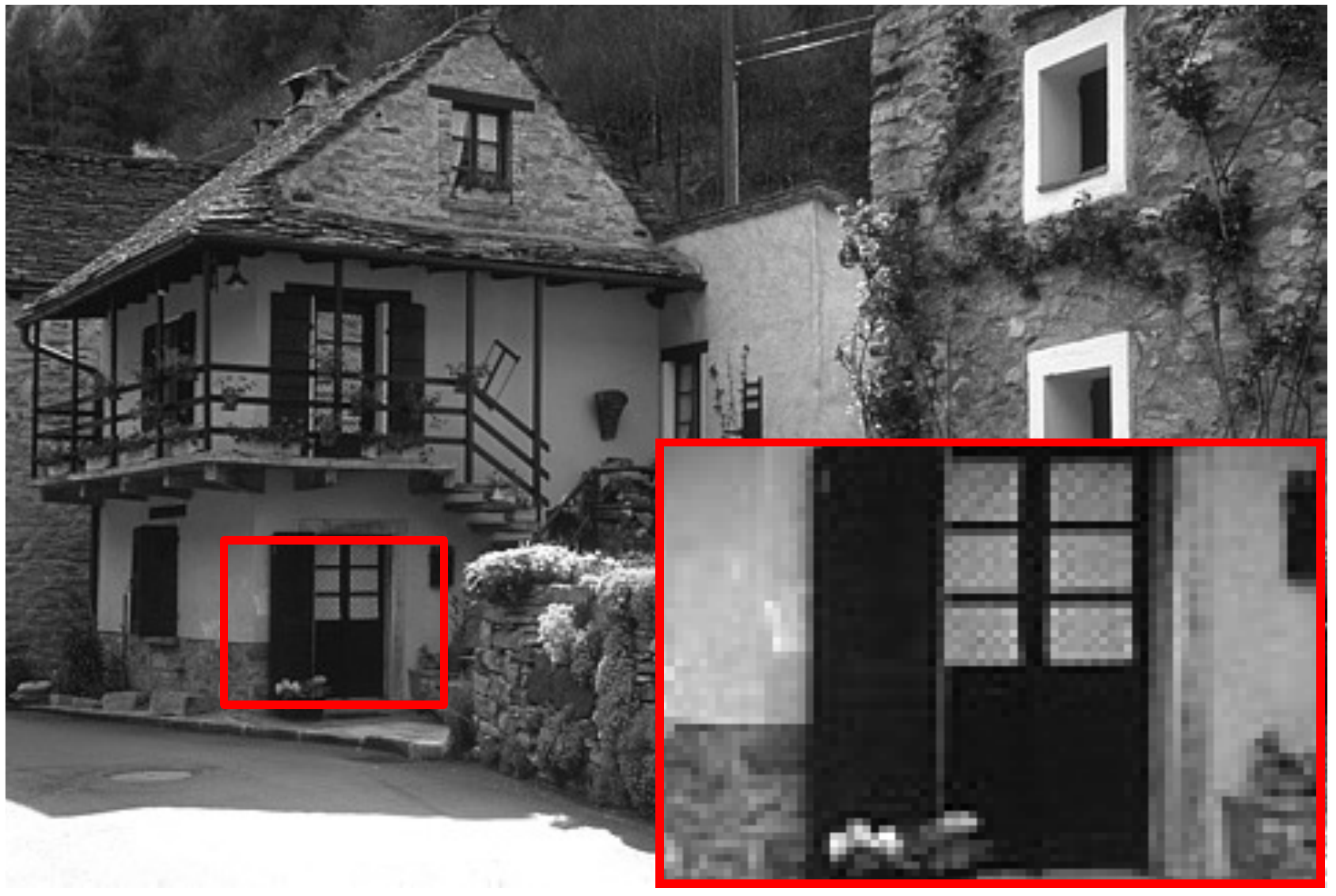}
		\caption{Groundtruth}
	\end{subfigure}
	\begin{subfigure}[b]{0.19\textwidth}
		\centering
		\includegraphics[width=\textwidth]{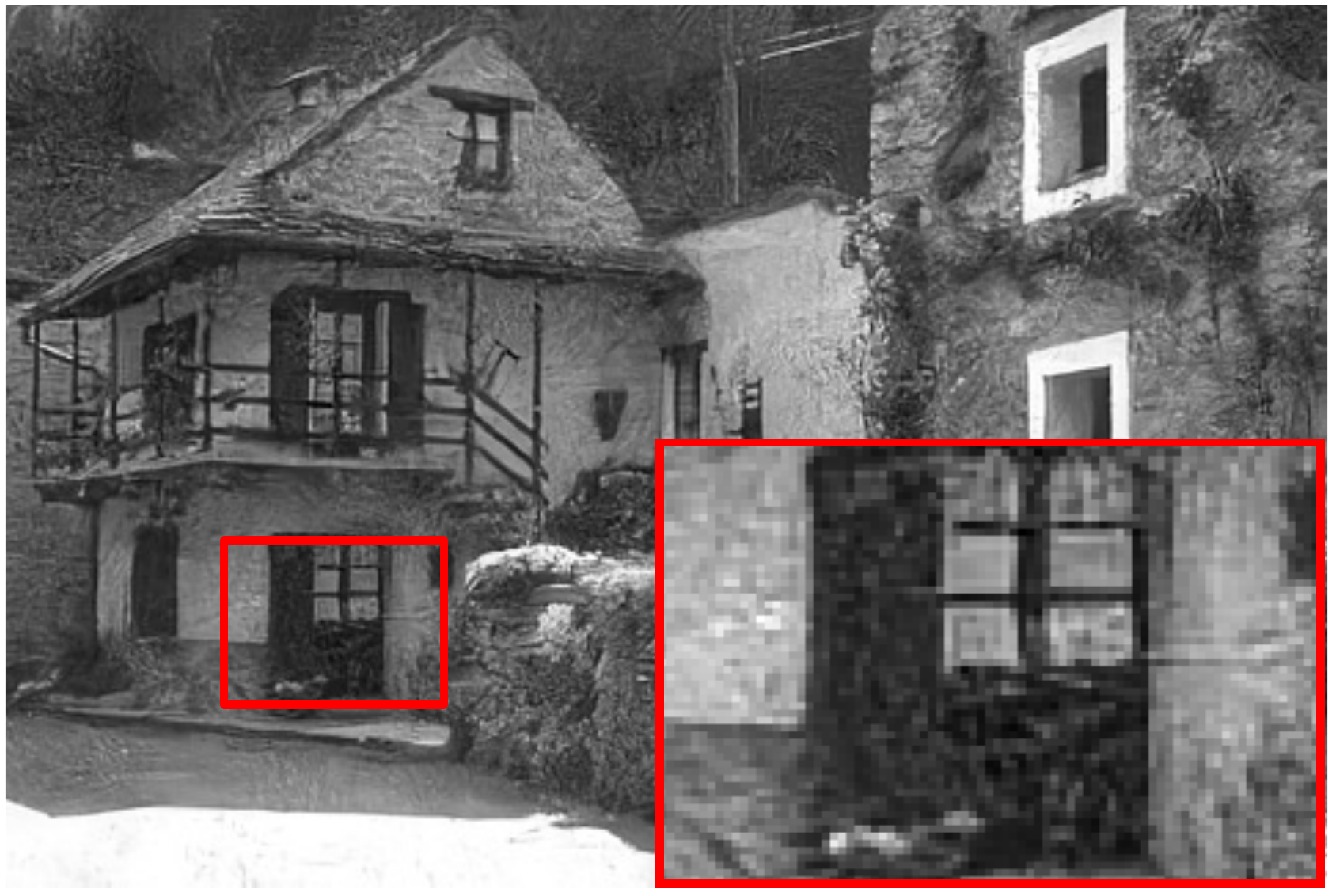}
		\caption{$\mbox{DnCNN}_{30}$, 24.49dB, 0.670}
	\end{subfigure}
	\begin{subfigure}[b]{0.19\textwidth}
		\centering
		\includegraphics[width=\textwidth]{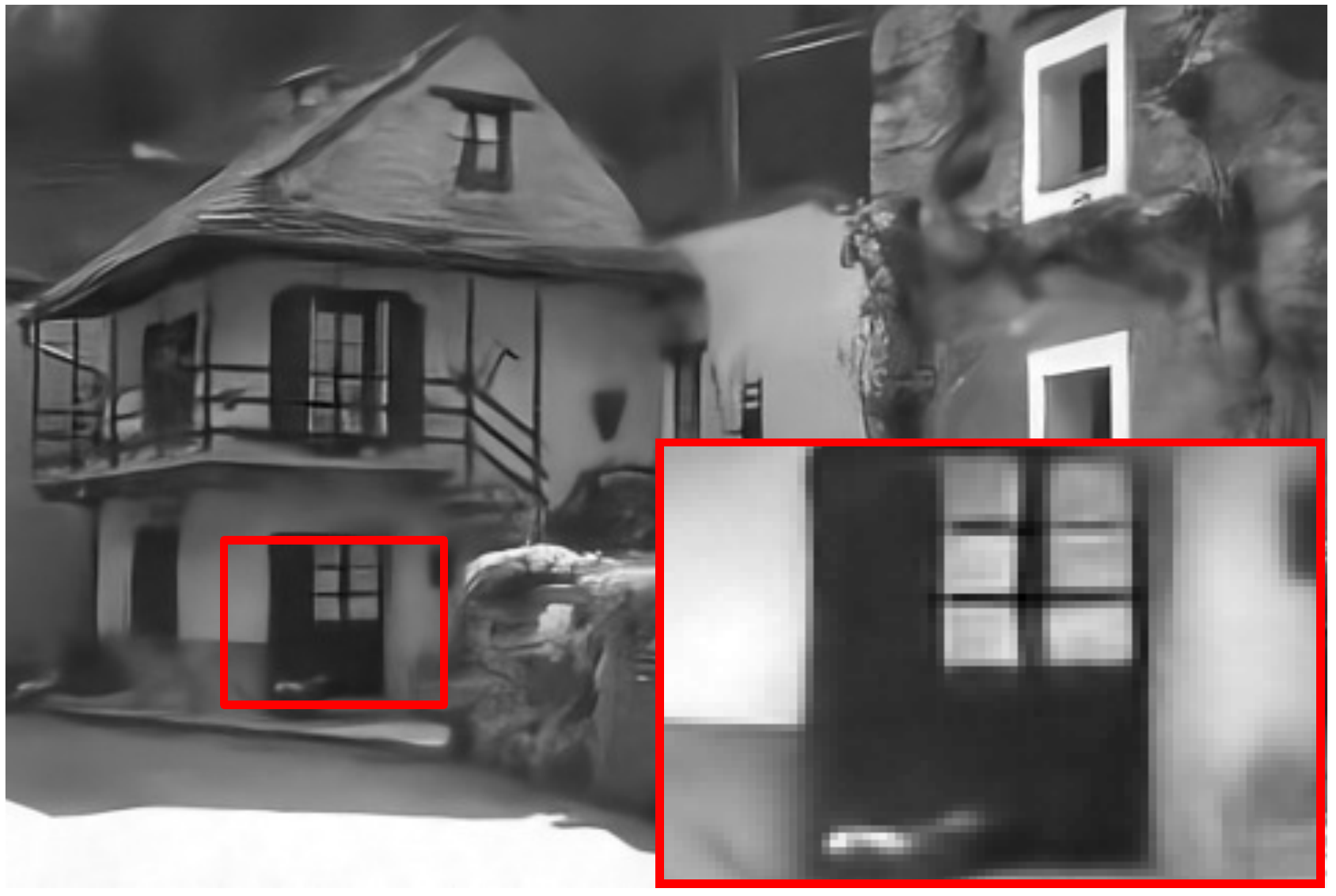}
		\caption{$\mbox{DnCNN}_{40}$, 24.88dB, 0.654}
	\end{subfigure}
	\begin{subfigure}[b]{0.19\textwidth}
		\centering
		\includegraphics[width=\textwidth]{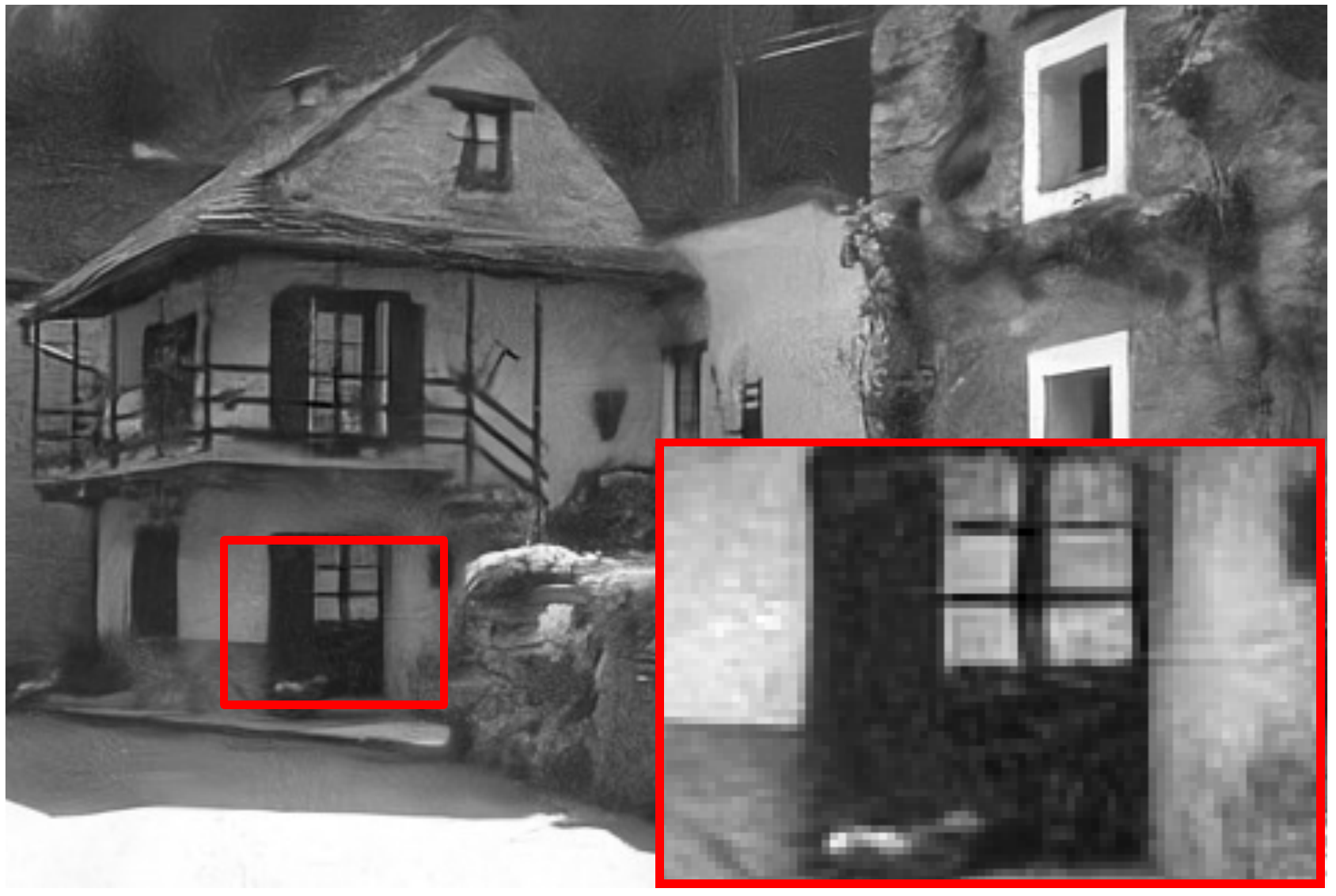}
		\caption{Before, 25.38dB, 0.703}
	\end{subfigure}
	\begin{subfigure}[b]{0.19\textwidth}
		\centering
		\includegraphics[width=\textwidth]{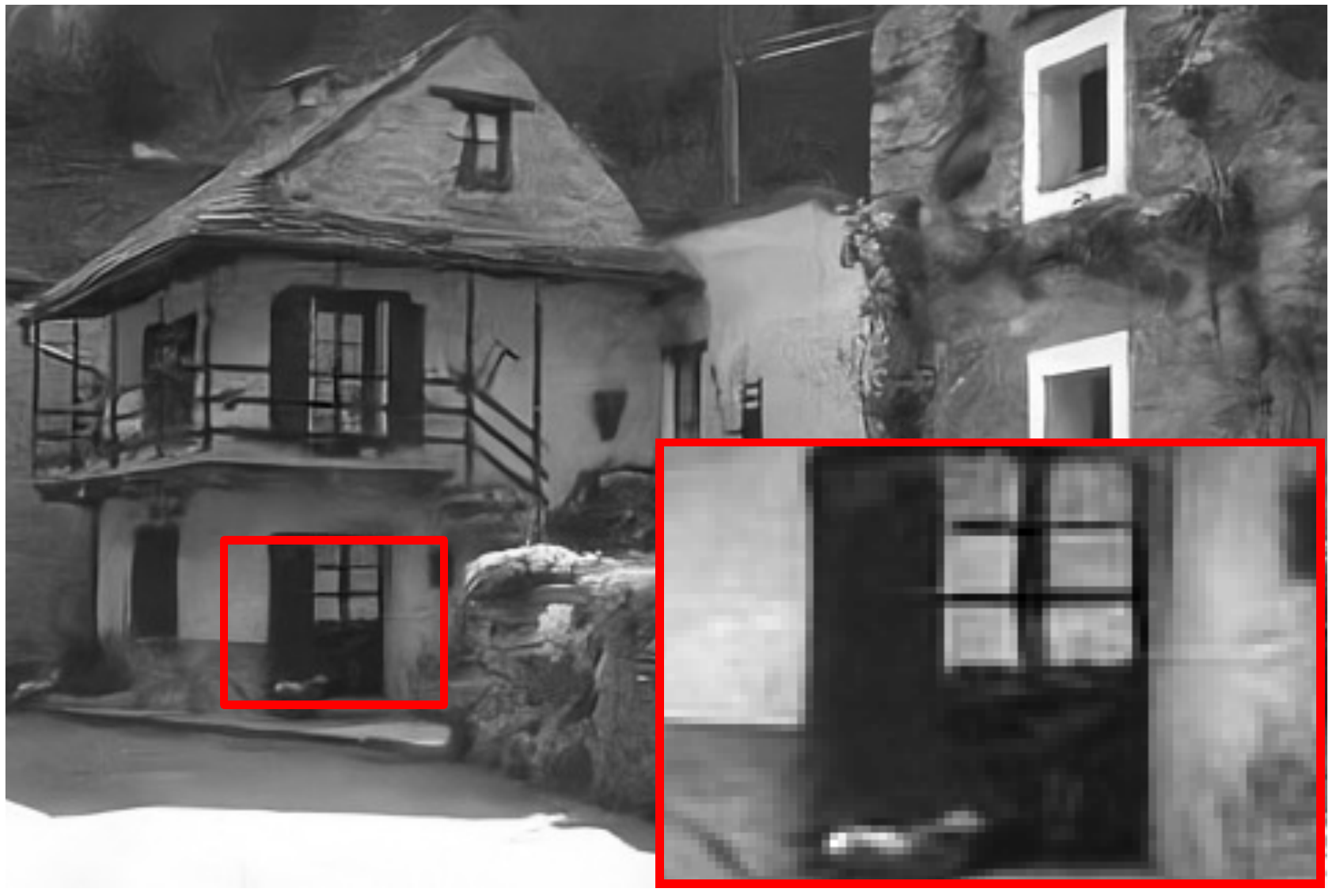}
		\caption{After, \textbf{25.44dB}, \textbf{0.710}}
	\end{subfigure}
	
	\hfill
	\begin{subfigure}[b]{0.19\textwidth}
		\centering
		\includegraphics[width=\textwidth]{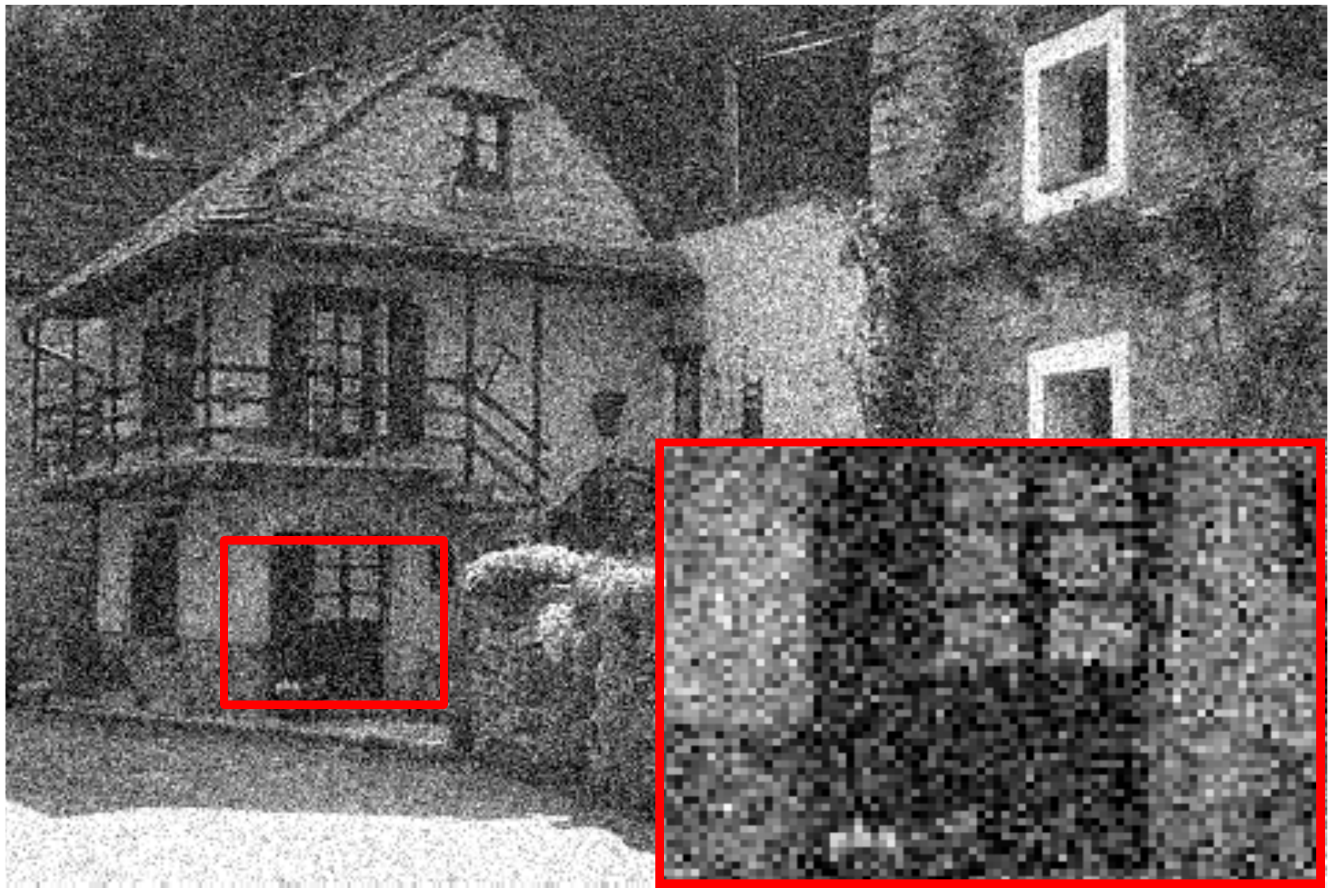}
		\caption{Input, 35, 17.53dB, 0.360}
	\end{subfigure}
	\begin{subfigure}[b]{0.19\textwidth}
		\centering
		\includegraphics[width=\textwidth]{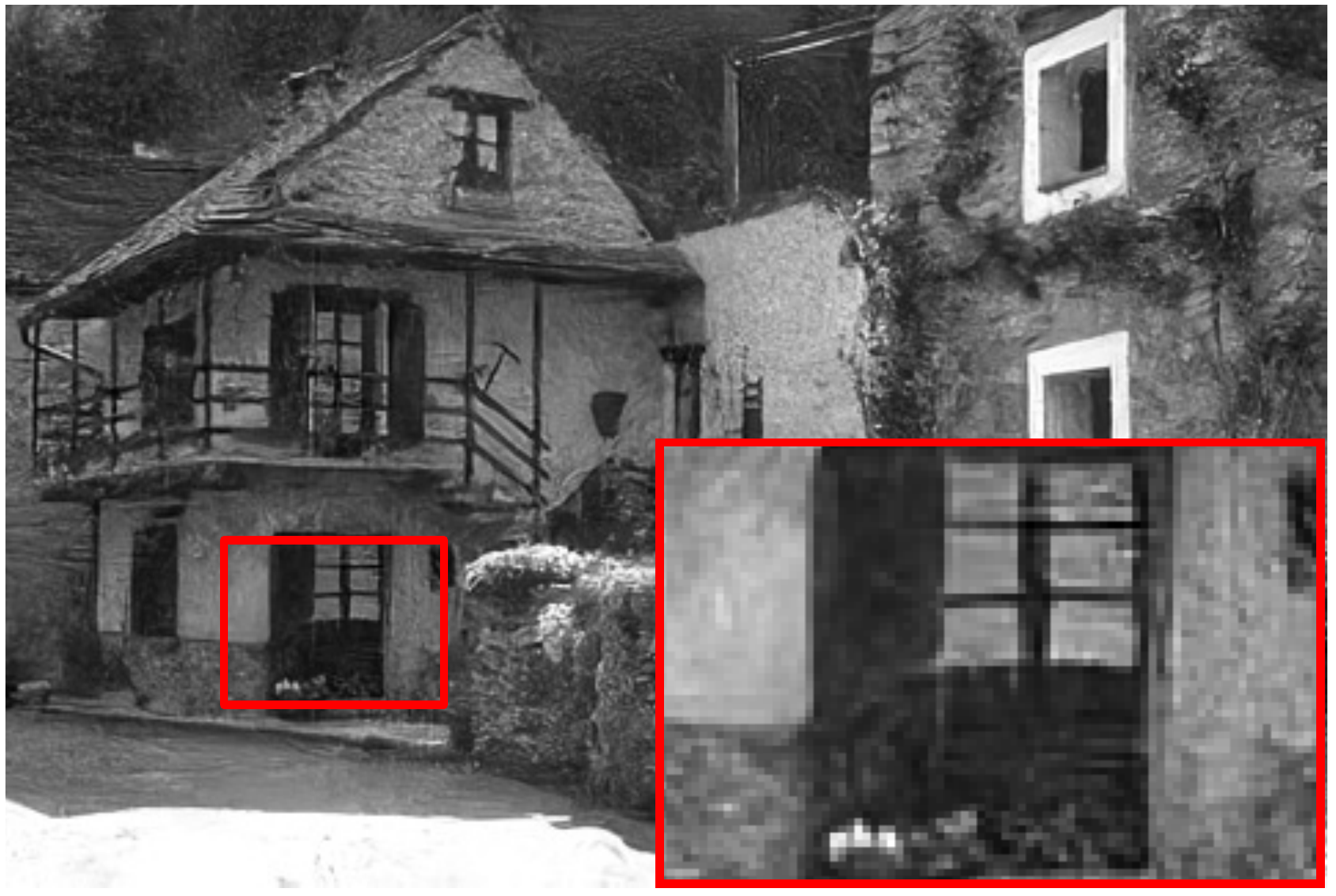}
		\caption{$\mbox{RED}_{30}$, 24.77dB, 0.681}
	\end{subfigure}
	\begin{subfigure}[b]{0.19\textwidth}
		\centering
		\includegraphics[width=\textwidth]{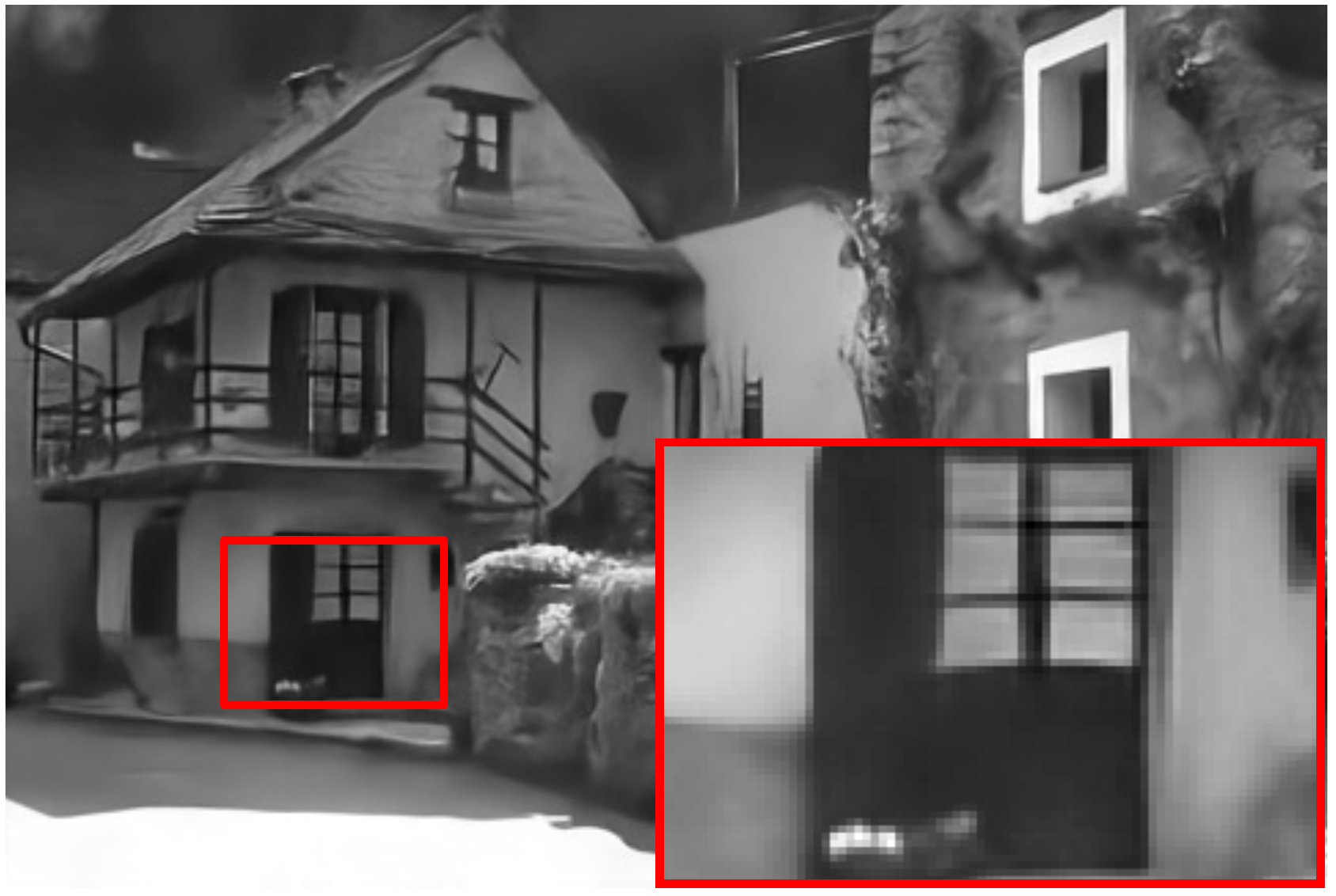}
		\caption{$\mbox{RED}_{40}$, 24.98dB, 0.663}
	\end{subfigure}
	\begin{subfigure}[b]{0.19\textwidth}
		\centering
		\includegraphics[width=\textwidth]{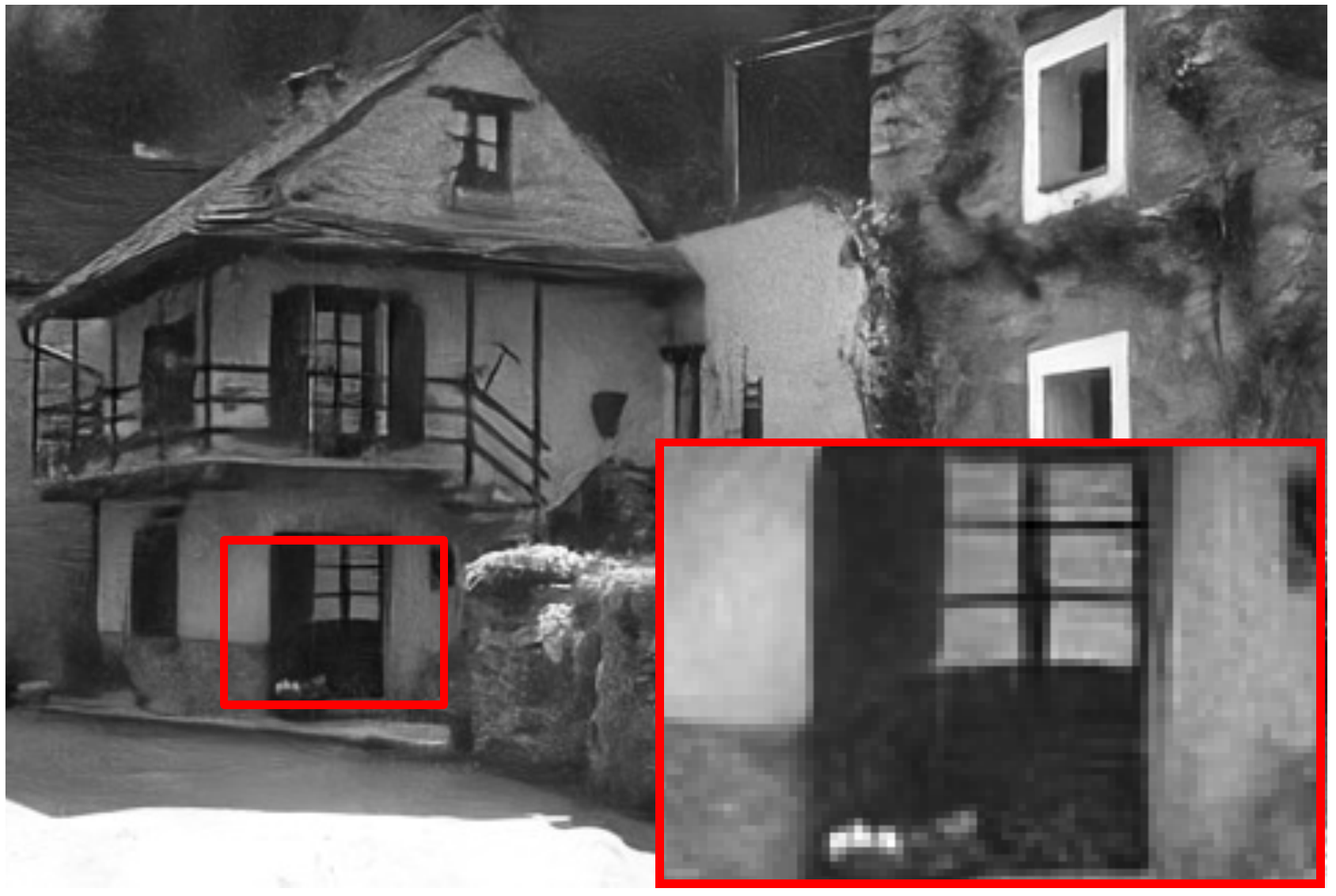}
		\caption{Before, 25.37dB, 0.703}
	\end{subfigure}
	\begin{subfigure}[b]{0.19\textwidth}
		\centering
		\includegraphics[width=\textwidth]{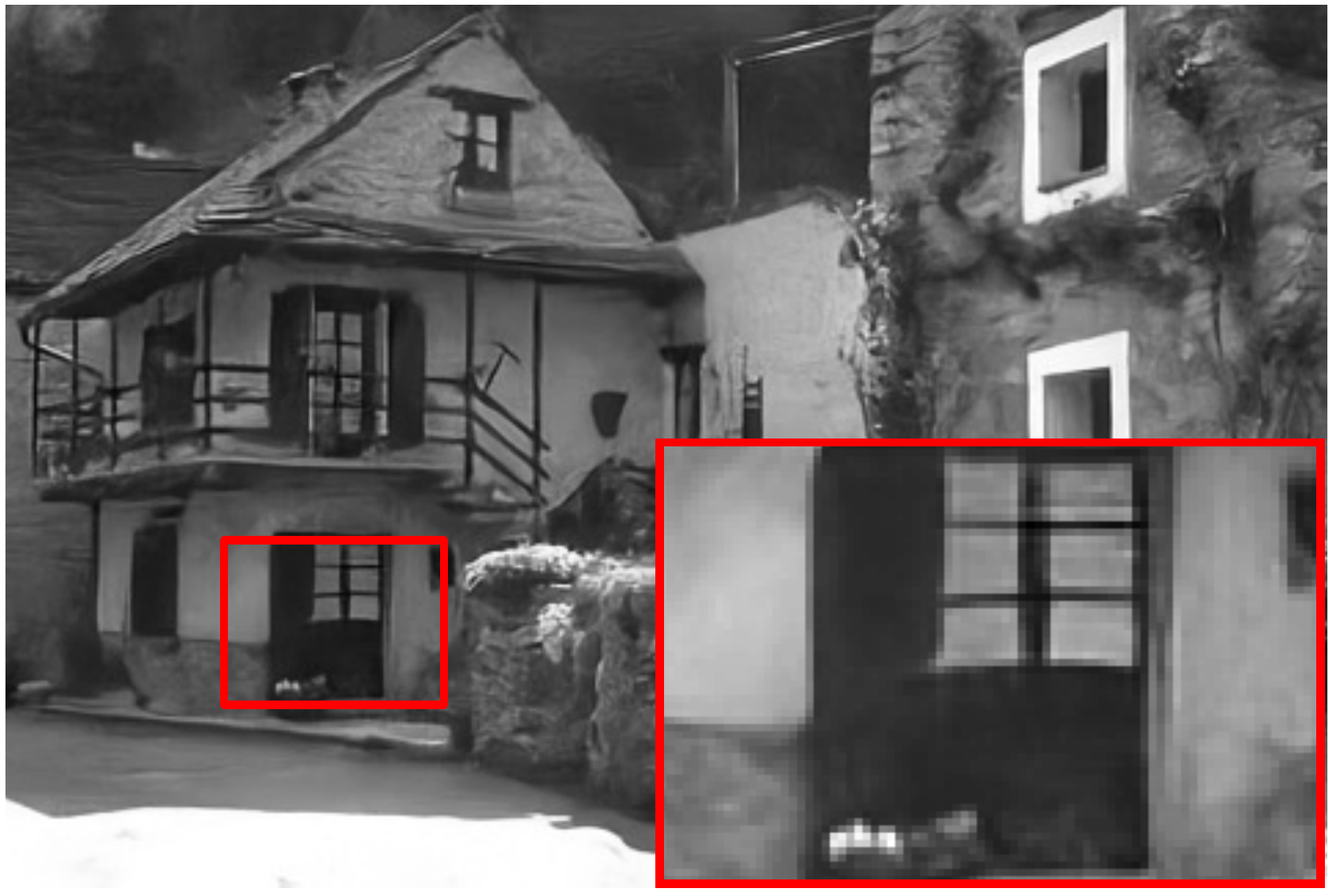}
		\caption{After, \textbf{25.42dB}, \textbf{0.709}}
	\end{subfigure}
	\caption{Experiment 1: Noise-level mismatch for image \texttt{House} (size 321$\times$481) from BSD500. The actual noise level is $\sigma=35$. Top row: use DnCNN as initial denoisers; Bottow row: use REDNet as initial denoisers. Reported are the PSNR and SSIM values. In this figure, ``before'' and ``after'' refer to the result before and after applying the booster.}
	\label{fig:sigma35}
\end{figure*}

\begin{table*}[!ht]
	\centering
	\begin{tabular}{c|ccccc|cc|cc}
		\hline
		& & & & & &  Before  & After	& Before & After 	\\
		& REDNet & REDNet & REDNet & REDNet & REDNet & Booster & Booster & Booster & Booster 	\\
		& ($\sigmahat=10$) & ($\sigmahat=20$) & ($\sigmahat=30$) & ($\sigmahat=40$) & ($\sigmahat=50$) & (est) & (est)  & (oracle) & (oracle) \\
		\hline
		$\sigma=10$ & \textbf{34.1705}	& 30.7509 &	28.2515 & 27.0308 & 25.9679 & \textbf{34.1438} & 33.9859 & \textbf{34.1747} &	33.9913 \\
		$\sigma=15$ & 28.2492 & \textbf{30.8902} & 28.3384 & 27.0760 & 25.9920 & 31.4585 & \textbf{31.7896} & 31.4729 & \textbf{31.7905}  \\
		$\sigma=20$ & 24.1948 & \textbf{30.4820} & 28.4766 & 27.1502 & 26.0329 & 30.4768 & \textbf{30.4805} & 30.4931 & \textbf{30.4888}  \\
		$\sigma=25$ & 21.6813 & 26.6475 & \textbf{28.6138} & 27.2381 & 26.0826 & 29.0650 & \textbf{29.2997} & 29.0723 & \textbf{29.3038}  \\
		$\sigma=30$ & 19.8598 & 22.9125 & \textbf{28.5231} & 27.3544 & 26.1494 & 28.5199 & \textbf{28.5494} & 28.5323 & \textbf{28.5571}  \\
		$\sigma=35$ & 18.4271 & 20.5155 & 26.5631 & \textbf{27.4453} & 26.2322 & 27.7247 & \textbf{27.8402} & 27.7352 & \textbf{27.8460}  \\
		$\sigma=40$ & 17.2471 & 18.8398 & 23.2288 & \textbf{27.2409} & 26.3338 & 27.2387 & \textbf{27.2781} & 27.2542 & \textbf{27.2887}  \\
		$\sigma=45$ & 16.2479 & 17.5394 & 20.7749 & 25.3760 & \textbf{26.4112} & 26.6592 & \textbf{26.7435} & 26.6722 & \textbf{26.7609}  \\
		$\sigma=50$ & 15.3815 & 16.4471 & 18.9488 & 22.5099 & \textbf{26.3197} & \textbf{26.3191} & 26.3145 & \textbf{26.3250} & 26.3227 \\
		\hline
	\end{tabular}
	\caption{Experiment 1A: Noise-level mismatch for \textbf{Unclipped Noise}, where noise is i.i.d. Gaussian \emph{without} clipping the signal to $[0,1]$. The average PSNRs of REDNet ($\sigmahat = 10,20,30,40,50$), Blind REDNet with 50 layers and CsNet on 200 test images from BSD500.In this figure, ``est'' and ``oracle'' refer to estimated MSE and the oracle MSE, respectively.}
	\label{tab:1}
\end{table*}

\subsection{Performance of Booster}
The effectiveness of the booster can be seen in Figure~\ref{fig:booster}, where we show a few examples taken from the BSD500 dataset. In this example, we consider a neural network denoiser trained at five different noise levels (See Section~\ref{sec:experiment 3} for experiment details).

As we see in Figure~\ref{fig:booster}, the booster is doing particularly well for two types of improvements. The first type of improvement is the recovery of the fine details. For example, in the Swam image we can recover the lines on the feather; in the House image we can recover branches of the tree. These are also reflected in the PSNR. The second type of improvement is the contrast enhancement. For example, before boosting the House image we see that the background sky has a gray-ish intensity. However, after boosting the background sky has a brighter background.

\section{Experiments}
We build our neural networks using Tensorflow and run on Intel(R) Core(TM) i5-4690K CPU 3.50GHz with an Nvidia Titan-X GPU, except DnCNN which is downloaded from the author's website \footnote{Note that the original REDNet in \cite{Mao2016} was implemented in Caffe, and the network was trained using patches of $50 \times 50$. We implemented REDNet on Tensorflow with patch size $64 \times 64$. On BSD200 dataset, our implementation shows better PSNR than the original REDNet.}.

\subsection{Experiment 1: Noise-Level Mismatch}
\label{sec:experiment 1}
Our first experiment is to evaluate CsNet for the case of noise-level mismatch. We consider two types of initial denoisers: DnCNN \cite{Zhang2017_tip} and REDNet \cite{Mao2016}. For each denoiser type, we use 300 training and validation images in BSD500 to train five initial denoisers $\calD_1,\ldots,\calD_5$. The denoising strength is set as one of the values $\sigmahat = 10,20,30,40, \textrm{and }50$. When testing, we use a noise level of $\sigma \in [10,50]$. In this experiment, the noise is unclipped i.i.d. Gaussian.

\begin{table*}[t]
	\centering
	\begin{tabular}{c|ccccc|cc|cc}
		\hline
		& & & & & &  Before  & After	& Before & After 	\\
		& REDNet & REDNet & REDNet & REDNet & REDNet & Booster & Booster & Booster & Booster 	\\
		& ($\sigmahat=10$) & ($\sigmahat=20$) & ($\sigmahat=30$) & ($\sigmahat=40$) & ($\sigmahat=50$) & (est) & (est)  & (oracle) & (oracle)  \\
		\hline
		$\sigma=10$ & \textbf{34.1428}	& 30.6934 &	28.2434 & 26.8287 & 25.8601 & \textbf{34.0756} & 33.9220 & \textbf{34.1434} & 33.9061 \\
		$\sigma=15$ & 28.4337 & \textbf{30.7544} & 28.2961 & 26.8381 & 25.8532 & 31.3295 & \textbf{31.7896} &  31.3878 & \textbf{31.8022} \\
		$\sigma=20$ & 24.4306 & \textbf{30.3462} & 28.3595 & 26.8382 & 25.8341 & 30.3121 & \textbf{30.4621} & 30.3516 & \textbf{30.4763} \\
		$\sigma=25$ & 21.8383 & 26.9932 & \textbf{28.4116} & 26.8396 & 25.8065 & 28.8881 & \textbf{29.3027} & 28.9210 & \textbf{29.3030} \\
		$\sigma=30$ & 19.9669 & 23.4285 & \textbf{28.2041} & 26.8316 & 25.7651 & 28.1983 & \textbf{28.5163} & 28.2213 & \textbf{28.5225} \\
		$\sigma=35$ & 18.4955 & 21.0504 & 26.2027 & \textbf{26.7998} & 25.7074 & 27.2566 & \textbf{27.7785} & 27.2774 & \textbf{27.7848} \\
		$\sigma=40$ & 17.2907 & 19.3423 & 23.2651 & \textbf{26.6291} & 25.6314 & 26.6547 & \textbf{27.2147} & 26.6777 & \textbf{27.2208} \\
		$\sigma=45$ & 16.2759 & 18.0084 & 20.9738 & \textbf{25.5692} & 25.5244 & 25.9516 & \textbf{26.6856} & 25.9750 & \textbf{26.6975} \\
		$\sigma=50$ & 15.3992 & 16.9077 & 19.2047 & 23.3792 & \textbf{25.3426} & 25.3940 & \textbf{26.2533} & 25.4284 & \textbf{26.2612} \\
		\hline
	\end{tabular}
	\caption{Experiment 1B: Noise-level mismatch for the \textbf{Clipped Noise}, where the i.i.d. Gaussian is clipped to ensure that the signal lies in $[0,1]$. The average PSNRs of REDNet ($\sigmahat = 10,20,30,40,50$), Blind REDNet with 50 layers and CsNet on 200 test images from BSD500. In this figure, ``est'' and ``oracle'' refer to estimated MSE and the oracle MSE, respectively.}
	\label{tab:clipped noise}
\end{table*}

The result of this experiment is shown in Table~\ref{tab:1} and Figure~\ref{fig:sigma35}. Table~\ref{tab:1} shows the comparison with REDNet as initial denoisers, whereas Figure~\ref{fig:sigma35} shows a visual comparison of an image in the BSD500 dataset. We can make a few observations here:

\begin{itemize}
\item \textbf{General Performance}. For each $\sigma$, the best performing REDNet is the one with $\sigmahat$ right above $\sigma$. This result is consistent with the suggestion made by Zhang et al \cite{Zhang2017_tip}. However, the combination (before boosting) is able to improve the performance by an average of 0.3dB for noise levels that are originally not trained for, i.e., $\sigma = 15,25,35,45$. For noise levels that are originally in the training set, i.e., $\sigma = 10,20,30,40,50$, the improvement is marginal.
    \vspace{1ex}
\item \textbf{Effect of Boosting}. If the actual noise level is unseen by the denoiser, e.g., $\sigma = 15$, the PSNR gain due to the booster is significant. For noise levels that have been observed, e.g., $\sigma = 20$, the gain is marginal. The reason is that the booster has less room to improve when the denoised image is already good. This is consistent to the results reported in the boosting literature \cite{Romano2015}. We also observe that for noise levels $\sigma = 10$ and $\sigma = 50$ there is a minor drop in the booster. This is because the booster is itself an estimator. When handling a wide range of noise levels, the network is only able to maximize the performance on the average case. For the extreme cases, there is a fundamental limitation which prevents the booster from being able to produce consistently good results. The same finding holds for other blind deep neural network denoisers, e.g., \cite{Zhang2017_tip}, which has worse performance for extreme low-noise and high-noise cases.
    \vspace{1ex}
\item \textbf{Oracle VS Estimate}. The difference between the oracle MSE and the estimated MSE is very small. Here, by oracle MSE we meant that the MSE is calculated from the ground truth. This will give us the best possible $\mSigma$ when solving the convex optimization, and the PSNR can be regarded as the upper bound of any estimation method. As shown in the table, the performance of the MSE estimator is very similar to the oracle. This suggests that our neural network MSE estimator can reliably predict the MSE and hence facilitates the combination scheme.
\end{itemize}

\subsection{Deeper Vanilla Network?}
A natural question we can ask is that since we have five initial deep neural networks, is the performance gain due to the increased model capacity of the overall denoiser? To answer this question, we consider a blind denoiser of the same model capacity as the overall CsNet before boosting. Specifically, since we are using five REDNet-30 in the previous experiment, here we train a blind REDNet with 150 layers by repeating the structure of REDNet-30 five times. We call this the deep vanilla network.

\begin{table}[h]
	\centering
	\begin{tabular}{c|cc|c@{ }c@{}}
		\hline
		& Before  & After  & REDNet & REDNet \\
		& Booster & Booster & Blind 150 & Blind 150 	\\
		& (est) & (est)  & ($\sigmahat$=1,2,...,70) & ($\sigmahat$=10,20,...,50) \\
		\hline
		$\sigma$=10 &  \textbf{34.1438} & 33.9859 & 33.8295 & \textbf{33.9487} \\
		$\sigma$=15 &  31.4585 & \textbf{31.7896} & \textbf{31.7352}  & 30.2000 \\
		$\sigma$=20 &  30.4768 & \textbf{30.4805} & 30.3304  & \textbf{30.3557} \\
		$\sigma$=25 &  29.0650 & \textbf{29.2997} & \textbf{29.2868}  & 28.1811 \\
		$\sigma$=30 &  28.5199 & \textbf{28.5494} & 28.4640  & \textbf{28.4782} \\
		$\sigma$=35 &  27.7247 & \textbf{27.8402} & \textbf{27.7810}  & 27.1076 \\
		$\sigma$=40 &  27.2387 & \textbf{27.2781} & \textbf{27.2084}  & 27.1945 \\
		$\sigma$=45 &  26.6592 & \textbf{26.7435} & \textbf{26.7229}  & 25.0532 \\
		$\sigma$=50 &  \textbf{26.3191} & 26.3145 & \textbf{26.3013}  & 26.2898 \\
		\hline
	\end{tabular}
	\caption{ConsensusNet vs. Deep Vanilla Network. For the Deep Vanilla Network, one REDNet is trained with $\sigmahat=1, 2, ..., 70$ and the other is with $\sigmahat=10, 20, ..., 50$ like the initial denoisers.}
	\label{tab: vanilla}
\end{table}

The result of this experiment is shown in Table~\ref{tab: vanilla}. The first two columns of this table show the unclipped noise performance using our proposed method. The third column is the vanilla 150-layer REDNet trained using noisy samples of noise level from 1 to 70. This is an advantageous setting, because the network is allowed to see samples of noise levels such as 15 or 35 which are not present in the five baseline REDNet-30's. The last column is another vanilla 150-layer REDNet, but trained using noise levels of $\{10,20,30,40,50\}$. This is more fair, as the network has the same training samples as the five baseline REDNet-30's. Both networks are trained with the same number of training examples.

As we observe from Table~\ref{tab: vanilla}, the proposed combination scheme actually works better than the 150-layer REDNet. If we compare ``before boosting'' and the last column (the REDNet trained with the same set of samples as ours), the combination scheme produces significantly better performance in all cases. This suggests that the improvement is not due to the increased model capacity but the intrinsic power of the combination. If we allow the 150-layer REDNet to see the unseen examples (i.e., the third column), then the performance is worse than our ``before boosting'' for noise levels $\sigma = 10,20,\ldots,50$. For noise levels such as $15,25,\ldots,45$, the 150-layer REDNet is better than ``before boosting''. However, this is an unfair comparison because this REDNet is allowed to see images of those noise levels.

We also observe in some cases the weaker REDNet-150 (last column) performs better than the more powerful REDNet-150 (third column). These happens when $\sigma = 10, 20, 30$. One reason is that for the same amount of training examples, the more powerful REDNet distributes the training examples to all noise levels from 1 to 70, whereas the weaker REDNet only focuses on $10,20,\ldots,50$. This puts advantageous on the weaker REDNet-150 when it goes to those noise levels. In fact, even for $\sigma = 40$ and $50$, the difference between the two REDNet's are marginal.

\subsection{Clipped and Unclipped Noise}
Since our proposed framework can be adapted to different types of noise (by training a different MSE estimator), here we demonstrate the performance of the proposed method on clipped and unclipped noise. To generate the clipped noisy image, we first add i.i.d. Gaussian noise to the image and clip the resulting image to the range $[0,1]$. We argue that this is a more natural configuration, because most physical sensor have limited dynamic range.

The result of this experiment is shown in Table~\ref{tab:clipped noise}. One thing to notice is that the REDNet's are still the same; They are re-trained using the clipped noise. As a result, their performance is worse than the unclipped version because of the training-testing mismatch. However, this deficiency of the initial denoiser brings out a useful feature of the proposed framework: Regardless of what the initial denoiser does, the proposed framework is able to pick the strongest denoiser and make improvements. If we look at Table~\ref{tab:clipped noise}, besides the case of $\sigma = 10$, the proposed method is always better than the initial denoiser, despite the fact that the noise is clipped.

\subsection{Experiment 2: Different Image Classes}
\label{sec:experiment 2}

\begin{figure*}[t]
	\centering
	\captionsetup[subfigure]{labelformat=empty, justification=centering}
	\begin{subfigure}[b]{0.13\textwidth}
		\centering
		\includegraphics[width=\textwidth]{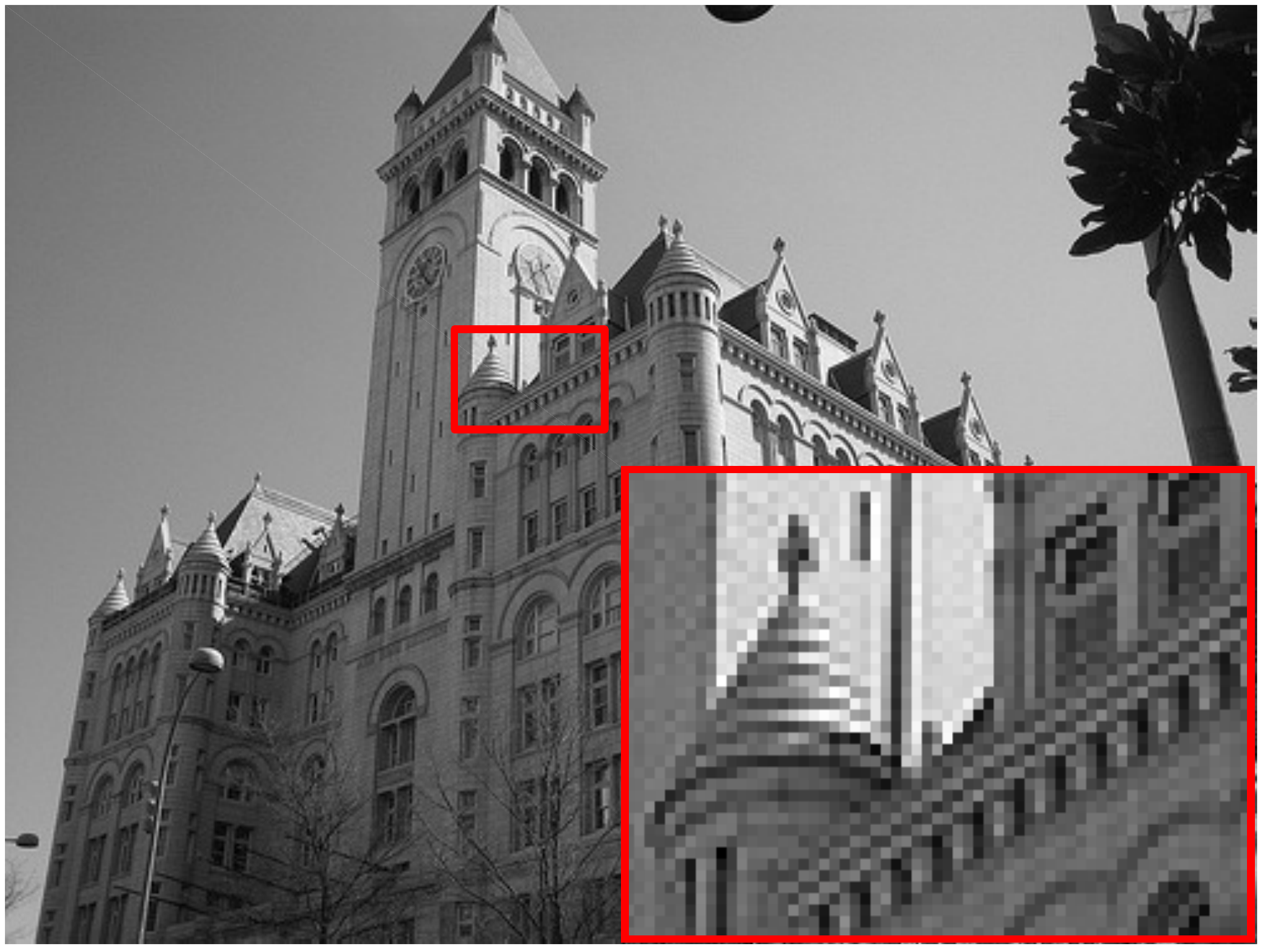}
		\caption{Groundtruth \\ Bldg}
	\end{subfigure}
	\begin{subfigure}[b]{0.13\textwidth}
		\centering
		\includegraphics[width=\textwidth]{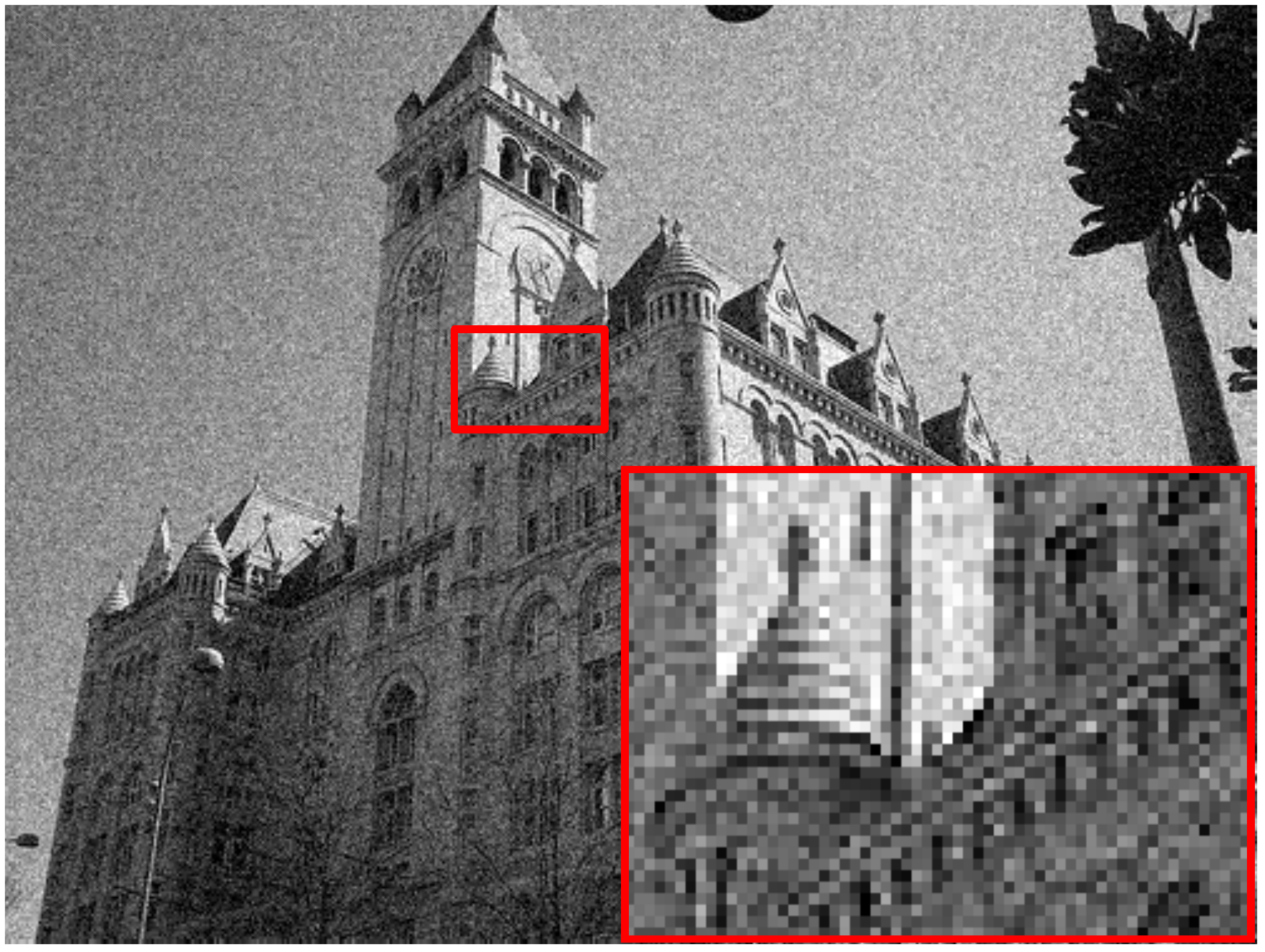}
		\caption{22.20dB, 0.4892 \\ Input}
	\end{subfigure}
	\begin{subfigure}[b]{0.13\textwidth}
		\centering
		\includegraphics[width=\textwidth]{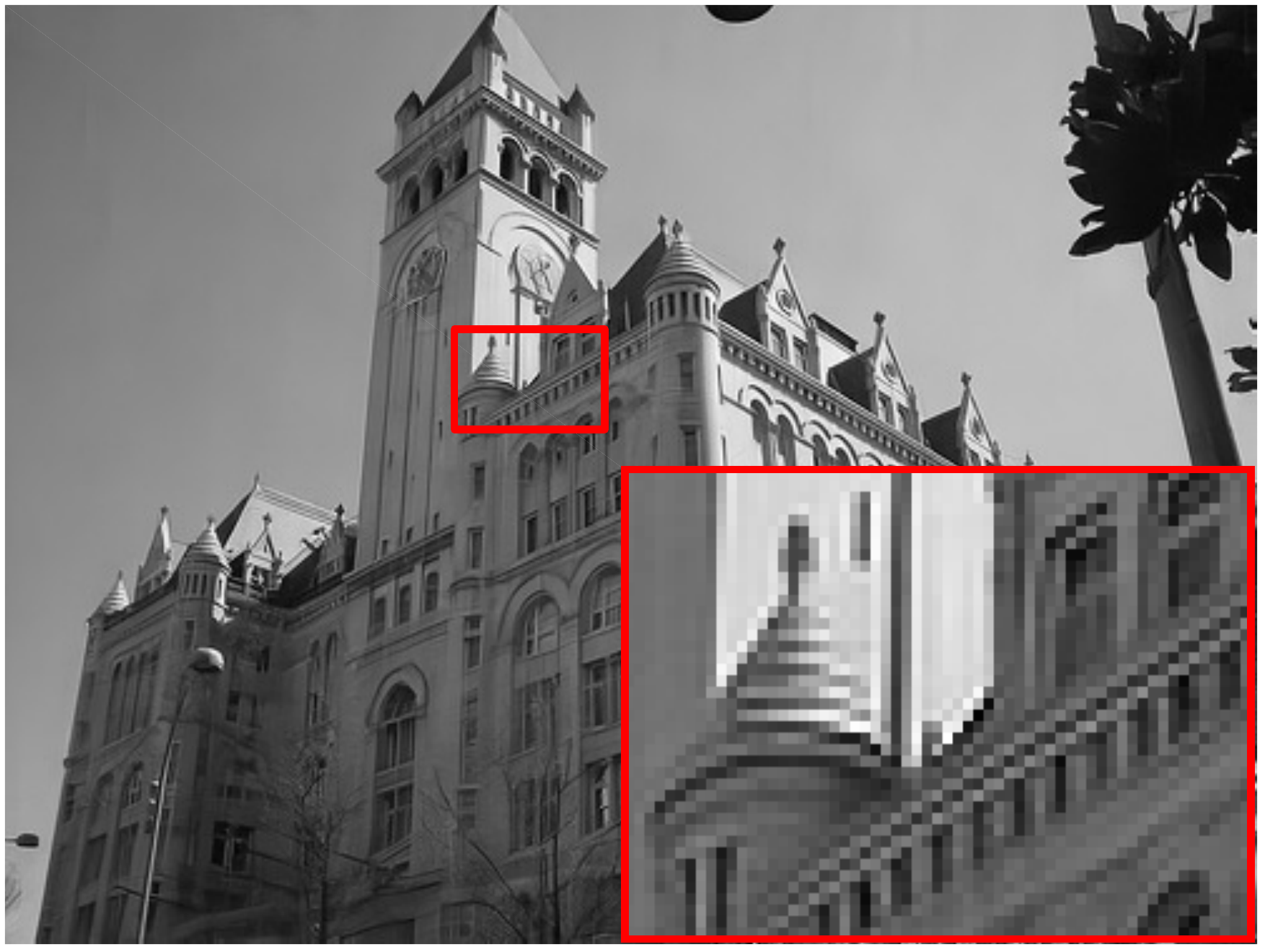}
		\caption{29.61dB, 0.8941 \\ RED-Bldg}
	\end{subfigure}
	\begin{subfigure}[b]{0.13\textwidth}
		\centering
		\includegraphics[width=\textwidth]{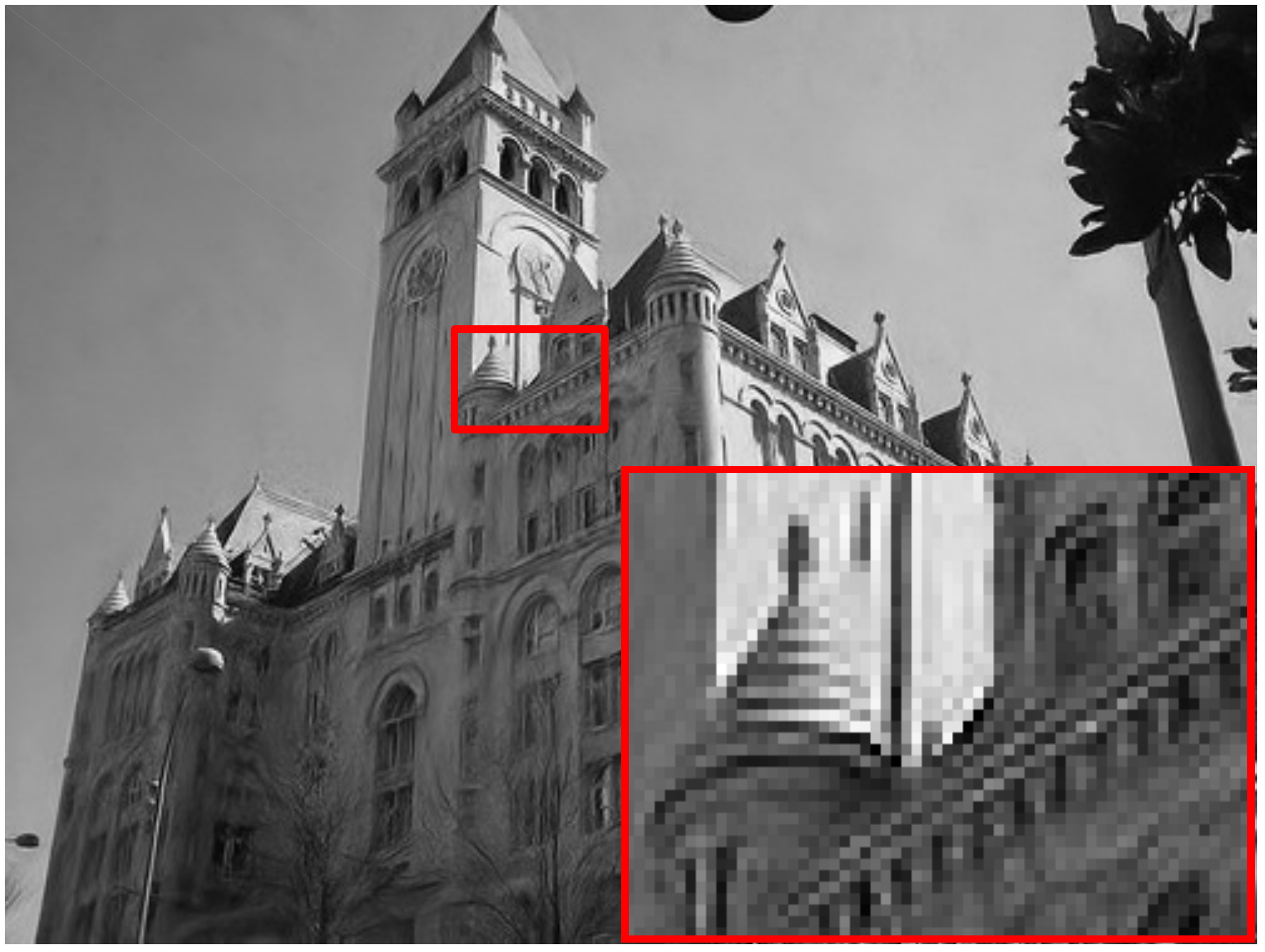}
		\caption{28.67dB, 0.8731 \\ RED-Face}
	\end{subfigure}
	\begin{subfigure}[b]{0.13\textwidth}
		\centering
		\includegraphics[width=\textwidth]{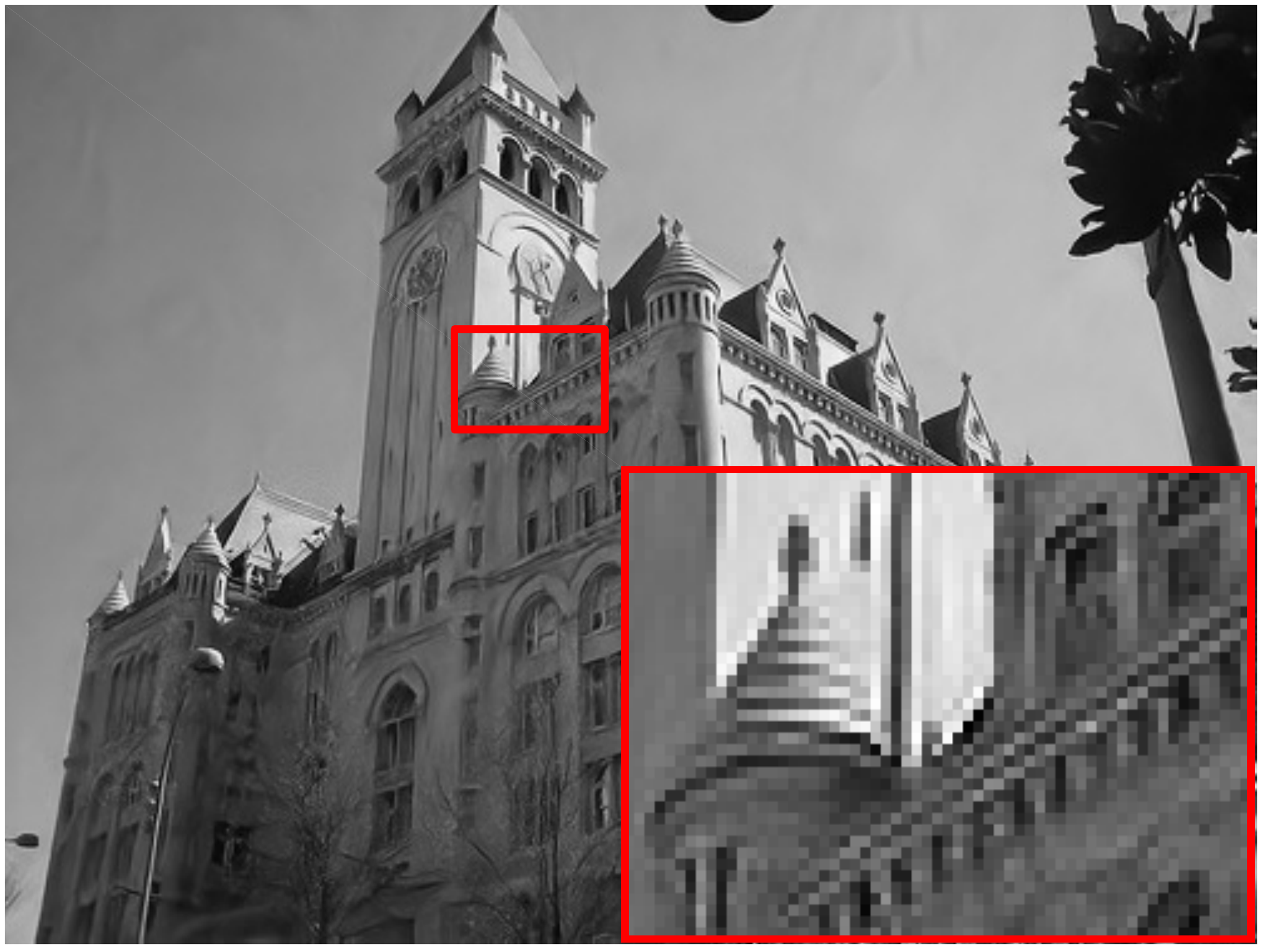}
		\caption{29.02dB, 0.8798 \\ RED-Flower}
	\end{subfigure}
	\begin{subfigure}[b]{0.13\textwidth}
		\centering
		\includegraphics[width=\textwidth]{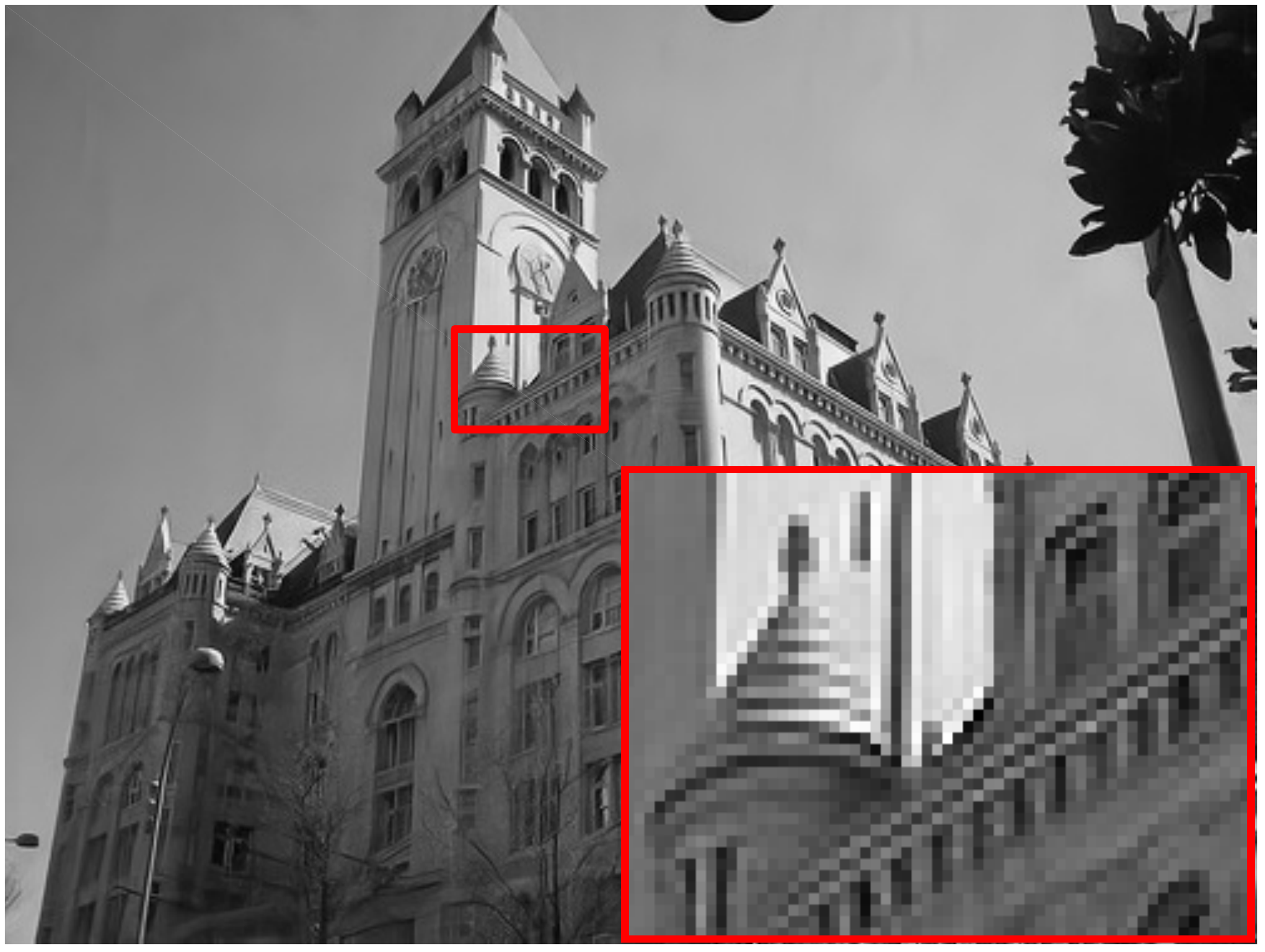}
		\caption{29.56dB, 0.8918 \\ Before}
	\end{subfigure}
	\begin{subfigure}[b]{0.13\textwidth}
		\centering
		\includegraphics[width=\textwidth]{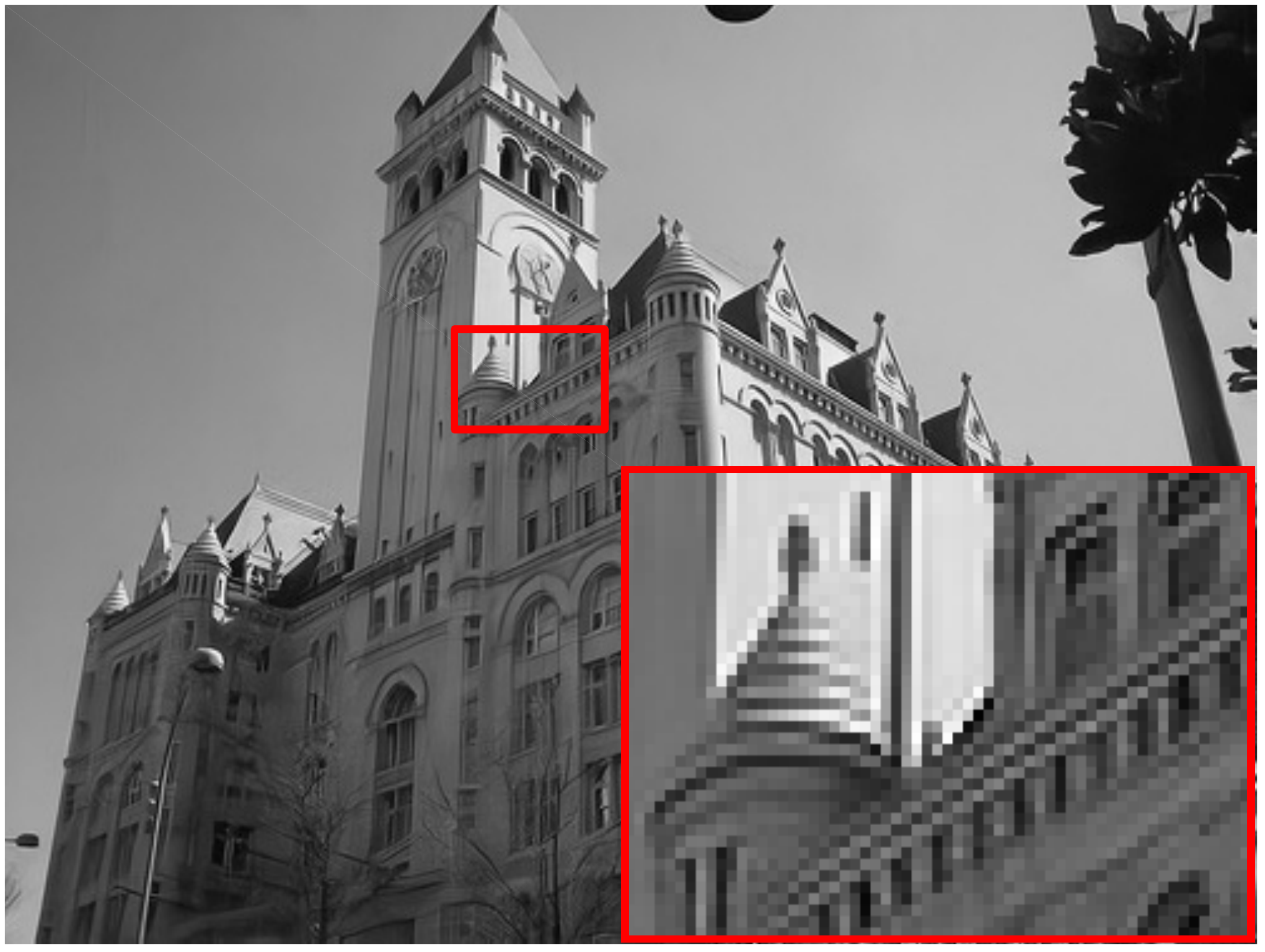}
		\caption{\textbf{29.73dB}, \textbf{0.8946} \\ After}
	\end{subfigure}

	\begin{subfigure}[b]{0.13\textwidth}
		\centering
		\includegraphics[width=\textwidth]{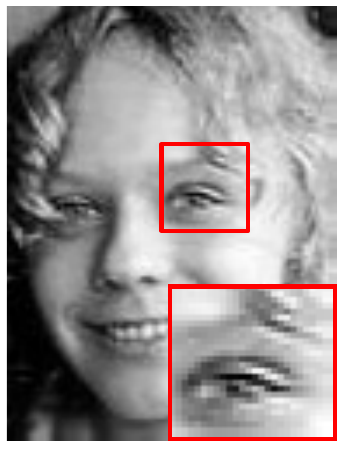}
		\caption{Ground Truth \\ Face}
	\end{subfigure}
	\begin{subfigure}[b]{0.13\textwidth}
		\centering
		\includegraphics[width=\textwidth]{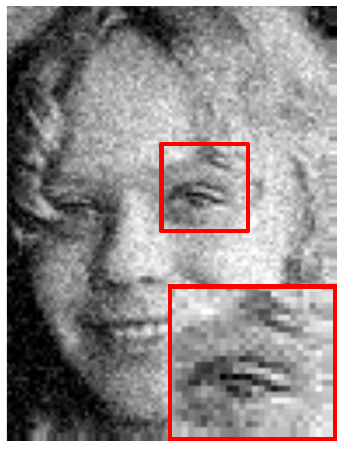}
		\caption{22.38dB , 0.5732\\ Input}
	\end{subfigure}
	\begin{subfigure}[b]{0.13\textwidth}
		\centering
		\includegraphics[width=\textwidth]{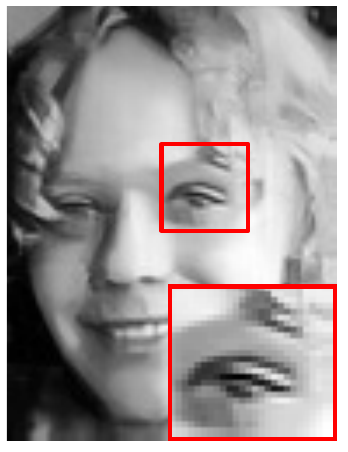}
		\caption{29.14dB, 0.8716 \\ RED-Bldg}
	\end{subfigure}
	\begin{subfigure}[b]{0.13\textwidth}
		\centering
		\includegraphics[width=\textwidth]{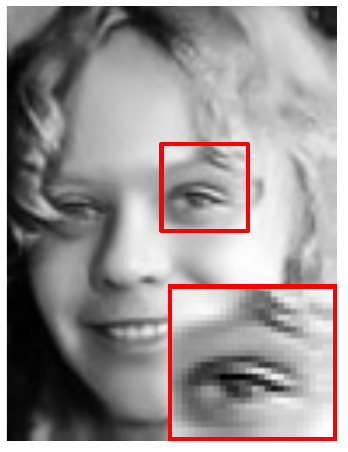}
		\caption{29.43dB, 0.8780 \\ RED-Face}
	\end{subfigure}
	\begin{subfigure}[b]{0.13\textwidth}
		\centering
		\includegraphics[width=\textwidth]{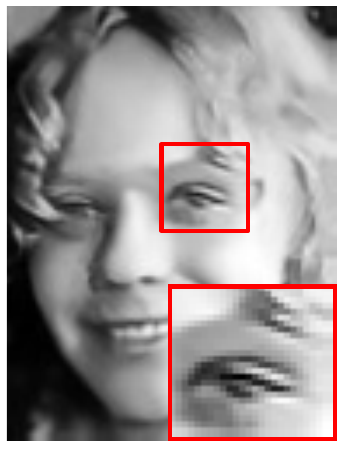}
		\caption{29.41dB, 0.8769 \\ RED-Flower}
	\end{subfigure}
	\begin{subfigure}[b]{0.13\textwidth}
		\centering
		\includegraphics[width=\textwidth]{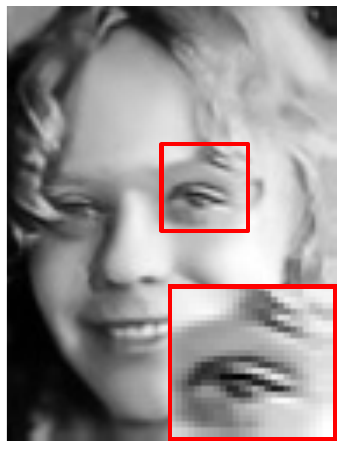}
		\caption{29.50dB, 0.8789 \\ Before}
	\end{subfigure}
	\begin{subfigure}[b]{0.13\textwidth}
		\centering
		\includegraphics[width=\textwidth]{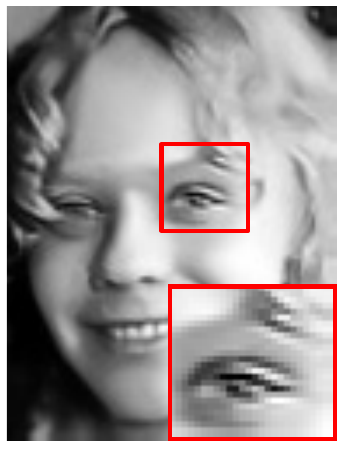}
		\caption{\textbf{29.92dB}, \textbf{0.8847} \\ After}
	\end{subfigure}

	\begin{subfigure}[b]{0.13\textwidth}
		\centering
		\includegraphics[width=\textwidth]{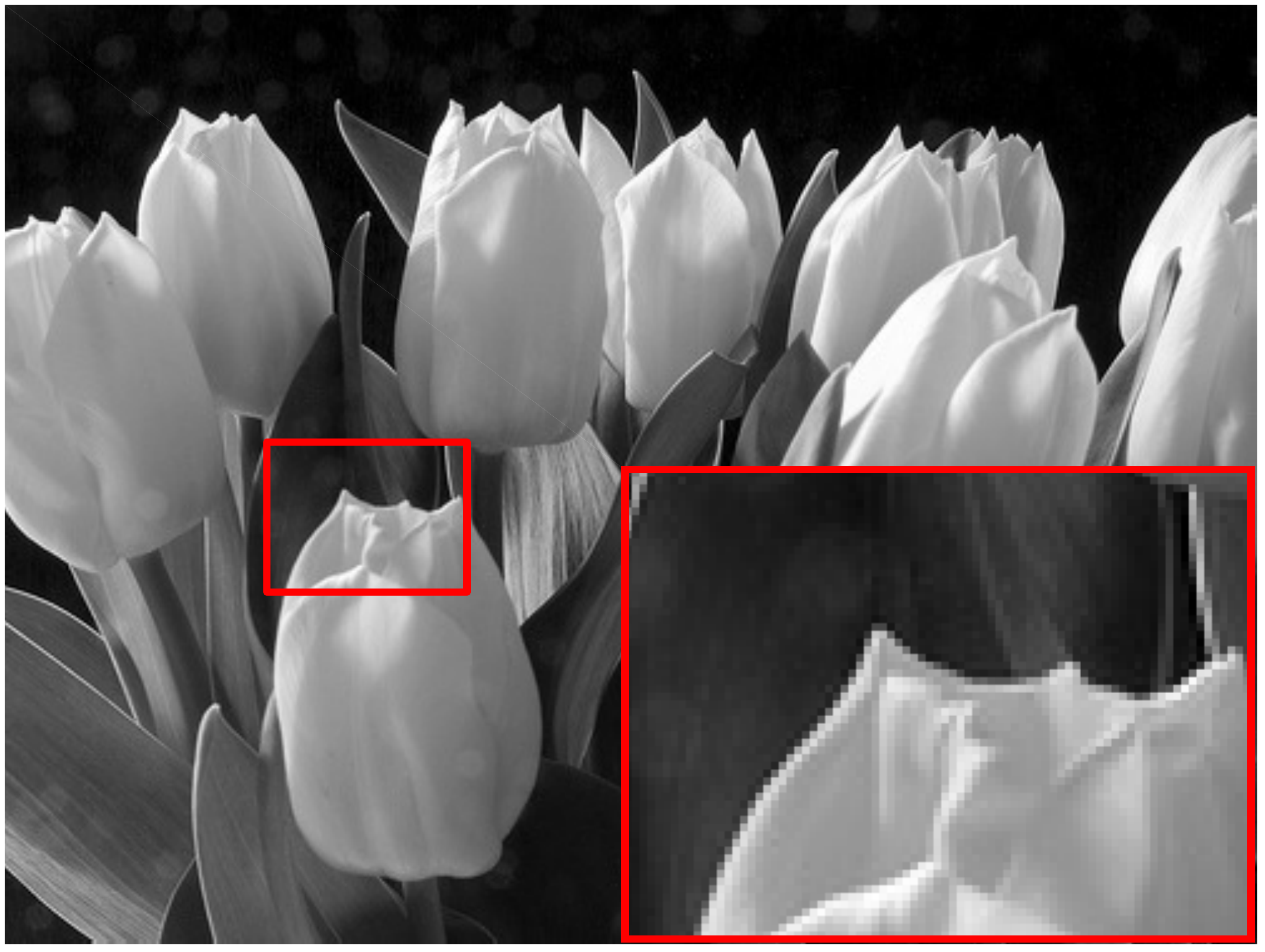}
		\caption{Groundtruth \\ Flower}
	\end{subfigure}
	\begin{subfigure}[b]{0.13\textwidth}
		\centering
		\includegraphics[width=\textwidth]{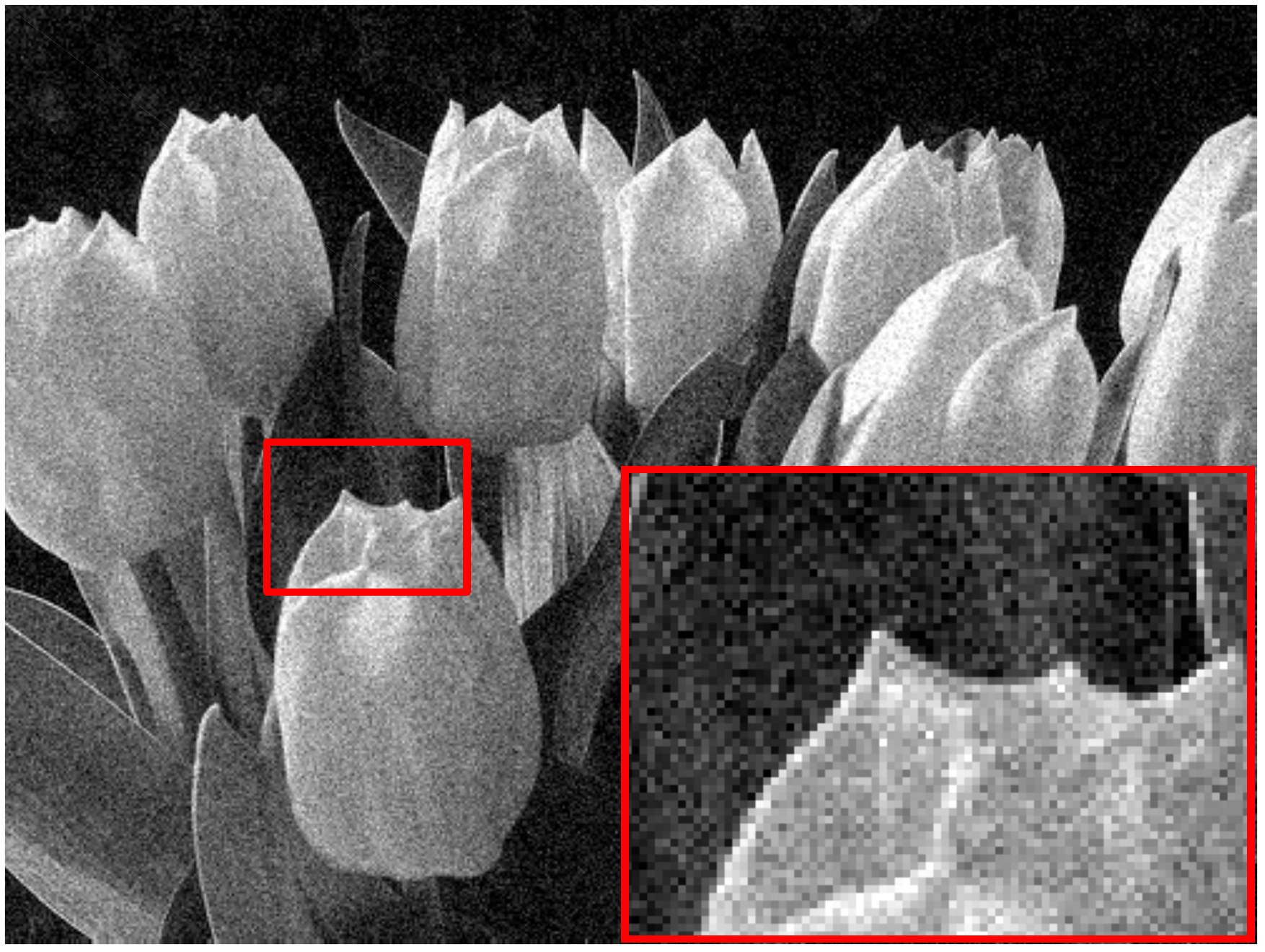}
		\caption{22.59dB, 0.3469 \\ Input}
	\end{subfigure}
	\begin{subfigure}[b]{0.13\textwidth}
		\includegraphics[width=\textwidth]{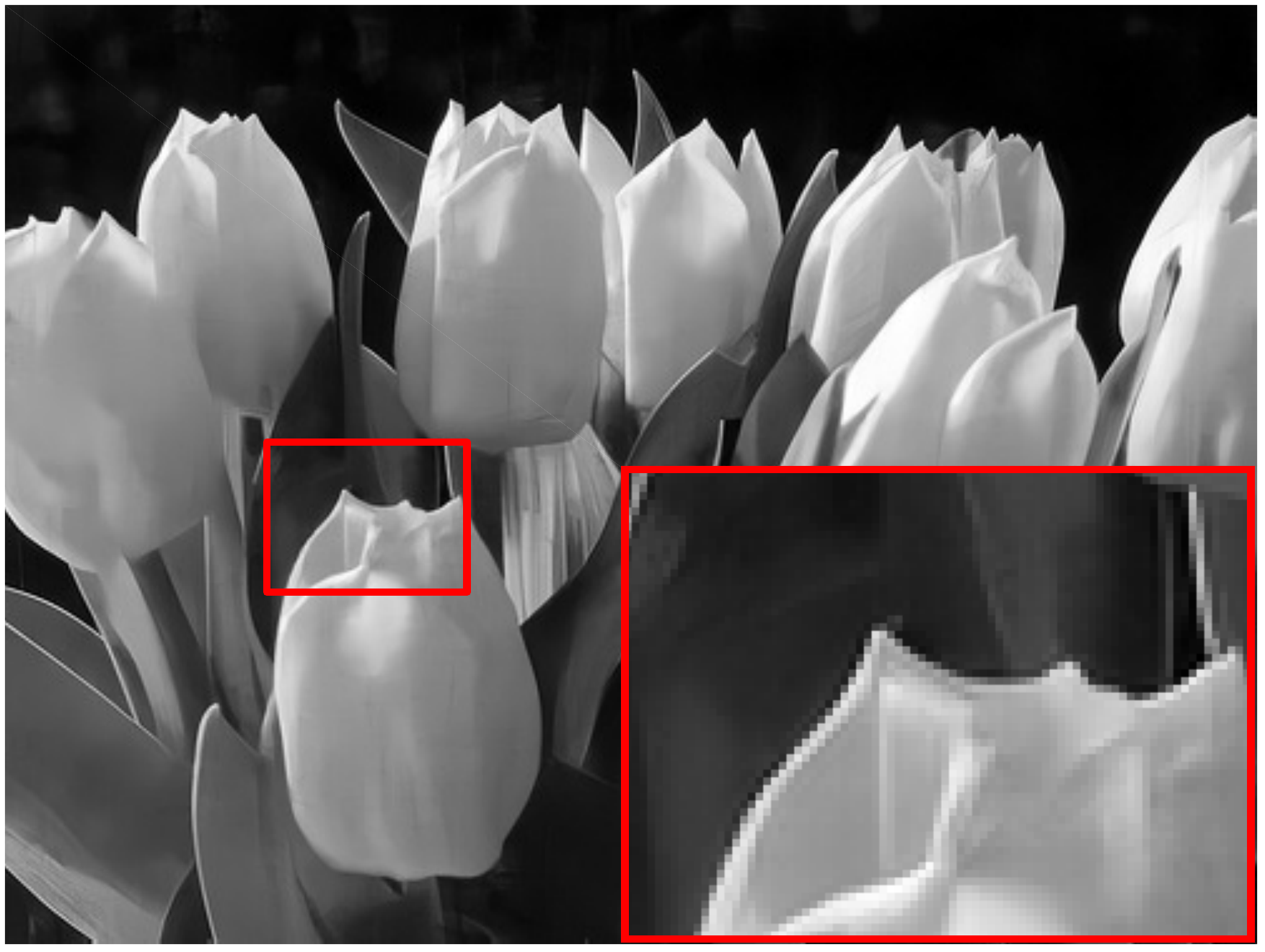}
		\caption{34.18dB, 0.9069 \\ RED-Bldg}
	\end{subfigure}
	\begin{subfigure}[b]{0.13\textwidth}
		\centering
		\includegraphics[width=\textwidth]{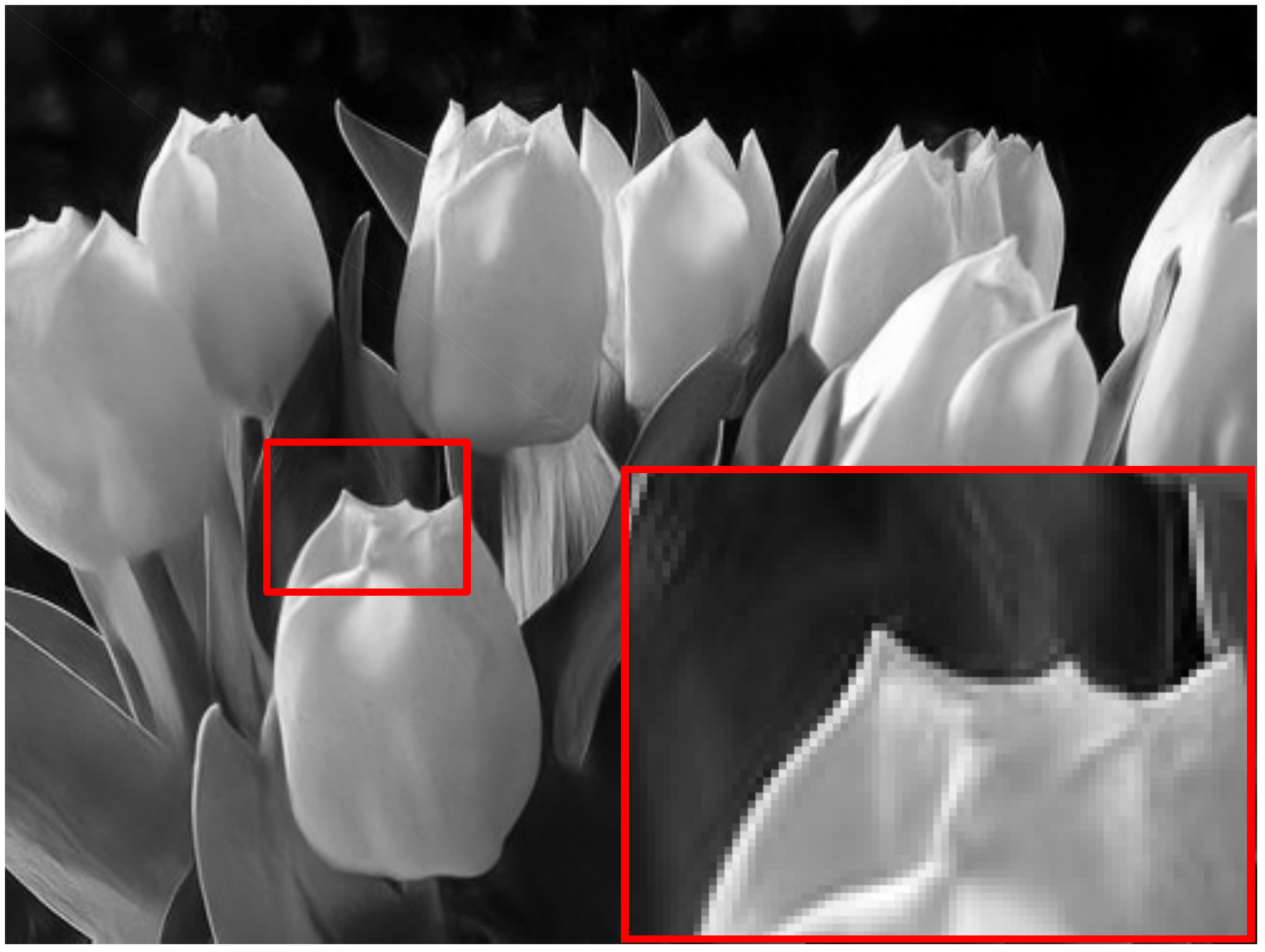}
		\caption{33.78dB, 0.9042 \\ RED-Face}
	\end{subfigure}
	\begin{subfigure}[b]{0.13\textwidth}
		\centering
		\includegraphics[width=\textwidth]{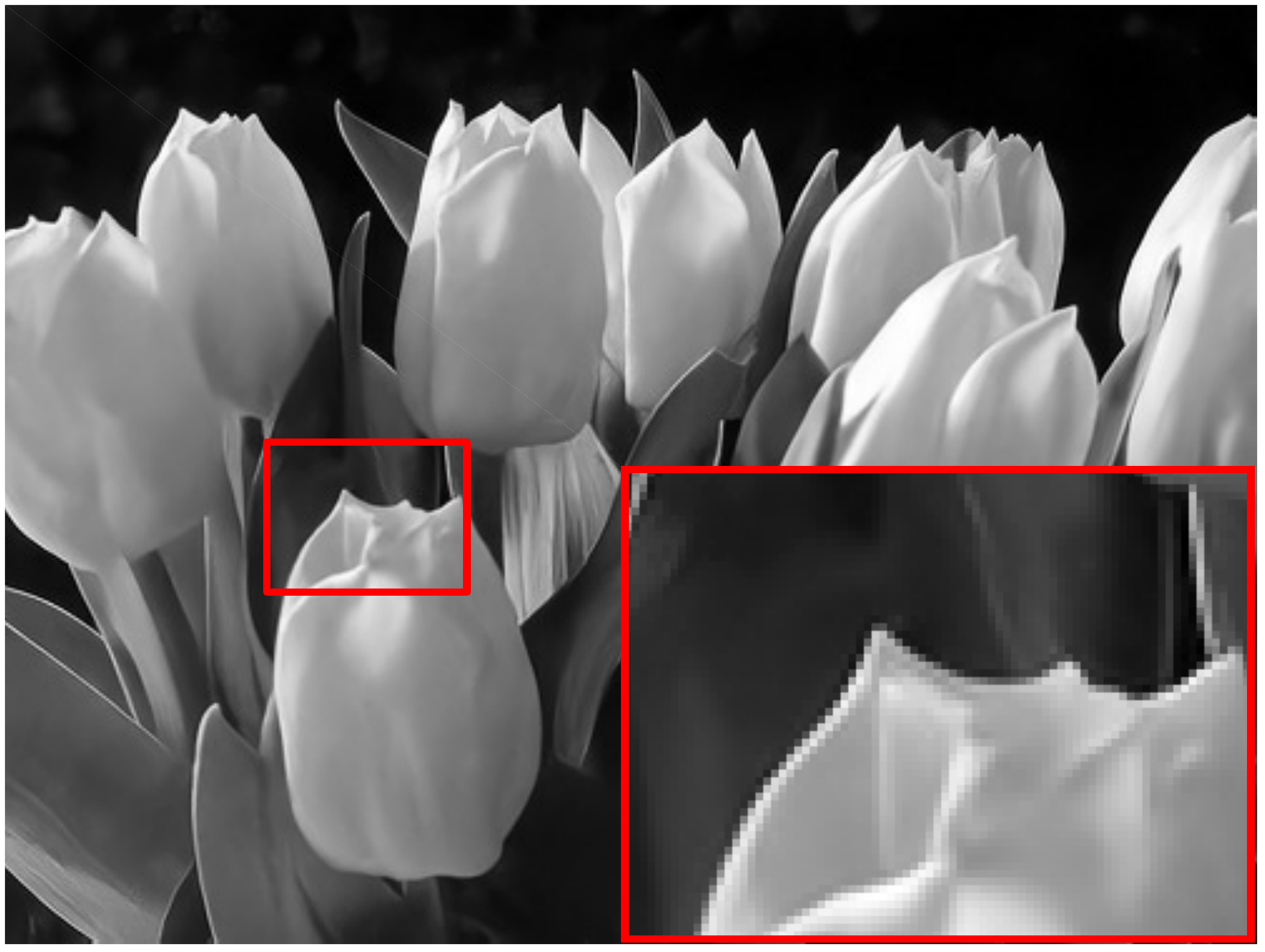}
		\caption{34.42dB, 0.9138 \\ RED-Flower}
	\end{subfigure}
	\begin{subfigure}[b]{0.13\textwidth}
		\centering
		\includegraphics[width=\textwidth]{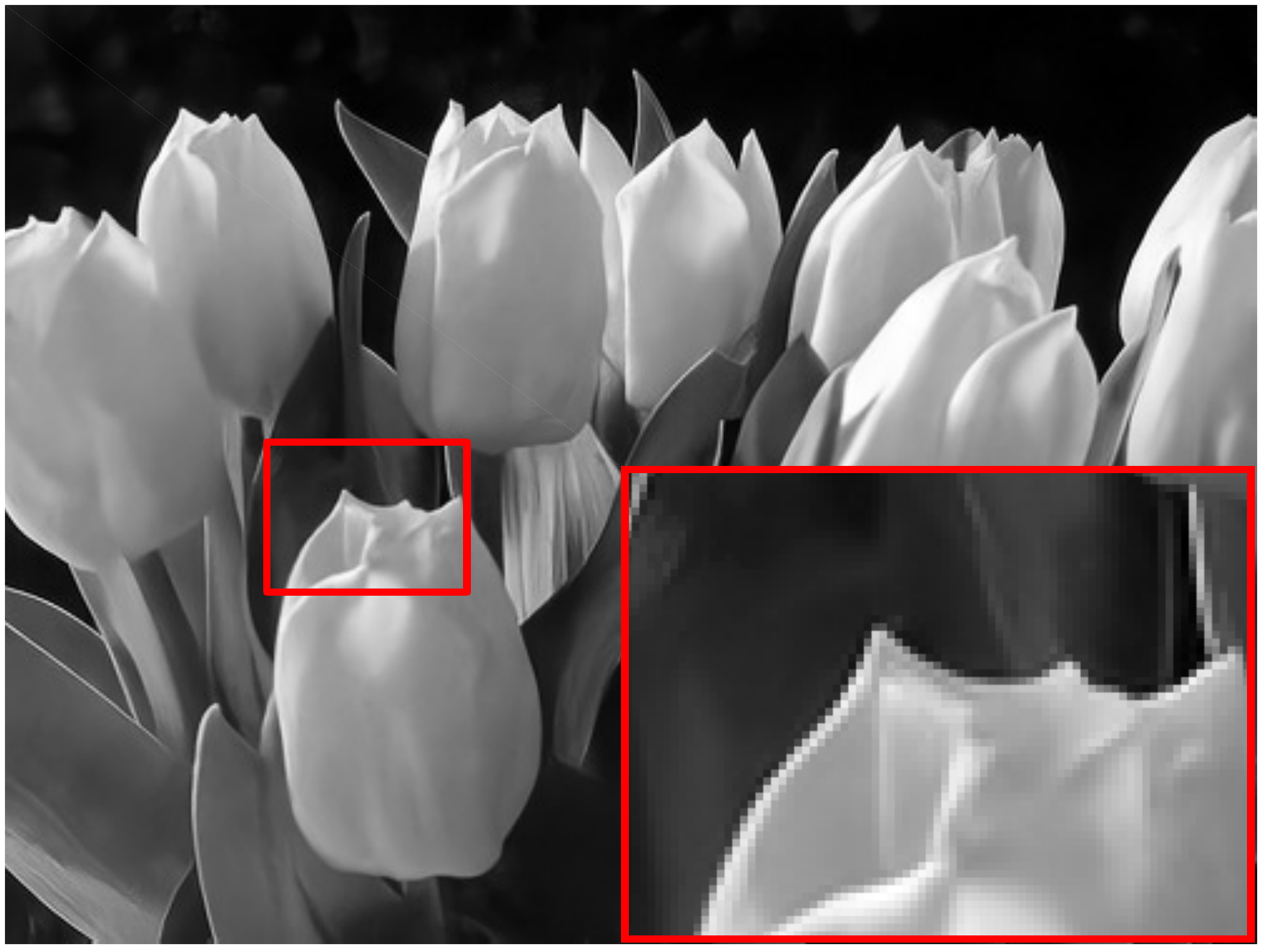}
		\caption{34.49dB, 0.9144 \\ Before}
	\end{subfigure}
	\begin{subfigure}[b]{0.13\textwidth}
		\centering
		\includegraphics[width=\textwidth]{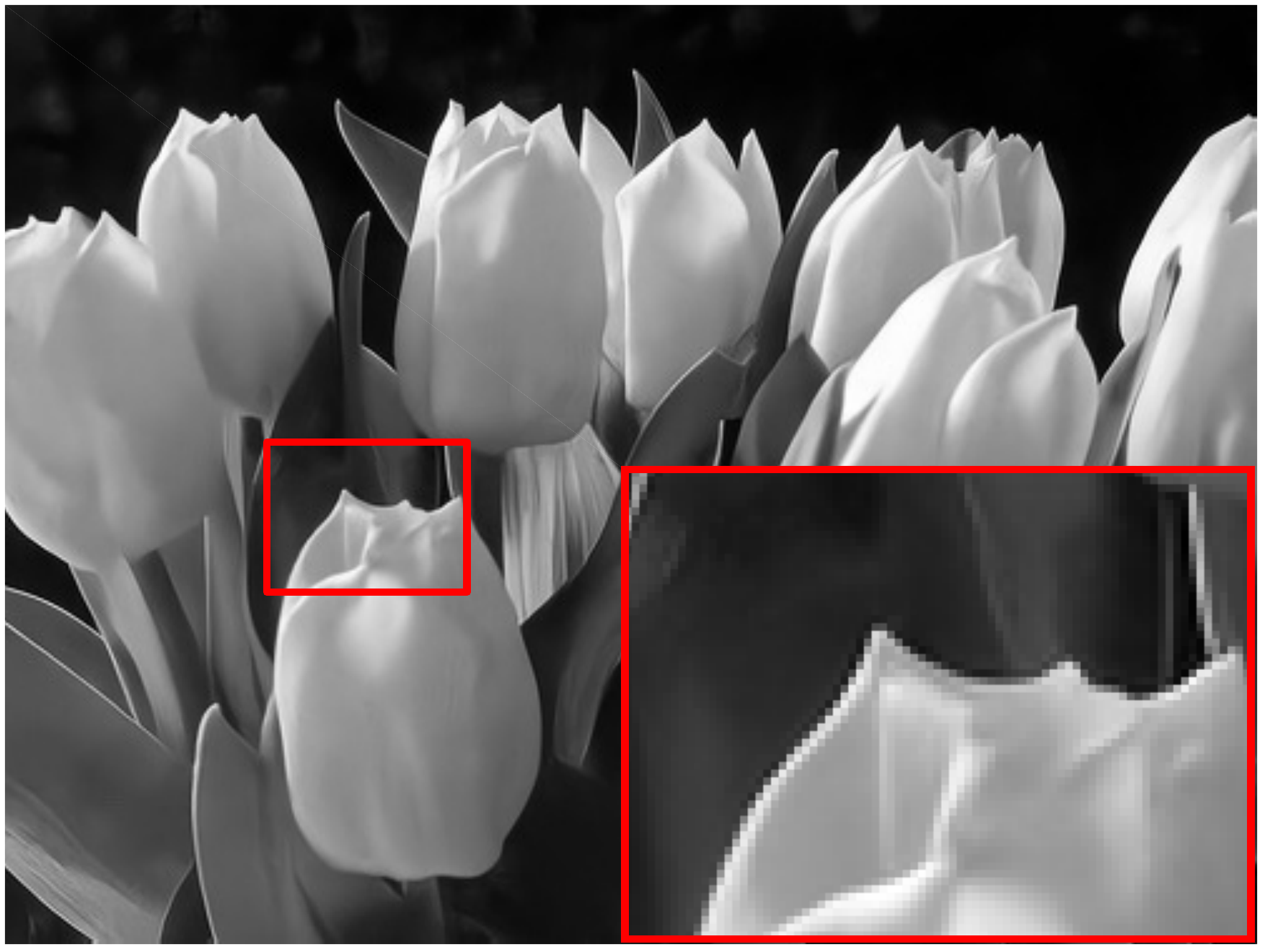}
		\caption{\textbf{34.60dB}, \textbf{0.9152} \\ After}
	\end{subfigure}
	
	\caption{Experiment 2: \texttt{Building}, \texttt{Face} and \texttt{Flower} classes. Testing images are from ImageNet. Reported are the PSNR and SSIM values. In this figure, ``before'' and ``after'' refer to the result before and after applying the booster.}
	\label{fig:class}
\end{figure*}

\begin{table*}[t]
	\centering
	\begin{tabular}{c|ccc|cc|cc|ccc}
		\hline
		& & & & Before & After & Before & After & &&  \\
		& REDNet & REDNet & REDNet & Booster & Booster & Booster & Booster & BM3D & DnCNN & REDNet  \\
		& (Building) & (Face) & (Flower) & (est) & (est) & (oracle) & (oracle) &  & (generic) & (generic) \\
		\hline
        \multicolumn{11}{c}{Unclipped Noise}\\
        \hline
		Building & \textbf{30.6038} & 29.1219 & 29.5430 & 30.3509 & \textbf{30.9371} & 30.6136 & \textbf{30.9391} & 29.5059 & \textbf{30.0341} & 29.9658 \\
		Face & 30.5437 & \textbf{30.7606} & 30.7116 & 30.8047 & \textbf{30.9923} & 30.8907 & \textbf{31.0569} & 30.2397 & 30.6967 & \textbf{30.7020} \\
		Flower & 31.2785 & 31.1325 & \textbf{31.5428} & 31.5788 & \textbf{31.6035} & 31.6009 & \textbf{31.6103} & 30.6088 & \textbf{31.4211} & 31.4105 \\
		\hline
        \multicolumn{11}{c}{Clipped Noise}\\
        \hline
		Building & \textbf{30.3962} & 28.9529 & 29.3453 & 30.3303 & \textbf{30.4095} & 30.4020 & \textbf{30.4749} & 29.2986 & 29.7722 & \textbf{29.7743} \\
		Face & 30.1871 & \textbf{30.3889} & 30.3443 & 30.4501 & \textbf{30.7419} & 30.5086 & \textbf{30.7957} & 29.9685 & 30.2813 & \textbf{30.2896} \\
		Flower & 31.0875 & 30.9497 & \textbf{31.3114} & 31.3221 & \textbf{31.5041} & 31.3759 & \textbf{31.5404} & 30.4224 & 31.1534 & \textbf{31.1752} \\
		\hline
	\end{tabular}
	\caption{Experiment 2: Different image classes. Class-specific REDNets have better performance than BM3D, DnCNN (generic) and REDNet (generic). CsNet selects the best class. We use 10 images from ImageNet for testing. The labels ``est'' and ``oracle'' refer to estimated MSE and the oracle MSE, respectively.}
	\label{tab:2}
	
\end{table*}

The objective of this experiment is to evaluate the performance of CsNet when the initial denoisers are trained for different image classes. To this end, we fix the type of initial denoisers as REDNet, and train three different REDNets using three classes of images: \texttt{Flower}, \texttt{Face} and \texttt{Building}. We have experimented with other initial denoisers such as DnCNN, but the results are similar. 
 
To train the initial denoisers, we manually select 200 class-specific images for each class from the ImageNet \cite{Deng2009}. We fix the noise level as $\sigma = 20$ to eliminate the complication of having uncertainty in both noise levels and image classes. Initial denoisers are trained with unclipped noise. We train two different MSE estimators, one for unclipped noise and one for clipped noise.

The result of this experiment is shown in Table~\ref{tab:2} with a few representative examples in \fref{fig:class}. We observe that denoisers trained with generic database such as DnCNN (generic) and REDNet (generic) perform worse than class-specific denoisers. For example, in the \texttt{Building} image, DnCNN (generic) and REDNet (generic) attain 29.7722dB and 29.7743dB respectively in the clipped case. In contrast, REDNet-\texttt{Building} has a PSNR of 30.39dB, approximately 0.7dB above the REDNet (generic).

When using the proposed scheme, the ``before boosting'' result is already better than the initial denoiser's. This result holds for both clipped and unclipped, and all classes. Moreover, ``before boosting'' is better than all the generic denoisers, indicating the effectiveness of the convex optimization part. If we apply a booster, then the performance is boosted further.


\subsection{Experiment 3: Different Denoiser Types}
\label{sec:experiment 3}

\begin{figure*}[t]
	\hfill
	\captionsetup[subfigure]{labelformat=empty}
	\captionsetup{justification=centering}
	\begin{subfigure}[b]{0.24\textwidth}
		\centering
		\includegraphics[width=\textwidth]{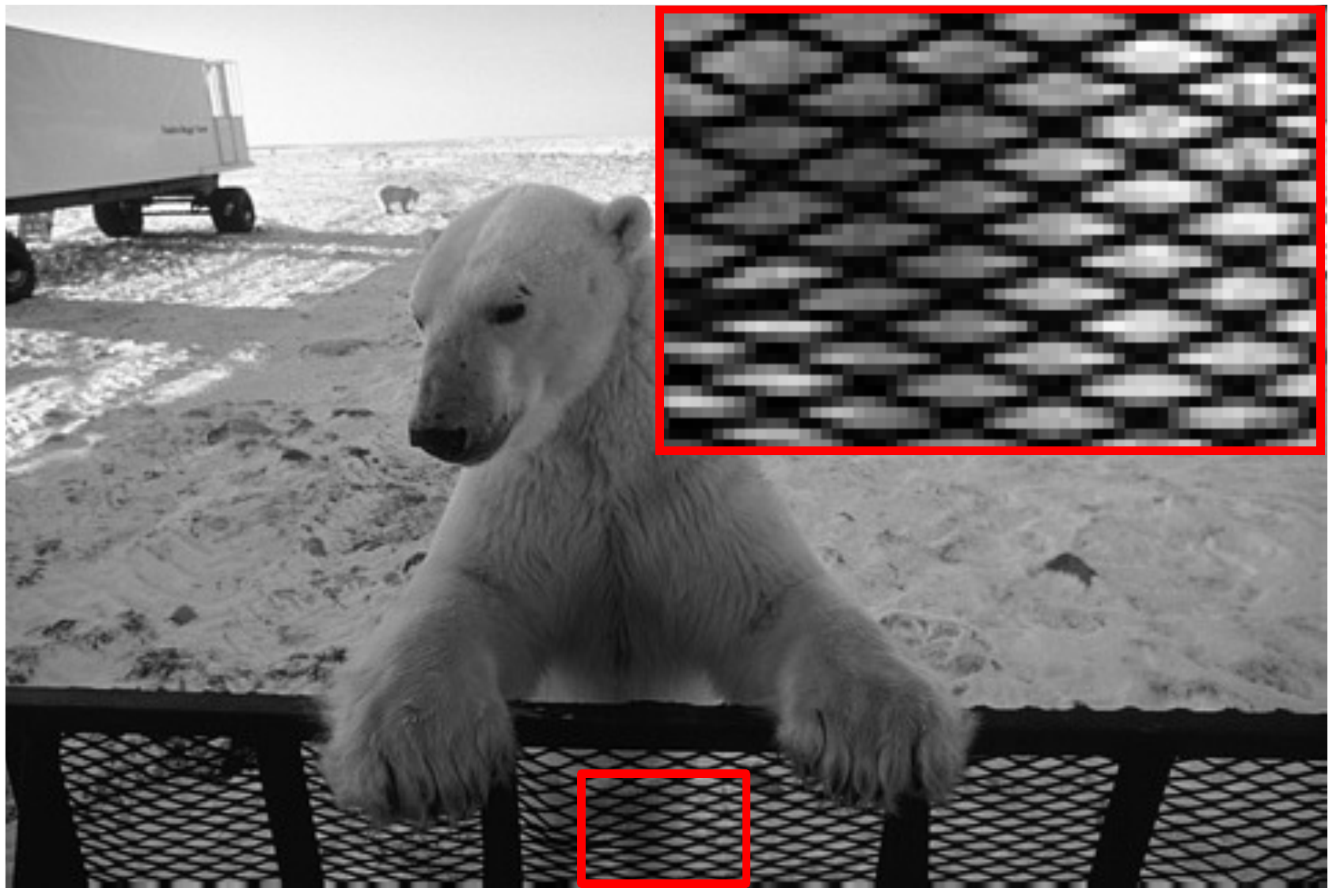}
		\caption{Groundtruth}
	\end{subfigure}
	\begin{subfigure}[b]{0.24\textwidth}
		\centering
		\includegraphics[width=\textwidth]{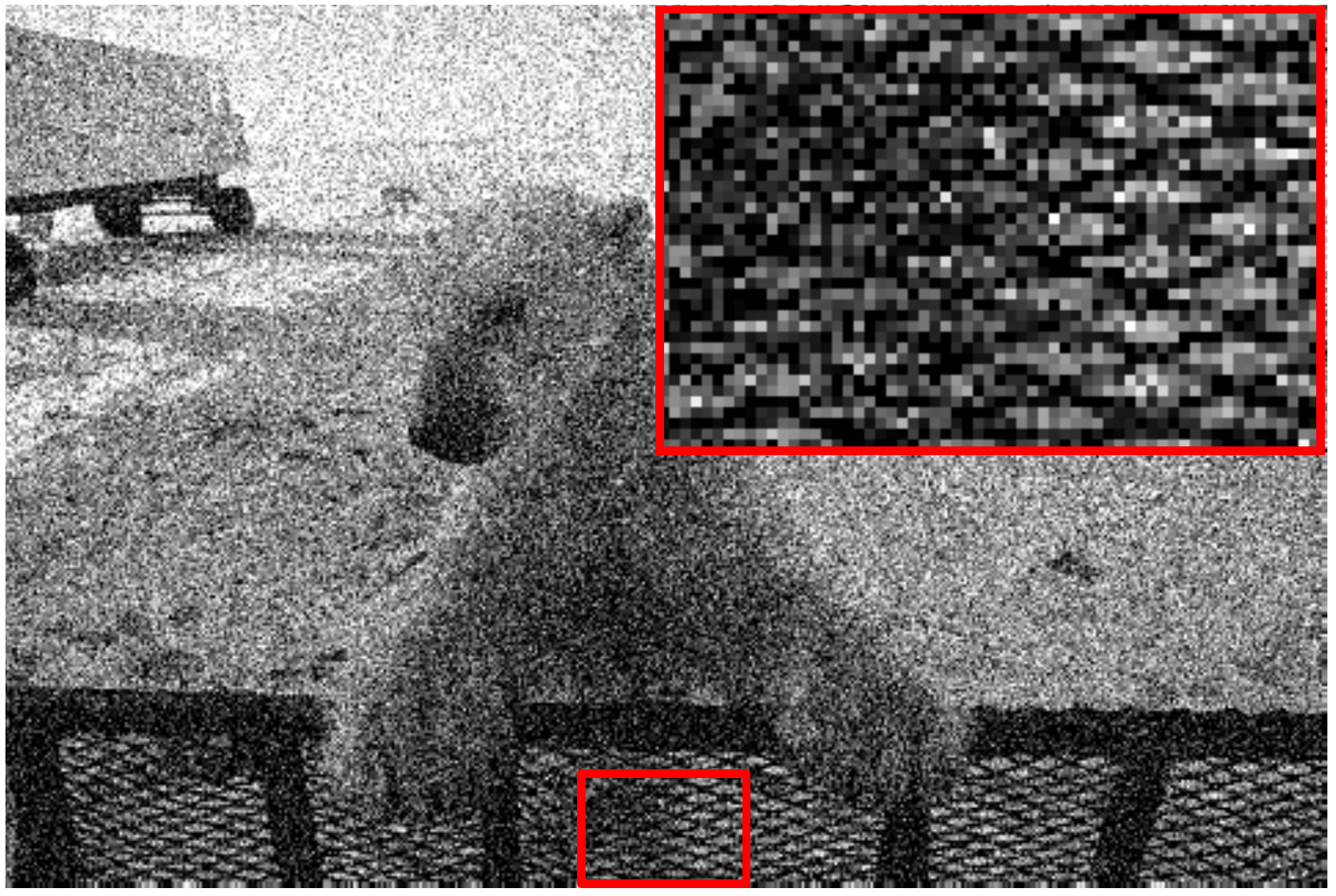}
		\caption{Input, $\sigma$=50, 14.99dB, 0.1998}
	\end{subfigure}
	\begin{subfigure}[b]{0.24\textwidth}
		\centering
		\includegraphics[width=\textwidth]{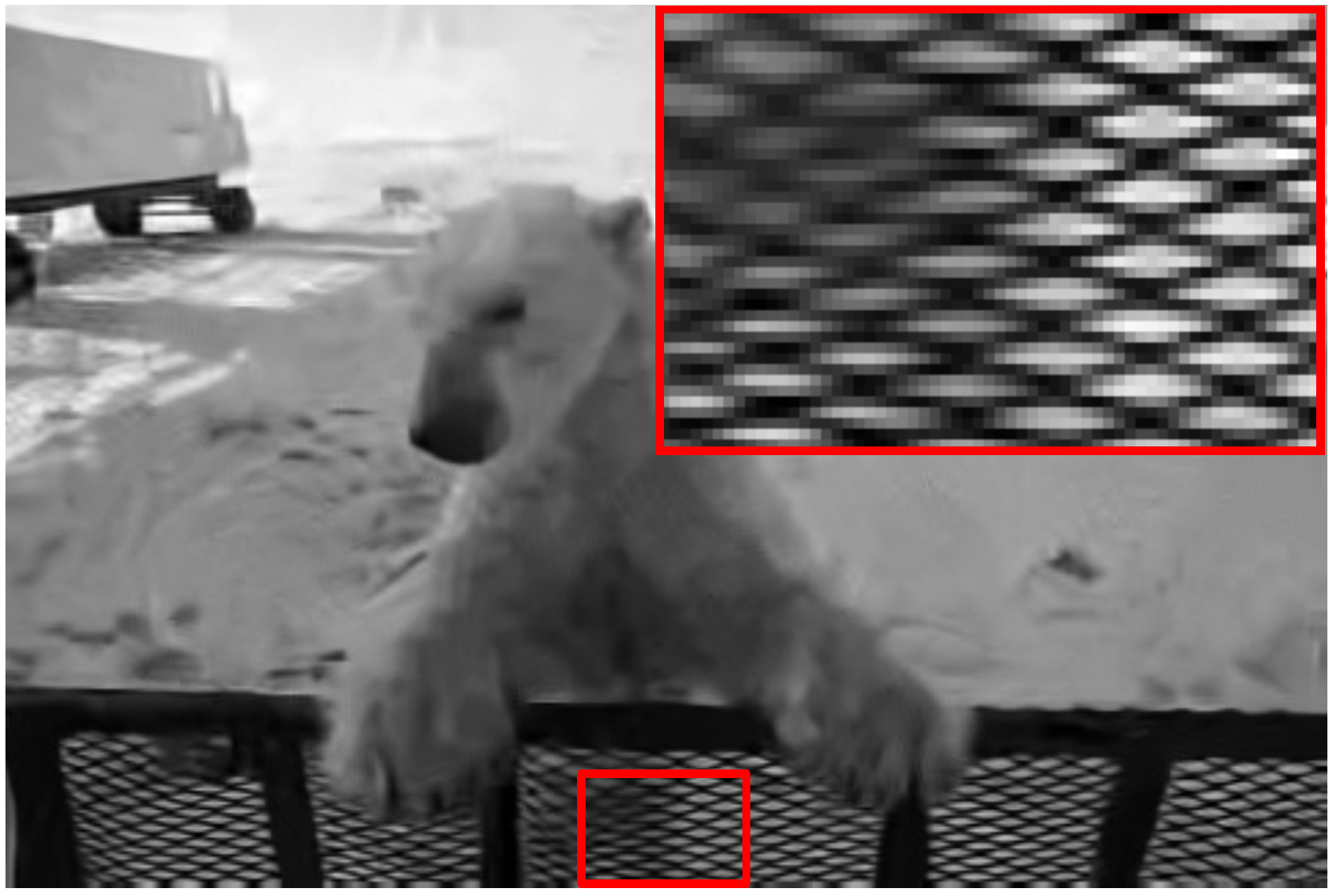}
		\caption{BM3D, $\sigmahat$=50, 25.83dB, 0.7094}
	\end{subfigure}
	\begin{subfigure}[b]{0.24\textwidth}
		\centering
		\includegraphics[width=\textwidth]{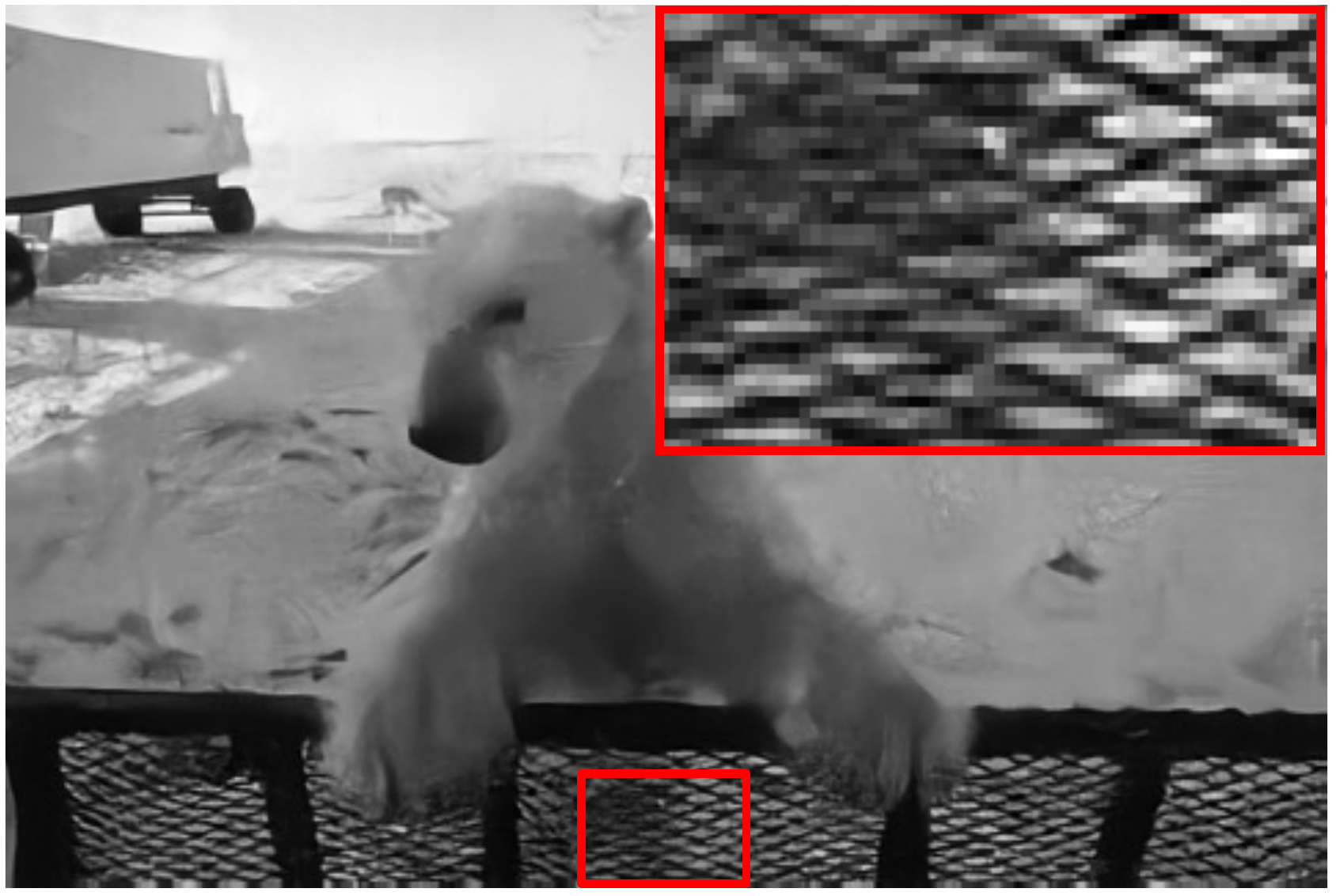}
		\caption{DnCNN, $\sigmahat$=50, 25.47dB, 0.7219}
	\end{subfigure}
	
	\hfill
	\begin{subfigure}[b]{0.24\textwidth}
		\centering
		\includegraphics[width=\textwidth]{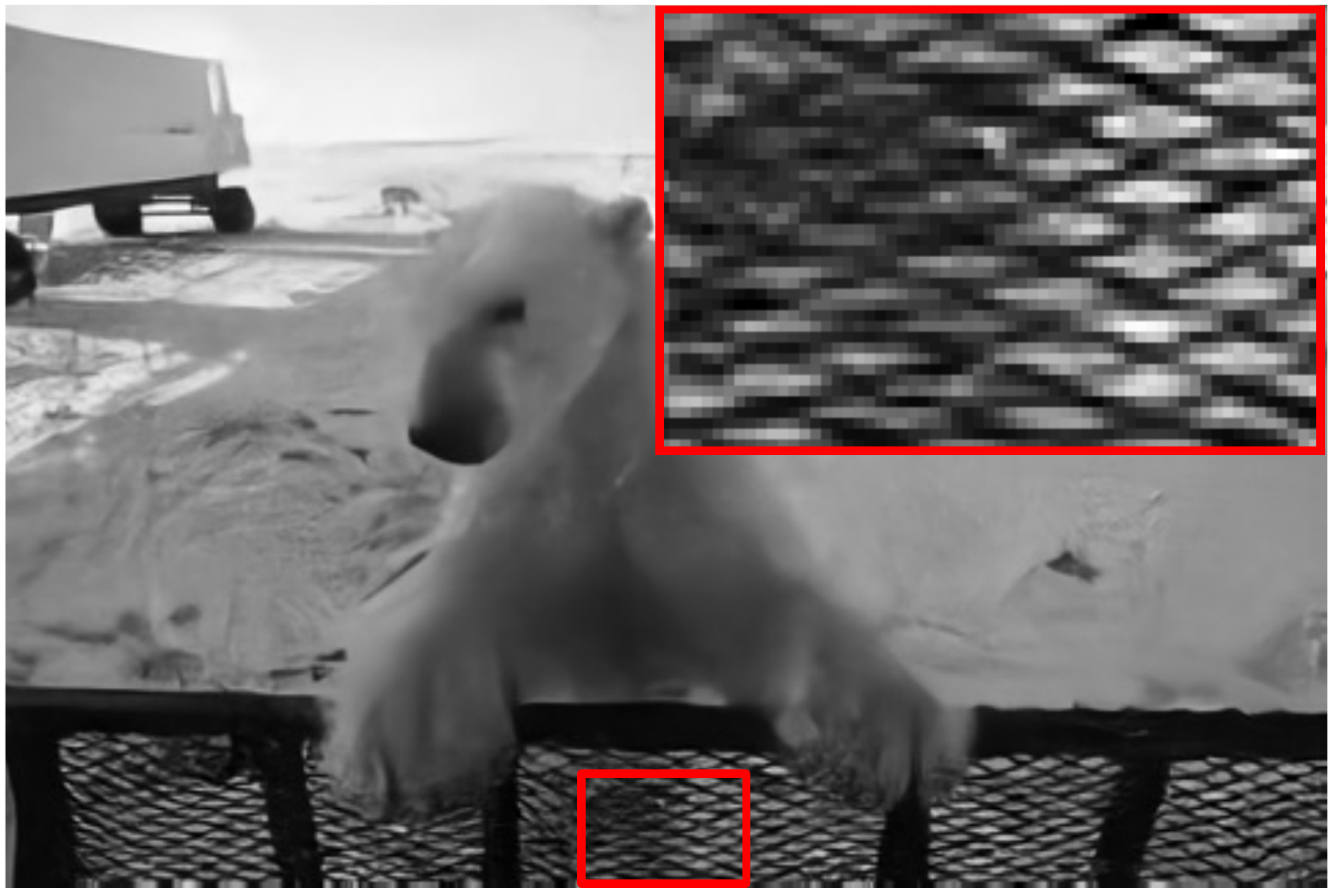}
		\caption{FFDNet, $\sigmahat$=50, 25.58dB, 0.7269}
	\end{subfigure}
	\begin{subfigure}[b]{0.24\textwidth}
		\centering
		\includegraphics[width=\textwidth]{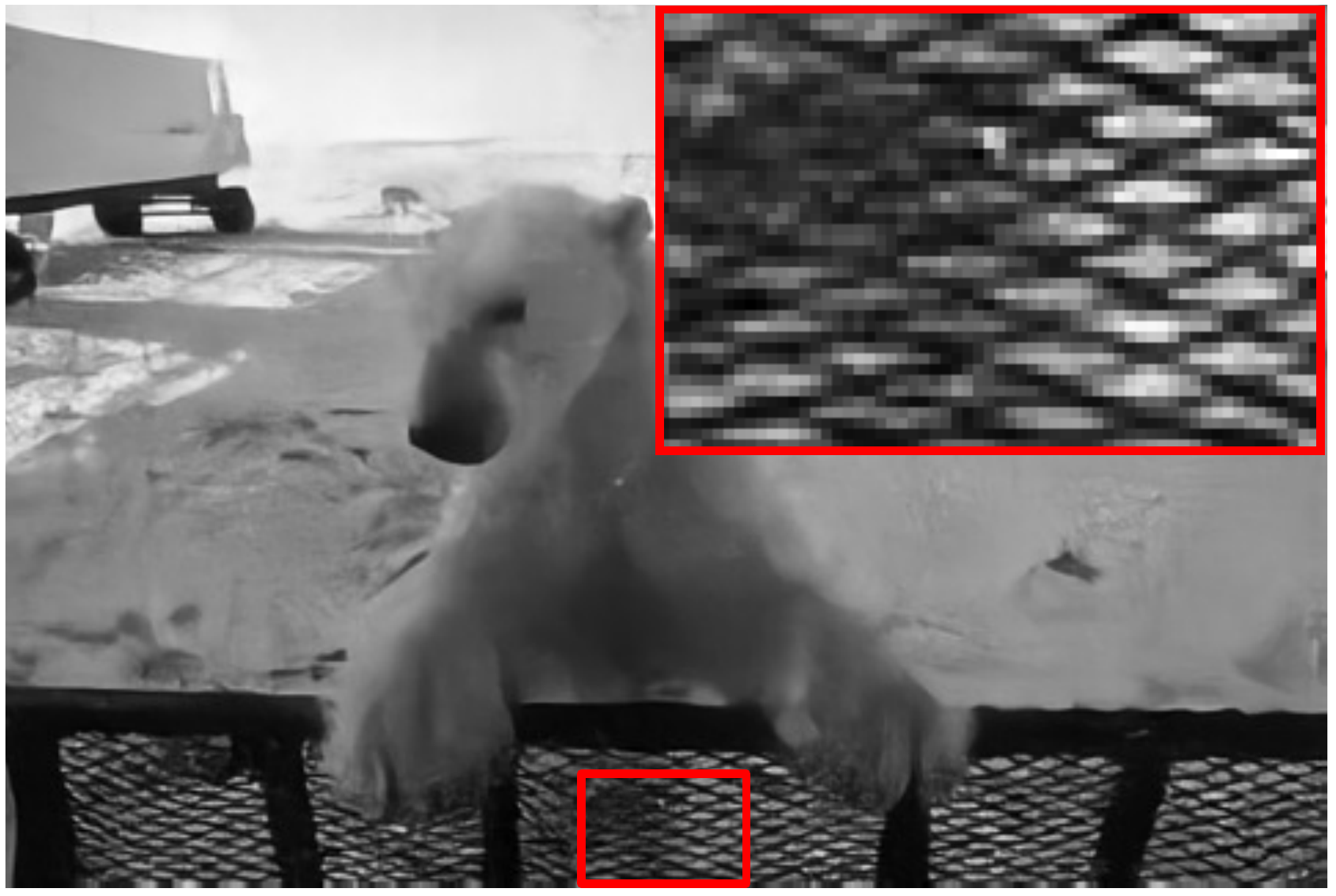}
		\caption{REDNet, $\sigmahat$=50, 25.54dB, 0.7241}
	\end{subfigure}
	\begin{subfigure}[b]{0.24\textwidth}
		\centering
		\includegraphics[width=\textwidth]{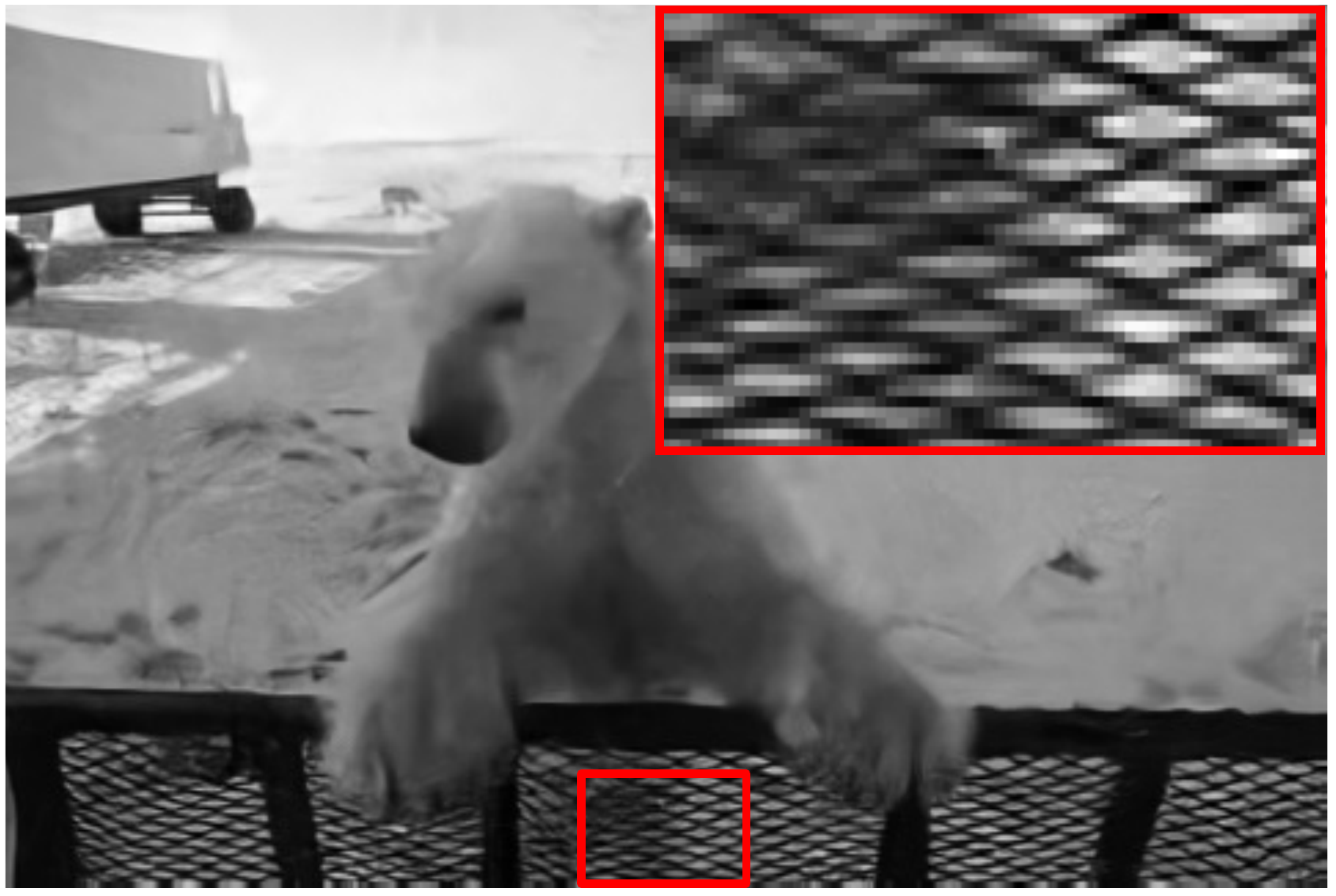}
		\caption{Before Booster, 25.91dB, 0.7263}
	\end{subfigure}
	\begin{subfigure}[b]{0.24\textwidth}
		\centering
		\includegraphics[width=\textwidth]{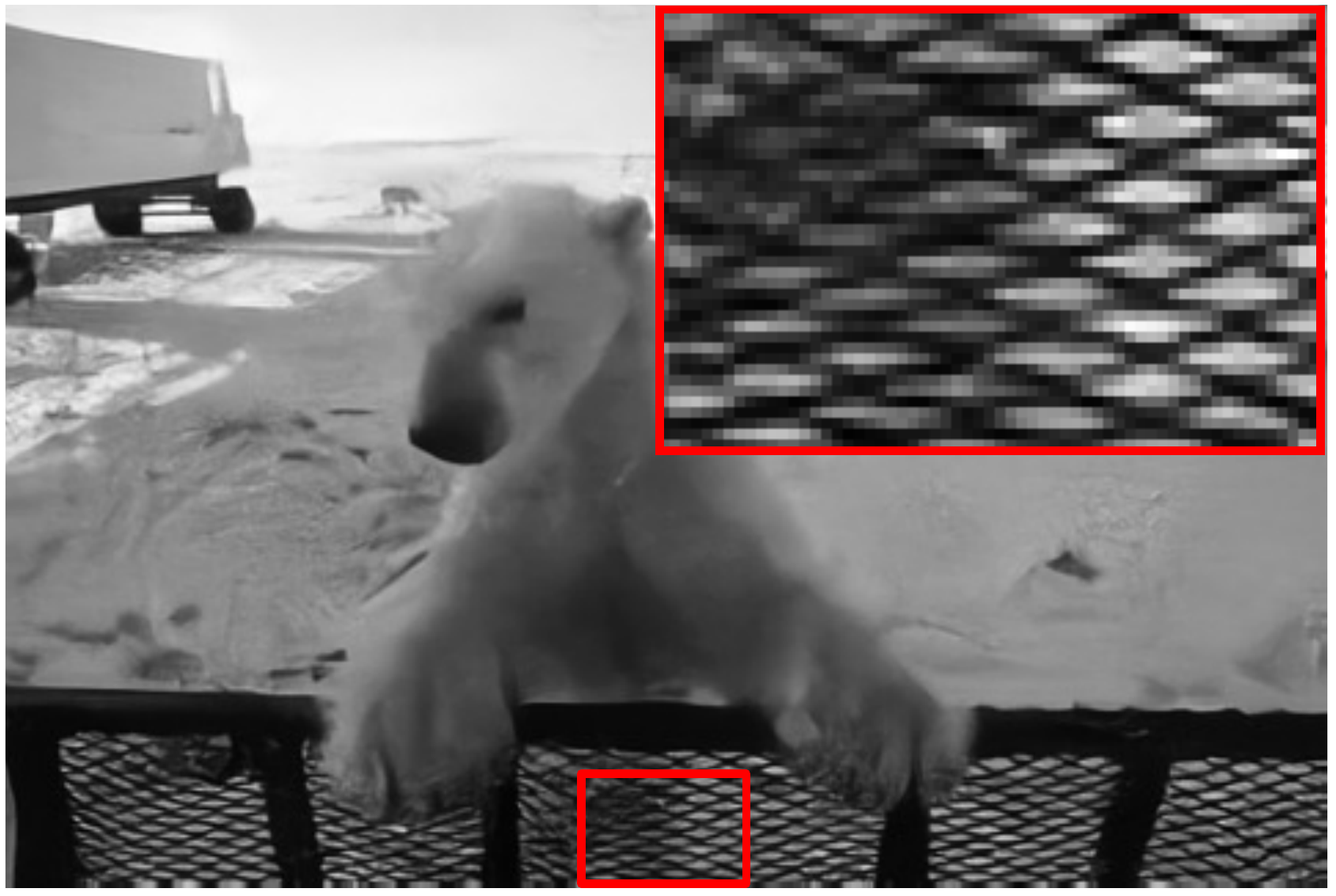}
		\caption{After Booster, \textbf{25.98dB}, \textbf{0.7333}}
	\end{subfigure}
	\caption{Experiment 3: Different denoiser type. The initial denoisers are BM3D \cite{Dabov2007}, DnCNN \cite{Zhang2017_tip}, REDNet \cite{Mao2016}, and FFDNet \cite{Zhang2017_arxiv}. The testing image is \texttt{Bear} (size 321$\times$481) from BSD500. Reported are the PSNR and SSIM values.}
	\label{fig:diff_type}
\end{figure*}

\begin{table*}[t]
	\centering
	\begin{tabular}{c|cccc|cc|cc}
		& BM3D	& DnCNN	& FFDNet	& REDNet	& Before  & After 	&	Before  & After  \\
        & \cite{Dabov2007} & \cite{Zhang2017_tip} & \cite{Zhang2017_arxiv}  & \cite{Mao2016} & Boost & Boost & Boost & Boost \\
		&	&	&	&	&	(est)	& (est)	& (oracle)	& (oracle) \\
		\hline
        \multicolumn{9}{c}{Unclipped Noise}\\
        \hline
		$\sigma=10$ & 33.6067 & \textbf{34.1625} & 34.0178 & 34.1619 & \textbf{34.1813} & 34.1678 & \textbf{34.2147} & 34.1906 \\
		$\sigma=20$ & 29.8558 & \textbf{30.4924} & 30.4357 & 30.4755 & 30.5258 & \textbf{30.5401} & \textbf{30.5559} & 30.5554 \\
		$\sigma=30$ & 27.9271 & 28.5286 & \textbf{28.5458} & 28.5209 & 28.5869 & \textbf{28.6198} &  28.6199 & \textbf{28.6299} \\
		$\sigma=40$ & 26.5688 & 27.2202 & \textbf{27.2845} & 27.2393 & 27.2978 & \textbf{27.3384} & 27.3381 & \textbf{27.3438} \\
		$\sigma=50$ & 25.7005 & 26.3159 & \textbf{26.3675} & 26.3249 & 26.3695 & \textbf{26.4235} & \textbf{26.4226} & 26.4223 \\
		\hline
        \multicolumn{9}{c}{Clipped Noise}\\
        \hline
		$\sigma=10$	& 33.5628 & 34.1030 & 33.9434 & \textbf{34.1216} & \textbf{34.1362} & 33.8933 & \textbf{34.1625} & 33.9012 \\
		$\sigma=20$	& 29.7309 & 30.3266 & 30.2683 & \textbf{30.3378} & 30.3672 & \textbf{30.4846} & 30.3994 & \textbf{30.5076} \\
		$\sigma=30$	& 27.6804 & 28.1727 & 28.1846 & \textbf{28.2007} & 28.2282 & \textbf{28.5211} & 28.2764 & \textbf{28.5529} \\
		$\sigma=40$	& 26.2208 & 26.6024 & \textbf{26.6452} & 26.6205 & 26.6788 & \textbf{27.1906} & 26.7187 & \textbf{27.2108} \\
		$\sigma=50$	& 24.9885 & 25.3449 & \textbf{25.3491} & 25.3479 & 25.3952 & \textbf{26.1573} & 25.4354 & \textbf{26.1766} \\
		\hline
	\end{tabular}
	\caption{Experiment 3: Different denoiser type. The initial denoisers are BM3D \cite{Dabov2007}, DnCNN \cite{Zhang2017_tip}, REDNet \cite{Mao2016}, and FFDNet \cite{Zhang2017_arxiv}. We use 200 images from BSD500 for testing. In this figure, ``before'' and ``after'' refer to the result before and after applying the booster. The labels ``est'' and ``oracle'' refer to estimated MSE and the oracle MSE, respectively.}
	\label{tab:3}
\end{table*}

The objective of this experiment is to evaluate CsNet for different types of initial denoisers. To this end, we consider four denoisers running at specific noise levels $\sigmahat$ that match with the actual noise level $\sigma$. These denoisers are BM3D \cite{Dabov2007}, DnCNN \cite{Zhang2017_tip}, REDNet \cite{Mao2016} and FFDNet \cite{Zhang2017_arxiv}. We use the original implementation by the authors for DnCNN and FFDNet, and build our own REDNet.

The result of this experiment is shown in Table~\ref{tab:3}. Among the four denoisers, FFDNet and REDNet have comparable performance at the top, followed by DnCNN and then BM3D. For the five noise levels we tested, CsNet consistently improves the performance. In particular, ``before boosting'' is always better than the initial denoiser. This means the convex optimization has effectively selected the best initial denoiser. The margin between the best initial denoiser and ``before boosting'' is small, because the denoisers have similar behavior and so the convex optimization solution is at one of the vertices of the constraint hyperplane. \fref{fig:diff_type} shows a visual comparison on the \texttt{Bear} image. In this image, BM3D actually performs better than DnCNN. The proposed CsNet can pick this best estimate (25.91dB), and boost the PSNR to 25.98dB.

\subsection{Limitations and Extensions}
The effectiveness of CsNet is dominated by the accuracy of the MSE estimate. The proposed neural network MSE estimator has a bias but a small variance. This is better than deterministic estimators such as SURE which has no bias but excessively large variance. However, if the noise statistics changes, we need to train a different MSE estimator.

If the images are large and complex, we can partition the image into sub-regions and use CsNet to handle each region separately. The bottleneck, again, is the accuracy in estimating the MSE. One resolution is to consider regularization in \eref{eq:optimization}. Possible choices of regularization include forcing similar weights for denoisers that are known to perform similarly. We leave the discussion of such regularization to future work.

Real noise of an image is significantly more complicated than i.i.d. Gaussian. Typical sources of noise include: photon shot noise, optical diffusion, minority carrier, thermal effect, dark current, circuit instability, and various nonlinear operations due to the image processing pipeline. When taking all these into account, a better noise model beyond Gaussian (and even mixed Poisson-Gaussian) is needed. Readers interested in this topic can consult, e.g., \cite{Azzari_Foi_2014, Zhang_Hirakawa_2017,Cheng_Hirakawa_2018} for theory, and \cite{Xu2017,Xu2018_1} for some recent progress on algorithms. The current CsNet is not designed to handle this type of real noise. However, if one can show that real noise is a mixture of individual noises, then CsNet could potentially be a solution.

When training the neural networks we choose to use the $L_1$ metric, for it gives slightly better visual quality then the usual $L_2$ metric. We do not heavily tune this metric because it is not the focus of the paper. For readers who are concerned about the loss function, we refer to \cite{Zhao2017} for some recent empirical findings on the topic.

The advantage of CsNet relative to other class-aware neural network denoisers is that we allow combination of multiple denoisers. Typical class-aware denoisers, e.g., \cite{Luo2015,Luo2016,Remez2017}, rely on semantic classifiers to greedily select only one denoiser. As we demonstrated in Section~\ref{sec:experiment 2}, a combination of the denoisers is better than the best of the individuals.

CsNet is a general framework for combining estimators. That is, one is not limited to applying CsNet to image denoising problems, although we use denoising as a demonstration. A straight forward extension of CsNet is to combine multiple deblurring algorithms, or to combine multiple image super-resolution algorithms. In complex imaging scenarios where no single method performs uniformly better than the others, CsNet offers a solution to integrate individual weak estimators.

\section{Conclusion}
We present an optimal framework called the Consensus Neural Network (CsNet) to combine multiple weak image denoisers. CsNet consists of three major components. Starting with a set of initial image denoisers, CsNet first uses a novel deep neural network to estimate the MSE. The deep neural network is more robust than the traditional estimators such as SURE for estimating the MSE. Once the MSE is estimated, CsNet solves a convex optimization problem. The optimality of the CsNet is guaranteed by the convex formulation. Finally, the combined estimate is boosted using a new deep neural network image booster. Experimental results confirm the effectiveness of CsNet, where it shows superior performance compared to other state-of-the-art denoising algorithms on tasks including: overcoming noise level mismatch, combining denoisers for different image classes, and combining different denoiser types.

\section*{Acknowledgement}
We thank the anonymous reviewers for very valuable feedback which significantly improves the paper. We also thank Nvidia for donating the Titan-X GPU for this work.

\section*{Appendix: Proofs}
\subsection*{Proof of Proposition 5}
First, we show that
\begin{align*}
\E \|\vztilde - \vzhat \|^2
&\bydef \E \|\mZhat\vwtilde - \mZhat\vw \|^2 = \E \|\mZhat\vwtilde - \vz + \vz - \mZhat\vw \|^2\\
&= \E \|(\mZhat\vwtilde - \mZ\vwtilde) - (\mZhat\vw - \mZ\vw) \|^2\\
&= \E \|(\mZhat - \mZ)(\vwtilde - \vw) \|^2 = (\vwtilde - \vw)^T \mSigma (\vwtilde - \vw).
\end{align*}
The term $(\vwtilde - \vw)^T \mSigma (\vwtilde - \vw)$ can be upper bounded by
\begin{align*}
(\vwtilde - \vw)^T \mSigma (\vwtilde - \vw)
&= \vwtilde^T\mSigma\vwtilde - \vw^T\mSigma\vw - 2(\vwtilde-\vw)^T\mSigma\vw\\
&\le \vwtilde^T\mSigma\vwtilde - \vw^T\mSigma\vw.
\end{align*}
The last inequality holds because the function $f(\vw) = \vw^T\mSigma\vw$ attains its first order optimality at $\vw$ when
\begin{align*}
\nabla f(\vw)^T (\vwtilde - \vw) \ge 0.
\end{align*}
Therefore,
\begin{align*}
&\vwtilde^T\mSigma\vwtilde - \vw^T\mSigma\vw \\
&= \vwtilde^T\mSigma\vwtilde - \vwtilde^T\mSigmatilde\vwtilde + \vwtilde^T\mSigmatilde\vwtilde - \vw^T\mSigma\vw \\
&\le \vwtilde^T\mSigma\vwtilde - \vwtilde^T\mSigmatilde\vwtilde + \vw^T\mSigmatilde\vw - \vw^T\mSigma\vw\\
&= \vwtilde^T\mSigmatilde\vwtilde \left(\frac{\vwtilde^T\mSigma\vwtilde}{\vwtilde^T\mSigmatilde\vwtilde} - 1\right)
+ \vw^T\mSigma\vw\left(\frac{\vw^T\mSigmatilde\vw}{\vw^T\mSigma\vw} - 1\right)\\
&\le (\vwtilde^T\mSigmatilde\vwtilde + \vw^T\mSigma\vw) \delta,
\end{align*}
where
\begin{equation}
\delta = \max\left( \left|\frac{\vwtilde^T\mSigma\vwtilde}{\vwtilde^T\mSigmatilde\vwtilde} - 1 \right|,
\left|\frac{\vw^T\mSigmatilde\vw}{\vw^T\mSigma\vw} - 1\right|\right)
\label{eq:proof step 1}
\end{equation}
We can also show that
\begin{align*}
\vw^T\mSigmatilde\vw &\le \vw^T\mSigma\vw (1+\delta)
\end{align*}
Continue the calculation, we have
\begin{align*}
(\vwtilde^T\mSigmatilde\vwtilde + \vw^T\mSigma\vw) \delta
&\le (\vw^T\mSigmatilde\vw + \vw^T\mSigma\vw) \delta\\
&\le (\vw^T\mSigma\vw) (2\delta + \delta^2)
\end{align*}
This implies that
\begin{equation*}
\E \|\vztilde - \vzhat \|^2  \le \E\|\vzhat - \vz\|^2 (2\delta + \delta^2).
\end{equation*}
It remains to derive an upper bound on $\delta$. To this end, we consider the generalized Rayleigh quotient of two positive definite matrices $\mA$ and $\mB$. It is known that \cite{Boyd_Ghaoui_1993}
$$
\max_{\vw \not= 0} \; \frac{\vw^T\mA\vw}{\vw^T\mB\vw} = \lambda_{\max} \left(\mB^{-\frac{1}{2}}\mA\mB^{-\frac{1}{2}}\right).
$$
Therefore,
\begin{align*}
\left|\frac{\vw^T\mSigmatilde\vw}{\vw^T\mSigma\vw} - 1\right|
&\le \max_{\vw \not= 0} \; \left| \frac{\vw^T\mSigmatilde\vw}{\vw^T\mSigma\vw} - 1 \right|
= \max_{\vw \not= 0} \; \left| \frac{\vw^T(\mSigmatilde-\mSigma)\vw}{\vw^T\mSigma\vw} \right|\\
&= \max_{i} \left| \lambda_i\left(\mSigma^{-\frac{1}{2}}(\mSigmatilde - \mSigma)\mSigma^{-\frac{1}{2}}\right) \right|,
\end{align*}
where $\lambda_i(\mA)$ denotes the $i$-th eigen-value of the matrix $\mA$. With some additional algebra we can show that
\begin{align*}
&\max_{i} \left| \lambda_i\left(\mSigma^{-\frac{1}{2}}(\mSigmatilde - \mSigma)\mSigma^{-\frac{1}{2}}\right) \right|\\
&=\max_{i} \left| 1 - \lambda_i\left(\mSigma^{-\frac{1}{2}}\mSigmatilde\mSigma^{-\frac{1}{2}}\right) \right|\\
&\overset{(a)}{=} \max_i \left| 1 - \lambda_i\left(\mSigma^{-1}\mSigmatilde \right)\right|\\
&\overset{(b)}{\le}  \max_i \left| \frac{1}{\lambda_i\left(\mSigma^{-1}\mSigmatilde \right)} - \lambda_i\left(\mSigma^{-1}\mSigmatilde \right)\right|,
\end{align*}
where (a) holds because of Lemma 1, and (b) holds because for any $t \ge 0$, $|1-t| \le |t - \frac{1}{t}|$. By recalling the definition of the matrix operator norm, we have that
\begin{align*}
\left|\frac{\vw^T\mSigmatilde\vw}{\vw^T\mSigma\vw} - 1\right|
\le \left\| \mSigma\mSigmatilde^{-1}  - \mSigma^{-1}\mSigmatilde \right\|_2 \bydef \Delta.
\end{align*}
Substituting this result into \eref{eq:proof step 1}, and by symmetry, we complete the proof.

\begin{lemma}
Consider two matrices $\mA,\mB\in\R^{n \times n}$ where $\mA\mB$ and $\mB\mA$ are diagonalizable. If $\lambda$ is an eigen-value of $\mA\mB$, then $\lambda$ is also an eigen-value of $\mB\mA$.
\end{lemma}
\begin{proof}
Let $\vv \in \R^n$ be an eigen-vector of $\mA\mB$, i.e., $$\mA\mB\vv=\lambda \vv.$$ Then, multiplying both sides by $\mB$ yields $$\mB\mA\left(\mB\vv\right)=\lambda \left(\mB\vv\right).$$ Hence, $\lambda$ is an eigen-value of $\mB\mA$, with the corresponding eigen-vector $\mB\vv$.
\end{proof}

\balance
\bibliographystyle{IEEEbib}
\bibliography{egbib}

\end{document}